\documentclass[twoside, 11pt]{article}

\usepackage{amsmath,amsfonts,amssymb,amsthm}

\usepackage[abbrvbib, preprint]{jmlr2e}

\usepackage{graphicx}

\usepackage[dvipsnames]{xcolor}

\theoremstyle{definition}
\newtheorem{assumption}[theorem]{Assumption}

\usepackage{algorithm}
\usepackage{algpseudocode}
\usepackage{tikz}
\usetikzlibrary{decorations.pathreplacing,calc,backgrounds}
\usepackage{tikz-cd}
\usepackage{float}
\usepackage{bm}
\usepackage{bbm}

\newcommand{\ka}[1]{{\color{ForestGreen} #1}}

\let\phi\varphi
\let\epsilon\varepsilon

\def\Pr{\mathrm{Pr}}

\def\TV{{\rm\scriptscriptstyle TV}}
\def\cF{\mathcal{F}}
\def\diff{\mathrm{d}}
\def\cP{\mathcal{P}}
\def\cK{\mathcal{K}}
\def\cT{\mathcal{T}}

\def\Pred{\mathbf{Pred}}
\def\res{\mathrm{res}}
\def\de{\mathrm{d}}

\newcommand{\RR}{\mathbb{R}}
\newcommand{\XX}{\mathbb{X}}
\newcommand{\UU}{\mathbb{U}}
\newcommand{\NN}{\mathbb{N}}
\newcommand{\ZZ}{\mathbb{Z}}
\newcommand{\YY}{\mathbb{Y}}

\newcommand{\EE}{\mathbb{E}}

\allowdisplaybreaks

\newcommand{\PX}{\mathcal{P}(\mathbb{X})}
\newcommand{\PXn}{\mathcal{P}(\mathbb{X}_n)}
\newcommand{\PS}{\mathcal{P}(S)}

\title{When can Optimally Trained Transformers Achieve Near-Lossless Token Prediction for non-Markovian Data?}

\author{\name Ka\u{g}an Akman \email kagan.akman@bilkent.edu.tr \\
       \addr Department of Mathematics\\
       Bilkent University\\
       Ankara, 06800, Turkey
       \AND
       \name Jake Feldman Starosta \email jake.feldmanstarosta@queensu.ca\\
       \addr Department of Mathematics and Statistics\\Queen's University\\Kingston K7L 3N6, ON, Canada
       \AND
       \name Naci Saldi \email naci.saldi@bilkent.edu.tr \\
       \addr Department of Mathematics\\
       Bilkent University\\
       Ankara, 06800, Turkey
       \AND
       \name Serdar Y\"{u}ksel \email yuksel@queensu.ca \\
       \addr Department of Mathematics and Statistics\\Queen's University\\Kingston K7L 3N6, ON, Canada}

\begin{document}
\maketitle

\begin{abstract}
There exists little mathematical theory explaining why  Transformer architectures are empirically successful even for non-Markovian data. To this end, this paper studies next-token prediction via stochastic optimal control theoretically trained Transformers for such sources, where we consider either stationary and ergodic non-Markovian sources or those which are hidden Markov models of some finite order (possibly non-stationary). In either case, we first obtain error bounds via finite models with finite window approximations. For stationary ergodic sources, the resulting error is bounded by an entropy gap; for $N$th order Markovian sources, it decays at an exponential rate $\alpha^{\lfloor M/N\rfloor}$, where $M$ is the window length and $\alpha$ is determined by filter stability (and the Dobrushin coefficients of the transition and measurement kernels). We then construct a Transformer of depth $M+1$, one head per layer, realizing the truncated filter, and we optimally train it with stochastic control developed in earlier work, which rigorously returns a near optimal Transformer. In particular, we show that for a non-Markovian source, either under stationarity and ergodicity for general models or under filter stability for hidden Markov models, an optimally trained Transformer has arbitrarily small performance error with increasing context length, data, and approximation refinement. Numerical studies accompany the theory.
\end{abstract}

\section{Introduction}\label{sec:introduction}

Transformer architectures are empirically strong predictors of the next token from a finite context, but how long must that context be? Practice has an answer by means of experiments, and asserts that longer context is often better. 

In this paper, we answer this question in terms of an intrinsic property of the general possibly non-Markovian source by investigating how much such a non-Markovian source can be approximated by an $N$th order finite Markov model, in a suitable sense, together with optimal training loss due to approximation and sample complexity errors. 

A specific application will involve partially observed Markov models (known also as hidden Markov models) for which stronger bounds will be presented in terms of filter stability.

The theory of partially observed Markov process (POMP) tells us that if the source is a latent Markov process, i.e. an unseen Markov process coupled with an observation process obtained via a noisy channel conditioned on the source, then the optimal predictor is a filter. Going one step further, we predict the next observation variable based on the history of previous observation variables, corresponding to the language modeling task of next-token prediction. Using the full history as context is often challenging in practice. Therefore, we truncate the input, that is, initialize it from a (possibly wrong) prior and take the finite window as the context. Then, the predictive performance on this truncated context is a question of filter stability \cite{atar1997exponential, le2004stability, chigansky2004stability,chigansky2009intrinsic}.

On the other hand, practice demonstrates that Transformer architectures are efficient for processing sequential data such as sentences for large language models. \cite{yun2020are} reveal that Transformers are universal approximators of continuous, sequence-to-sequence functions on compact domains. Thus, a natural question arises: {\it Can Transformers approximate finite-window filter processes arising from POMPs?} In this paper, we show that the answer is affirmative and we find quantitative bounds for the approximation error while specifying the class of Transformers achieving such task by explicit width and depth constructions.

On an $N^{\rm th}$ order Markovian model, that is, the law of the next state is determined by the previous $N$ states, we define noisy observation channels in order to make it latent. Then we define the associated prediction process: full history, finite-window and finite-window with quantized observations. There are three major unknowns: (i) Are finitely many quantized observations sufficient statistics? (ii) Can Transformers realize finite-window predictive distribution with quantized observations? (iii) How can we learn the optimal weight trajectory for such Transformer? We answer (i) and (ii) in this paper, whereas (iii) is studied by \cite{akman2026optimal}, where an optimal control theoretic formulation of Transformers is studied with rigorously justified finite model approximations and dynamic programming.

Transformer training amortizes the output cost: paying upfront at training to reduce inference costs compared to an empirical lookup table. This amortization is done by compressing the information from the empirical distribution lookup table to a representation that lives in the model weights, built on the structure of the Transformer architecture. While this compression requires an initial computational cost in the form of training, it saves the user compute and storage by replacing a large input table search with a deterministic equation that approximates the target distribution. This representation has other useful properties as well, such as continuity and generalization, and thus the ability to interpolate sequences absent from the training data set.

We begin by defining a generally non-Markovian process that we segment to create the training data set. In the context of LLMs, this can be thought of as a sequence of tokens on the internet. We then formulate the Transformer dynamics as a control theoretic problem: selecting the optimal model weights (controls), to approximate the empirical next-token distribution (target state) given the input sequence (current state). The sequence flows through the Transformer one layer at a time, each comprised of a feed-forward layer and self-attention block. We employ dynamic programming to train a quantized Transformer model to minimize the cross-entropy between the actual predicted and target empirical distributions. We show that the empirical prediction gets arbitrarily close to the truncated-window quantized prediction. Then, we prove the existence of a Transformer that approximates this truncated-window quantized prediction, which we approximate experimentally with quantized Transformer training.

\subsection{Related Works}\label{sub:related}

\subsubsection*{Control--theoretic approaches to neural networks and, in particular, to Transformers} There is a developing literature focusing on the control--theoretic formulations and properties of Transformers. \cite{weinan2017proposal} suggests to see machine learning problems -in particular, residual neural networks (ResNets)- as optimal control of continuous time dynamical systems governed by ordinary differential equations where the time index corresponds to the layers of a neural network. Around the same time, \cite{haber2018stable} study a similar ODE based analysis of ResNets by proving stability and well-posedness results. They also study discretization of such systems for numerical analysis and implementability. \cite{ruthotto2020deep} categorize convolutional neural networks (CNNs) into PDE inspired classes and authors discuss the regularity of CNNs via the theory of parabolic and hyperbolic PDEs. An alternative approach by seeing neural networks as solutions of suitable dynamical systems, called \emph{neural ordinary differential equations}, is studied by \cite{chen2018neural,lu2018beyond}, where the neural networks structure is replaced by the solution of the governing ODEs, obtained by the Runge-Kutta methods \cite[Chapter~3]{butcher2016numerical}. \cite{ruiz2023neural} study controllability of neural ordinary differential equations while \cite{ruiz2024control} study the controllability of \emph{neural transport equations}, which are continuity equations governed by a neural network's velocity field. 

\cite{geshkovski2025mathematical} study Transformers in a mathematically rigorous way while focusing on clustering properties, building on their earlier work \citep{geshkovski2023emergence}. The training scheme used in this paper is by \cite{akman2026optimal} where Transformers are represented as measure-state Markov decision processes lifted from discrete-time McKean--Vlasov dynamics \citep{mckean1966class} exhibited at the token level, and then, the training is reformulated as a dynamic programming problem. For computational tractability, \cite{akman2026optimal} derive a finite model approximation to the measure-state MDP using the algorithm  of \cite{reznik2011algorithm} and \cite[Chapter~3]{saldi2018finite}. The generalization properties of the quantized training method are studied by \cite{akman2026generalization} from two angles: norm-based excess risk bounds and Wasserstein distributional robustness.

A different use of control theory appears in \cite{chang2025dual}, who construct a Transformer-like architecture for prediction in hidden Markov models through a duality between nonlinear filtering and optimal control. \cite{kudre2026Transformer} extend this to linear-Gaussian models and compare the optimal control weights with the attention weights of a trained Transformer. Our work differs by the usage of control classes; in \cite{chang2025dual}, control describes inference over the full history (and is thus closed-loop or context dependent), whereas we use control for training at the distributional level leading to open-loop (and thus context-independent) parameters consistent with standard Transformer architecture, and quantify the cost of a finite context window via filter stability. 

\subsubsection*{What Transformers can represent}

Transformers are universal approximators of continuous sequence-to-sequence functions on compact domains with fixed width and arbitrary depth \citep{yun2020are} while  \cite{yun2020connections} show that sparse Transformers can preserve such universal approximation capabilities. In contrast, our target, in this paper, is a measure-valued map on a finite quantized domain for which we give explicit depth and width with one head per layer. As another refinement direction, \cite{kajitsuka2024transformers} prove Transformers with one-headed single self-attention having low-rank weight matrices and double feed-forward layers can approximate continuous and permutation-equivariant maps on compact domains. There is a line of work treating the context as a probability measure over tokens; \cite{furuya2025transformers} prove universal approximation for this setup, while \cite{geshkovski2024measure} focuses on Transformer's interpolation ability from measures to measures. \cite{agrachev2025generic} show controllability of general equivariant systems and treat self-attention mechanism as a special case. There are several studies providing explicit approximation rates: \cite{jiang2024approximation} prove Jackson-type rates for sequence-to-sequence matching scenario, while \cite{wang2024understanding} gives approximation rates for Transformers with long but sparse memory, and they characterize how various hyperparameters such as depth, width, and the number of heads contribute.

For $k$-th order fully observable Markov sources, \cite{rajaraman2024Transformers} show that three layers with one head suffice to represent the conditional $k$-gram with an attention-only Transformer. Their constructed network represents a statistic of the prompt, namely, the empirical distribution of successors to the current token with the sequence itself, which is unbiased for a fully observed source and is kernel-agnostic; whereas we represent the predictive distribution of a partially observed source in the weights.

A complementary line bounds what Transformers cannot do. Transformers with hard attention fail to recognize even relatively simple formal languages \citep{hahn2020theoretical}. \cite{merrill2023parallelism} show that a fixed-depth Transformer with limited numerical precision cannot carry out a computation that requires many sequential steps. \cite{strobl2024formal} survey the expressibility of Transformers. 

\subsubsection*{Filter stability and memory in prediction}

The filter stability problem refers to the correction of an incorrectly initialized nonlinear filter for a partially observed stochastic dynamical system (controlled or control-free) with increasing measurements \cite{chigansky2004stability,chigansky2009intrinsic,atar1997exponential,le2004stability}. \cite{mcdonald2020exponential} give an explicit geometric rate for filter stability through Dobrushin coefficient \citep{dobrushin1956central} of the transition kernel and the observation channel of the latent Markov source. Truncating a predictor to a finite window is exactly this situation since discarding the data amounts to replacing the true conditional law of the hidden state by a fixed prior. \cite{kara2022near} use a similar observation to establish near optimality of finite-memory feedback policies in partially observed MDPs, converting the belief-state error into an error of discrepancy in the values; \cite{kara2025near} extend this to continuous spaces. In this paper, we use belief error and propagate it through the one-step prediction kernel instead. Alternative approaches to filter stability also establish exponential filter stability through a different route: \cite{atar1997exponential} and \cite{le2004stability} prove exponential stability rates for the Hilbert projective metric via Birkhoff's contraction coefficient. Alternatively, asymptotic filter stability can also be utilized if one only wishes to have asymptotic performance bounds without explicit rates of convergence \cite{chigansky2009intrinsic,mcdonald2018stability}. An infinite product space approach, where finite memory truncation is viewed as quantization, is presented in \cite{anotherLookPOMDPs}. 

\subsection{Contributions}\label{sub:contributions}

Our main results can be summarized as follows. Given a general non-Markovian model $\{Y_r : r \in \mathbb{Z}_+\}$ taking values from a compact metric space $\mathbb{Y}$, consider its finite approximation $q^{s}(Y_r), r \in \mathbb{Z}_+$ where $|q(\mathbb{Y})| =: R(s) < \infty$ and $q$ is a quantizer. We show that Transformer performance is fundamentally related to finite model with finite memory approximability. 

Such a finite model with finite memory approximability is a highly qualified expression and requires care with regard to its precise presentation. For a general stationary and ergodic model, this term is given by
\[
    \begin{split}
        \sup_{r\geq M}\EE_{Y_{0:r}}\bigg[\|\Pr(Y_{r+1} \in \cdot \mid Y_{0:r}) - \Pr(Y_{r+1} \in \cdot \mid q^s(Y_r), q^s(Y_{r-1}),\ldots,&q^s(Y_{r-M + 1}))\|_\TV\bigg]\\&:= F(M,q^s)
    \end{split}
\]
for the case with stationary and ergodic data. 

We show that as $F(M,q^s) \to 0$ with $M \to \infty, R(s) \to \infty$, Transformers, trained optimally via stochastic control, perform asymptotically well, with explicit convergence bounds. Information theoretic or ergodic theoretic bounds on $F(M,q^s)$ can be utilized in the quantitative expressions.

To make our results stronger and more explicit in a particular context but one with significant applied relevance, we will also consider the special case of hidden Markov models, where there is an underlying Markov process $X_r, r \in \mathbb{Z}_+$ which generates $Y_r$, the observable measurement process, through $Y_r = H(X_r,V_r)$ for some measurable $H$ and i.i.d. noise $V_r$ (whose representation is without loss for such models). In this case, the above approximation bounds can be obtained more explicitly in terms of filter stability; stationarity can be relaxed away and more explicit conditions can be obtained. Without imposing ergodicity, define
\[
    \begin{split}
        \sup_{r\geq M}\EE_{Y_{0:r}}\|\Pr(Y_{r+1} \in \cdot | Y_{0:r}) -  \Pr(Y_{r+1} \in \cdot | q^s(Y_{r}), q^s(Y_{r-1}),\cdots,&q^s(Y_{r-M + 1},\hat{\pi})\|_\TV\\
        &:= \tilde{F}(M,q^s),
    \end{split}
\]
where $\hat{\pi}$ denotes a fixed probability measure serving as an artificial and incorrect prior at time $0$, to facilitate Bayesian updates. 

We show that as $\tilde{F}(M,q^s) \to 0$ with $M \to \infty, R(s) \to \infty$, Transformers, trained via optimal stochastic control \citep{akman2026optimal}, can rigorously be shown to perform prediction and filtering with near lossless performance under filter stability and regularity conditions. We thus address a fundamental question on Transformer architectures and obtain quantified error bounds due to finite model with finite memory approximation for general non-Markovian models. 

Figure \ref{fig:results} summarizes our technical results. Extensive simulations corroborate the mathematical results. We note that our ingredients in the analysis are: (i) approximation via finite memory models (in a suitable, expectation or uniform over priors, sense), (ii) approximation of finite memory models with finite Markov models, (iii) optimal Transformer learning for finite Markov models and their performance, and (iv) a sample complexity analysis. 

\begin{figure}[H]
    \centering
    \includegraphics[width=\linewidth]{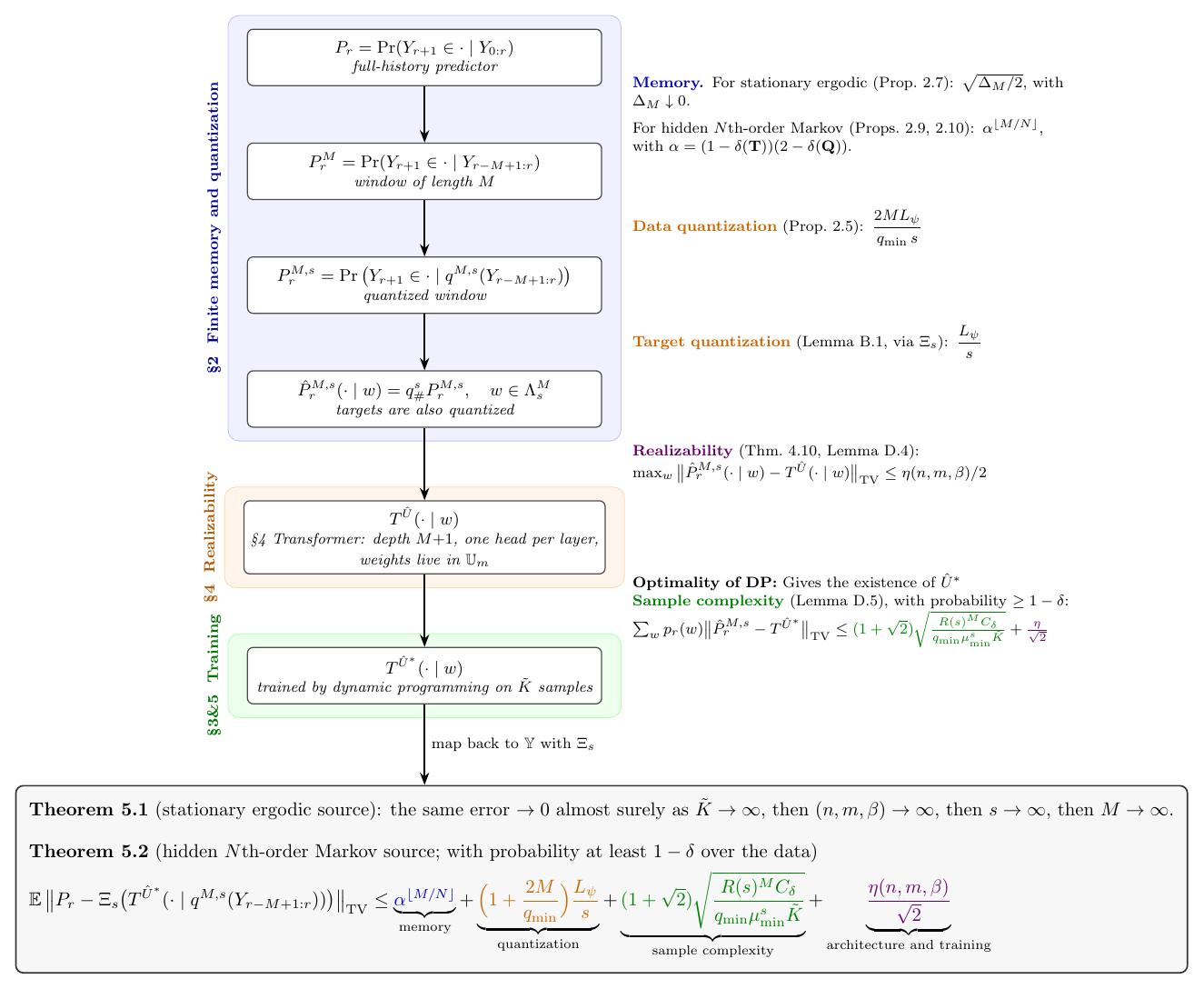}
    \caption{Summary of the results of this paper.}
    \label{fig:results}
\end{figure}

Our analysis also highlights what Transformers can and cannot do: Since our analysis builds on rigorously justified near optimal training via control theory, a corollary of our analysis is that when sequence to sequence approximations is realizable with no error using {\it some} Transformer, an optimally trained scheme can construct a Transformer which attains an arbitrarily small error performance. That is, analogous to standard control theoretic paradigm (see e.g. \cite{anderson2007optimal}), controllability/reachability certifies that under a suitably formulated cost, optimality arrives at arbitrarily small performance error whenever mild regularity conditions hold.

\section{Sufficiency of Finite Memory in Prediction}\label{sec:problem}

Let $S$ be a finite alphabet and $(\YY,\rho)$ be a compact metric space endowed with a reference probability measure $\mu$ defined on the Borel $\sigma$-field generated by $\rho$. Let $\{Z_r\}_{r\in\ZZ_+}$ be a stochastic process taking values in $S$ and $Q\in\cP(\YY\mid S)$ be a stochastic kernel admitting a Radon-Nikodym derivative
\[
    \psi(y\mid z) := \frac{\diff Q(\,\cdot\mid z)}{\diff\mu}(y),\quad y\in\YY,\, z\in S.
\]
Since $S$ is finite, such $\mu$ always exists; we simply take $\mu(\cdot) = \sum_{z\in S}\frac{1}{|S|}Q(\,\cdot\mid z)$. We take observation variables $Y_r \sim Q(\,\cdot\mid Z_r)$ from this channel for all $r\in\ZZ_+$ and we assume that $Y_r$ is independent of $Z_t$ and $Y_t$ whenever $r\neq t$, conditioned on $Z_r$.

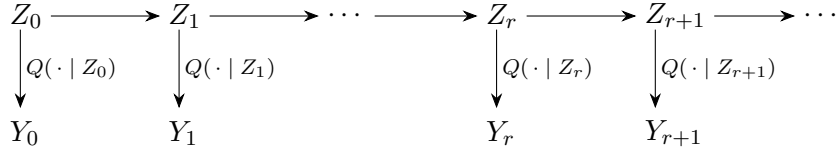
\begin{figure}[H]
    \centering
    \begin{tikzpicture}[
      >={Stealth[length=2mm]},
      ent/.style={anchor=west, inner sep=0pt, xshift={-0.5*width("$Z$")}},
      lab/.style={right, font=\scriptsize, inner sep=2pt}
    ]
      \def\dx{2.1}      
      \def\yY{-1.55}    
      \def\ax{0.09}  
 
      \foreach \lbl [count=\i from 0] in {Z_0, Z_1, \cdots, Z_r, Z_{r+1}, \cdots}
     \node[ent] (z\i) at (\i*\dx, 0) {$\lbl$};
 
      \foreach \i/\lbl in {0/Y_0, 1/Y_1, 3/Y_r, 4/Y_{r+1}}
     \node[ent] (y\i) at (\i*\dx, \yY) {$\lbl$};
 
      \foreach \i [evaluate=\i as \j using int(\i+1)] in {0,...,4}
         \draw[->] ($(z\i.east)+(3pt,0)$ |- 0,\ax) -- ($(z\j.west)+(-3pt,0)$ |- 0,\ax);

      \foreach \i/\lbl in {0/Z_0, 1/Z_1, 3/Z_r, 4/Z_{r+1}}
         \draw[->] (\i*\dx, -0.18) -- (\i*\dx, \yY+0.33)
            node[midway, lab] {$Q(\,\cdot\mid \lbl)$};
    \end{tikzpicture}
    \caption{Flows of the latent process and observations.}
    \label{fig:flow}
\end{figure}

As $\YY$ is a compact space, by total boundedness, for any $s\in\NN$, there exists a $1/s$-net $\Lambda_s := \{y_{s,i} : i=1,\ldots,R(s)\}$, where $R(s)\in\NN$, that is
\begin{equation}
    \label{eq:Quantizer}
    \min_{i=1,\ldots,R(s)}\rho(y, y_{s,i}) \leq \frac{1}{s}
\end{equation}
for all $y\in\YY$.

\begin{definition}\label{def:Quantization}
    The {\it quantizer} $q^s:\YY\to\Lambda_s$ is defined as 
    \[
        q^s(\cdot) := \arg\min_{i=1,\ldots,R(s)}\rho(\cdot,y_{s,i}),
    \] 
    where the ties are broken in a measurable way.
\end{definition}

We separate $\YY$ into $R(s)$-cells represented by the elements of $\Lambda_s$. Say $C(s,i)$ be the cell associated to $y_{s,i}$, that is, $C(s,i) := (q^s)^{-1}(y_{s,i})$. Then $\{C(s,i)\}_{i=1}^{R(s)}$ form a disjoint partition of $\YY$, i.e. $\bigsqcup_{i=1}^{R(s)} C(s,i) = \YY$. Note that $\operatorname{diam}(C(s,i))\leq 2/s$ for all $s\in\NN$ and $i=1,\ldots,R(s)$.

We define four objects: First, we have the true predictive distribution which makes a prediction on the next token based on the full history. This is an object with the highest informative value provided that the underlying structure is well-behaved. However, as we progress further into the time horizon, computing it may require vast computation power and memory requirements. To this end, we also introduce the $M$-window predictive distribution which uses only the $M$-many past observations prior to the next time step. After that we define the quantized version of the $M$-window predictive distribution, which uses the $M$-window data from the $\frac{1}{s}$-net approximation of $\YY$. Finally, since Transformers realize maps via finite set of logits, we quantize the target as well. Eventually, one part of our objective is to reveal that, depending on the window size $M$ and quantization level $s$, $M$-window $s$-quantized data are near-sufficient statistics with explicit error rates.

\begin{definition}\label{def:PredictiveDists}
    For any $r\in\NN$, we define
    \begin{enumerate}
        \item The {\it true predictive distribution} 
        \[
            P_r(\cdot) := \Pr(Y_{r+1}\in\cdot\mid Y_{0:r}).
        \]
        \item The {\it $M$-window predictive distribution}
        \[
            P_r^M(\cdot) := \Pr(Y_{r+1}\in\cdot\mid Y_{r-M+1},\ldots,Y_{r}),
        \]
        \item The {\it $M$-window quantized-data predictive distribution}
        \[
            P_r^{M,s}(\cdot) := \Pr(Y_{r+1}\in\cdot\mid q^s(Y_{r-M+1}),\ldots,q^s(Y_{r})),
        \]
        \item The {\it $M$-window fully quantized predictive distribution}
        \[
            \hat{P}_r^{M,s}(\cdot) := q^s_\#P_r^{M,s}(\,\cdot\mid q^s(Y_{r-M+1}),\ldots,q^s(Y_r)).
        \]
    \end{enumerate}
    In 2-4, we take the initial distribution as $\mathrm{Law}(Z_{r-M})$.
\end{definition}

In the theory of partially observed Markov decision processes (POMDPs), one desires to compute the \emph{filter process} $\pi_r(\cdot) = \Pr(Z_r\in\cdot\mid Y_{0:r})$ \citep{kara2024partially}. Hence, the problem we study in this paper is slightly different from the usual nonlinear filtering problem. A filter is an estimator of hidden states, based on the information given by the observation variables. In this paper, we are trying to estimate the next observation given the observations seen so far. For Markovian processes, these two problems are equivalent since one can arrive at $P_r$ by composing the filter $\pi_r$ with the transition and observation kernels of the system (which are non-expansive in the total variation distance), in that order. However, this is not true in general for an ergodic source.

The analysis of this problem requires several assumptions, which we regard as mild as they balance the level of difficulty imposed by the generality.

\begin{assumption}\label{a:BaseAssumptions}
    The following hold throughout:
    \begin{itemize}
        \item[(A1)] The density $\psi$ satisfies the minorization condition $\psi(y\mid z) \geq q_{\rm min} > 0$ for all $y\in\YY$ and all $z\in S$.
        \smallskip
        \item[(A2)] There exists $L_\psi\geq 0$ such that for all $z\in S$ and all $y,y'\in\YY$, $|\psi(y\mid z) - \psi(y'\mid z)| \leq L_\psi \rho(y, y')$.
        \smallskip
        \item[(A3)] For all $s\in\NN$, there is $\mu_{\rm min}^s > 0$ such that, for all $i=1,\ldots,R(s)$, $\mu(C(s, i)) \geq \mu_{\rm min}^s$ holds. 
    \end{itemize}
\end{assumption}

Once we obtain the necessary results, we need to be able to come back from the quantized space to the original space. In particular, in §\ref{sec:Approximation}, we build a Transformer to realize $\hat{P}_r^{M,s}$ in Definition~\ref{def:PredictiveDists}, using which we obtain the final error bounds in §\ref{sec:Main}. Since the Transformer gives a probability distribution over $\Lambda_s$, we have to revert back to a distribution on $\YY$, namely $P_r^{M,s}$.

\begin{definition}\label{def:RevertQuantization}
    For any $s\in\NN$, define $\Xi_s:\mathcal P(\Lambda_s) \to \mathcal P(\YY)$ by
    \[
        \Xi_s(\nu)(A) := \sum_{i=1}^{R(s)}\nu(y_{s,i})\frac{\mu(A \cap C(s,i))}{\mu(C(s,i))},\qquad \nu\in\mathcal P(\Lambda_s).
    \]
\end{definition}

The function $\Xi_s$ is well-defined if Assumption~\ref{a:BaseAssumptions}-(A3) holds: Then the denominator $\mu(C(s,i))$ is positive for all $s\in\NN$ and $i=1,\ldots,R(s)$.

We study this problem under two scenarios: (i) We add the assumption that the latent source $\{Z_r\}_{r\in\ZZ_+}$ is stationary ergodic, in addition to Assumption~\ref{a:BaseAssumptions}. (ii) $\{Z_r\}_{r\in\ZZ_+}$ is $N^{\rm th}$ order Markov, which we initially treat the first order case, then generalize to the $N^{\rm th}$ order case by reducing to the first order case back. This setting allows us to derive explicit rates showing exponentially fast convergence using Dobrushin coefficient-based analyses, following \citep{mcdonald2020exponential}.

\begin{proposition}\label{prop:Quantization}
    Under Assumptions~\ref{a:BaseAssumptions}-(A1-2), for all $r\geq M$ and all $s\in\NN$,
    \[
        \sup_{\mathbf{y}\in\YY^M}\|P_r^M(\,\cdot\mid\mathbf{y}) - P_r^{M,s}(\,\cdot\mid\mathbf{q}^{M,s}(\mathbf{y}))\|_\TV \leq \frac{2ML_\psi}{q_{\rm min}s}
    \]
    where $q^{M,s}$ is the application of $q^s$ to sequences of length $M$.
\end{proposition}

\subsection{The case where $\{Z_r\}_{r\in\ZZ_+}$ is a general stationary and ergodic process}
In this part, we cover the case where the latent source is stationary and ergodic. This setting allows us to derive vanishing error bounds using information--theoretic tools. However, the bounds are of qualitative nature. We commence with stating our base assumption for this subsection, and then we investigate how the regularity of the latent source affects the regularity of the observed variables. Finally, we show that the finite memory approximation yields a entropy gap based bound decreasing to $0$.

\begin{assumption}\label{a:Ergodicity}
    $\{Z_r\}_{r\in\ZZ_+}$ is ergodic with probability one, i.e.
    \[
        \lim_{T\to+\infty} \frac{1}{T} \sum_{r=0}^{T-1} f(Z_r) \overset{a.s.}{=} \sum_{z\in S} \pi_Z(z)\,f(z)
    \]
    for any $f:S\to\RR$. We further assume that $\{Z_r\}_{r\in\ZZ_+}$ is stationary.
\end{assumption}

Let $\mathcal A$ and $\mathcal B$ be measurable spaces, $A$ and $B$ be two continuous random variables with a joint probability density function $f:\mathcal A\times\mathcal B\to[0,1]$, where $A$ is $\mathcal A$-valued and $B$ is $\mathcal B$-valued. Then the \emph{conditional differential entropy is defined as}
\begin{equation}\label{eq:CondEntropy}
    H(A\mid B) := -\int_{\mathcal A\times\mathcal B}f(a,b) \log{f(a\mid b)}\,\de(a,b)
\end{equation}
whenever the integral above exists, where $f(a\mid b)$ is the conditional probability of $A$ given $B$. The \emph{memory gap} is then defined as
\begin{equation}\label{eq:MemoryGap}
    \Delta_{M,r} := H(Y_{r+1}\mid Y_{r-M+1:r}) - H(Y_{r+1}\mid Y_{0:r})
\end{equation}
Note that for fixed $r$, $\Delta_{M,r}\to 0$ monotonically as $M$ increases to $r$ since the conditioning $\sigma$-field grows monotonically to the full history $\sigma$-field generated by $Y_{0:r}$. In addition, for the proof of next lemma, we require the notation $H_j := H(Y_j\mid Y_{0:j-1})$ for all $j\in\NN$ and define $H_\infty := \lim_{j\to\infty} H_j$. Note that Assumptions~\ref{a:BaseAssumptions}-(A2) and (A3) ensure that the conditional entropies in \eqref{eq:MemoryGap} and $H_j$'s exist since $\psi(\,\cdot \mid z)$ is (Lipschitz) continuous on a compact space $\YY$, therefore, it is bounded. Hence, $H_j \in [-\log{\|\psi\|_\infty}, -\log{q_{\rm min}}]$. Note that $\{H_j\}_{j\in\NN}$ forms a non-increasing sequence:
\[
    H_{j+1} := H(Y_{j+1}\mid Y_{0:j}) \leq H(Y_{j+1}\mid Y_{1:j}) = H(Y_j\mid Y_{0:j-1}) =: H_j,\quad j\in\NN
\]
The inequality $H(Y_{j+1}\mid Y_{0:j}) \leq H(Y_{j+1}\mid Y_{1:j})$ holds for all $j\in\NN$ since conditioning on a coarser $\sigma$-field results in higher entropy. The equality $H(Y_{j+1}\mid Y_{1:j}) = H(Y_j\mid Y_{0:j-1})$ is due to the stationarity of $\{Y_r\}_{r\in\ZZ_+}$, established in Lemma~\ref{lem:ErgodicityIsInherited}. Therefore, by monotone convergence theorem for real-valued sequences, $H_\infty$ exists. 

We define the \emph{Kullback-Leibler divergence} between two probability measures $P_1$ and $P_2$ on a measurable space $(\Omega, \cF)$ as follows:
\begin{equation}\label{eq:KL}
    D_{\rm KL}(P_1\parallel P_2) := \begin{cases}
        \int_{\Omega} \de P_1 \log{\frac{\de P_1}{\de P_2}}, & P_1\ll P_2;\\
        +\infty, & \text{otherwise}.
    \end{cases}   
\end{equation}

In addition, recall that the \emph{mutual information} of two random variables $A$ and $B$, defined on the same probability space, is defined as
\begin{equation}\label{eq:MutualInfo}
     I(A;B) := D_{\rm KL}(\mathrm{Law}(A,B)\parallel\mathrm{Law}(A)\otimes\mathrm{Law}(B)).
\end{equation}
An additional third random variable $C$ allows us to define the \emph{conditional mutual information}:
\begin{equation}\label{eq:CondMutualInfo}
     I(A;B\mid C) := \EE_C D_{\rm KL}(\mathrm{Law}(A,B\mid C)\parallel\mathrm{Law}(A\mid C)\otimes\mathrm{Law}(B\mid C)).
\end{equation}
It can be shown that the conditional mutual information satisfies
\begin{equation}\label{eq:CondMutualInfoIdentity}
    I(A;B\mid C) = H(A\mid C) - H(A\mid B,C).
\end{equation}
Let $\zeta_1$ and $\zeta_2$ be probability measures on a measurable space $(\Omega,\cF)$. The {\em total variation distance} is defined as $\|\zeta_1 - \zeta_2\|_{\rm\scriptscriptstyle TV} := \sup_{A\in\cF}|\zeta_1(A) - \zeta_2(A)|$. 

\begin{proposition}\label{prop:Info}
    Under Assumption~\ref{a:Ergodicity}, $\Delta_M := \sup_{r\geq M}\Delta_{M,r} = H_M - H_\infty$, and $\Delta_M\downarrow 0$ as $M\to\infty$. Consequently, $\sup_{r\geq M}\EE_{Y_{0:r}}\|P_r - P_r^M\|_\TV \leq \sqrt{\Delta_M/2} \downarrow 0$.
\end{proposition}

\subsection{The hidden Markov model setup: $\{Z_r\}_{r\in\ZZ_+}$ is Markov but not necessarily stationary}

\subsubsection{The case where $\{Z_r\}_{r\in\ZZ_+}$ is first order Markov}\label{sub:Markov}
We assume that $\{Z_r\} _ {r \in \ZZ_+}$ is a time-homogeneous Markov chain, such that

\begin{equation}\label{eq:1Order}
    \Pr (Z_{r + 1} \in \cdot \mid Z_{0:r}  = \ z_{0:r}) = \Pr (Z_{r + 1} \in \cdot \mid Z_r= z_r ) =: \kappa(\,\cdot \mid z_r)
\end{equation}
for all $r \in \ZZ_+$ and $z_{0:r} \in S^{r+1}$. We have the observation channel $Q$ and sequence $\{Y_r\}_{r \in \ZZ _+}$ as before.

The filter stability bounds derived by \cite{mcdonald2020exponential} uses Dobrushin's coefficient \citep{dobrushin1956central}, thereby, it is a centerpiece to our analysis as well. In the following, we define Dobrushin coefficient of a stochastic kernel between general measurable spaces in the \emph{minorization form}.

\begin{definition}\label{def:Dobrushin}
    Dobrushin coefficient of a stochastic kernel $K$ on a measurable space $\mathsf{B}$ given another measurable space $\mathsf{A}$ is defined as
    \[
        \delta(K) := \inf \sum_{i=1}^n \min\{K(B_i \mid a), K(B_i \mid a')\}
    \] 
    where the infimum is taken over all $a,a'\in\mathsf{A}$ and all finite measurable partitions $\{B_i\}_{i=1}^n$ of $\mathsf{B}$.
\end{definition}

Informally, the Dobrushin coeffient of a kernel $K$ can be interpreted as the worst-case one-step probability of forgetting where you started. It takes values in $[0,1]$ where $1$ is often interpreted as total forgetting in one-step and $0$ tells us that there is no forced forgetting. For any two probability measures $\zeta_1,\zeta_2\in\mathcal{P}(\mathsf{B})$, one has the Lipschitz property $\|\zeta_1K - \zeta_2K\|_{\rm\scriptscriptstyle TV} \leq (1-\delta(K))\|\zeta_1 - \zeta_2\|_{\rm\scriptscriptstyle TV}$.

In the following proposition, we prove the exponential filter stability bounds for the next token predictive conditional distributions with full and truncated data, using \cite[Corollary~3.7]{mcdonald2020exponential}. We show that the difference is bounded by an exponential rate depending on the Dobrushin coefficients of the transition kernel $\kappa$ and the observation channel $Q$.

\begin{proposition}\label{prop:MarkovDobrushin}
    Suppose Assumption~\ref{a:BaseAssumptions}-(A1) holds. Then, whenever $r \geq M\geq 1$, we have
    \[
        \EE_{Y_{0:r}}\,\|P_r-P_r^M\|_{\rm\scriptscriptstyle TV} \leq \alpha^M
    \]
    with $\alpha = (1-\delta(\kappa))(2-\delta(Q))$.
\end{proposition}

Under suitable conditions, for which an example would be $\delta(\kappa) > 1/2$, it can be shown that the rate in Proposition~\ref{prop:MarkovDobrushin} gives an exponential decay. In general, the faster convergence requires both $\delta(\kappa)$ and $\delta(Q)$ to be near $1$. This is in parallel with the intuitive meaning of Dobrushin's coefficient: Whenever it is close to $1$, the system forgets where it started, hence, the truncated data filter converges to the full history filter.

\subsubsection{The case where $\{Z_r\}_{r\in\ZZ_+}$ is $N^{\rm th}$ order Markov}\label{sub:NOrderMarkov}

Now, we assume that $\{Z_r\}_{r\in\ZZ_+}$ is an $N^{\rm th}$ order Markov chain, that is
\begin{equation}\label{eq:Norder}
    \Pr(Z_{r+1}\in\cdot\mid Z_{0:r} = z_{0:r}) = \Pr(Z_{r+1}\in\cdot\mid Z_{r-N+1:r} = z_{r-N+1:r})
\end{equation}
for all $r\in\ZZ_+$ and all $z_{0:r}\in S^{r+1}$. Then we define the quantity on the right hand side in \eqref{eq:Norder}, $\boldsymbol{\kappa}(\,\cdot\mid z_{r-N+1:r})$ which defines a conditional probability distribution on $S$ given $S^N$. We call $\boldsymbol{\kappa}$ the \emph{transition kernel} of $\{Z_r\}_{r\in\ZZ_+}$. As before, $\{Y_r\}_{r\in\ZZ_+}$ is the induced observation variables emitted through the channel $Q$. The objects in Definition~\ref{def:PredictiveDists} transfer directly to here as well. In this setting, $\{(Z_r, Y_r)\}_{r\in\ZZ_+}$ defines a hidden (partially observed) $N^{\rm th}$ order Markov model.

We use a block hidden Markov model that we derive from the original $N^{\rm th}$-order partially observed Markovian process as an apparatus for enabling us to use the machinery provided by \cite{mcdonald2020exponential}. To this end, we fix a time index $r\in\NN$ and a window length $M$ that is larger than $N$ throughout. Moreover, we define $m:=\lfloor M/N\rfloor$ and $\res := r - mN$, where the former serves as the time-index for the block process and the latter is the residual time-index due to block lifting.

For each $i=1,\ldots,m$, we define two processes:
\begin{equation}
    \label{eq:block_state}
    \mathbf X_i := (Z_{\res + (i-1)N + 1},\ldots,Z_{\res + iN})\in S^N
\end{equation}
and
\begin{equation}
    \label{eq:block_obs}
    \mathbf Y_i := (Y_{\res + (i-1)N + 1},\ldots,Y_{\res + iN})\in \YY^N
\end{equation}
with $\mathbf{X}_0 := (Z_{1},\ldots, Z_\res)$. The indices are chosen in a way that $\mathbf X_m = (Z_{r - N + 1}, \ldots, Z_r)$ and $\mathbf Y_{1:m}$ lies inside $\{Y_{r-M+1},\ldots,Y_r\}$. In addition, we have $M-mN < N$ observations as leftover. For these processes, we define three kernels: For any $\mathbf{z}=(z_1,\ldots,z_N), \mathbf{z}'=(z_1',\ldots,z_N')\in S^N$ and $\mathbf{y}=(y_1,\ldots,y_N) \in \YY^N$
\begin{itemize}
    \item[(K1)] The {\bf block transition kernel}
    \[
        \mathbf{T}(\mathbf{z}'\mid\mathbf{z}) := \prod_{j=1}^N\boldsymbol{\kappa}(z_j'\mid z_j,\ldots,z_N,z_1',\ldots,z_{j-1}')\in\mathcal{P}(S^N\mid S^N).
    \]
    \item[(K2)] The {\bf block observation kernel}
    \[
        \mathbf{Q}(\mathbf{y}\mid\mathbf{z}) := \prod_{j=1}^N Q(y_j\mid z_j)\in\mathcal{P}(\YY^N\mid S^N).
    \]
    \item[(K3)] The {\bf one-step prediction kernel}
    \[
        \Pred(y\mid\mathbf{z})=\sum_{z\in S}\psi(y\mid z)\,\boldsymbol{\kappa}(z\mid\mathbf{z})\in\mathcal{P}(\YY\mid S^N),\quad y\in\YY.
    \]
\end{itemize}

Our proof strategy is to reduce the $N^{\rm th}$ order Markov chain to a first order one with the block lifting, which is now a known territory to us.

\begin{figure}[H]
\centering
\begin{tikzpicture}[
  zn/.style  = {draw, rounded corners=2pt, minimum width=1.25cm,
                minimum height=7mm, inner sep=1pt, font=\footnotesize},
  el/.style  = {minimum width=1.25cm, minimum height=7mm, font=\footnotesize},
  brc/.style = {decorate, decoration={brace, amplitude=4pt, mirror}, semithick},
  wlab/.style= {font=\footnotesize},
  cmt/.style = {font=\footnotesize, align=left},
  x=1.45cm
]

\node[font=\small\bfseries, anchor=west] at (0.25,1.80)
  {Sliding lift $\mathbf X_r=(Z_{r-N+1},\dots,Z_r)$};

\node[zn] (a1) at (1,0) {$Z_{1}$};
\node[zn] (a2) at (2,0) {$Z_{2}$};
\node[el] (a3) at (3,0) {$\cdots$};
\node[zn] (a4) at (4,0) {$Z_{N}$};
\node[zn] (a5) at (5,0) {$Z_{N+1}$};

\begin{scope}[on background layer]
  \fill[black!12, rounded corners=3pt]
    ($(a2.north west)+(-3pt,3pt)$) rectangle ($(a4.south east)+(3pt,-3pt)$);
\end{scope}
\node[wlab, anchor=south] at ($(a3.north)+(0,0.28)$) {$N-1$ shared coordinates};

\draw[brc] ($(a1.south west)+(-2pt,-0.20)$) -- ($(a4.south east)+(2pt,-0.20)$)
  node[midway, below=7pt, wlab] {$\mathbf X_1$};
\draw[brc] ($(a2.south west)+(-2pt,-1.05)$) -- ($(a5.south east)+(2pt,-1.05)$)
  node[midway, below=7pt, wlab] {$\mathbf X_2$};

\begin{scope}[yshift=-4.9cm]

\node[font=\small\bfseries, anchor=west] at (0.25,1.75)
  {Block lift $\mathbf X_i=(Z_{\res + (i-1)N+1},\dots,Z_{\res+iN})$};

\node[zn] (b1) at (1,0) {$Z_{1}$};
\node[el] (b2) at (2,0) {$\cdots$};
\node[zn] (b3) at (3,0) {$Z_{N}$};
\node[zn] (b4) at (4,0) {$Z_{N+1}$};
\node[el] (b5) at (5,0) {$\cdots$};
\node[zn] (b6) at (6,0) {$Z_{2N}$};
\node[el] (b7) at (7,0) {$\cdots$};

\draw[brc] ($(b1.south west)+(-2pt,-0.20)$) -- ($(b3.south east)+(2pt,-0.20)$)
  node[midway, below=7pt, wlab] {$\mathbf X_{1}$};
\draw[brc] ($(b4.south west)+(-2pt,-0.20)$) -- ($(b6.south east)+(2pt,-0.20)$)
  node[midway, below=7pt, wlab] {$\mathbf X_{2}$};

\draw[-latex, semithick] ($(b2.north)+(0,0.32)$) to[bend left=20]
  node[above, wlab] {$\mathbf T$} ($(b5.north)+(0,0.32)$);
\end{scope}
\end{tikzpicture}

\caption{Two ways of lifting an $N^{\rm th}$ order Markov chain to a first order one. Sliding lift makes the process share coordinates as it evolves, therefore, resulting in zero Dobrushin coefficient. The block lifting solves this issue.}
\label{fig:lifts}
\end{figure}
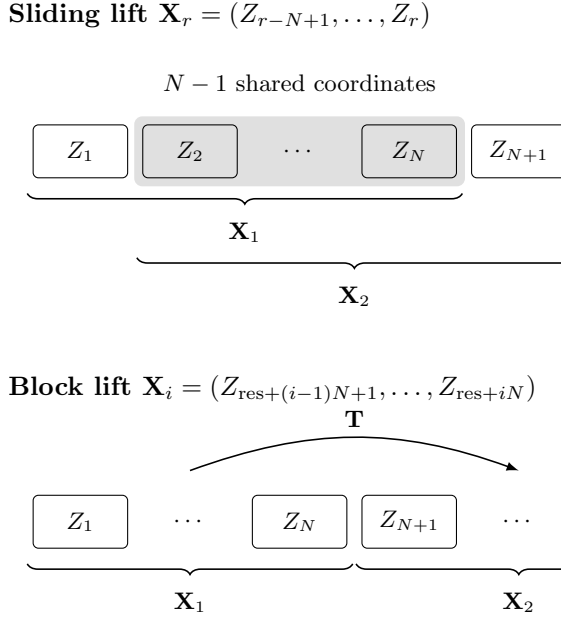

The following proof is a reiteration of the proof of Proposition~\ref{prop:MarkovDobrushin} for the block Markov reduction of the $N^{\rm th}$ order Markov model with additional worked out details.

\begin{proposition}\label{prop:NMarkovDobrushin}
    Suppose Assumptions~\ref{a:BaseAssumptions}-(A1, A2) hold. Then, whenever $r \geq M \geq N \geq 1$, we have
    \[
        \EE_{Y_{0:r}}\,\|P_r-P_r^M\|_{\rm\scriptscriptstyle TV} \leq \boldsymbol{\alpha}^{\lfloor M/N\rfloor}
    \]
    with $\boldsymbol{\alpha} := (1-\delta(\mathbf{T}))(2-\delta(\mathbf{Q}))$.
\end{proposition}

\begin{remark}
    It is also possible to block by taking the {\it sliding lift} $\tilde{\mathbf{X}}_r = (Z_{r-N+1},\ldots, Z_r)$ for each $r\in\ZZ_+$. However, the transition kernel is slightly {\it degenerate} in this case: For all $r\in\ZZ_+$, $\tilde{\mathbf{X}}_{r}$ contains all information in $\tilde{\mathbf{X}}_{r+1}$ except the last component. Therefore, the target supports are disjoint whenever we move on from two states differing in a non-terminal coordinate, resulting in $0$ Dobrushin coefficient for the respective kernel. In this case, we lose the contraction property of the kernel. Non-overlapping blocks of length exactly $N$ are the smallest block lift possible with positive Dobrushin coefficient. This is the reason why we have the exponent $\lfloor M/N\rfloor$ instead of $M$ as in the estimates of \cite{mcdonald2020exponential,kara2022near,kara2025near}.
\end{remark}

\subsubsection{Filter stability and sufficiency of finite memory}\label{sub:merger} We finally merge the findings of §\ref{sub:Markov} and §\ref{sub:NOrderMarkov} to fully characterize the effect of memory on next token prediction.

\begin{theorem}\label{thm:Main}
    Under Assumptions~\ref{a:BaseAssumptions}-(A1, A2), for all $M\geq N$ and all $s\in\NN$, we have
    \begin{itemize}
        \item[(i)] with $\alpha = (1-\delta(\kappa))(2-\delta(Q))$,
        \[
            \sup_{r\geq M}\EE_{Y_{0:r}}\|P_r - P_r^{M,s}\|_\TV \leq \alpha^M + \frac{2ML_\psi}{q_{\rm min}s}.
        \]

        \item[(ii)] with $\boldsymbol{\alpha} = (1-\delta(\mathbf{T}))(2-\delta(\mathbf{Q}))$,
        \[
            \sup_{r\geq M}\EE_{Y_{0:r}}\|P_r - P_r^{M,s}\|_\TV \leq\boldsymbol{\alpha}^{\lfloor M/N\rfloor} + \frac{2ML_\psi}{q_{\rm min}s}.
        \]
        
    \end{itemize}
    
\end{theorem}

Theorem~\ref{thm:Main}'s findings raise two naturally arising questions: (i) How much memory is sufficient for next token estimation? (ii) How fine must the quantization be?

\begin{corollary}\label{cor:MemoryAndQuantization}
    Under Assumptions~\ref{a:BaseAssumptions}-(A1, A2), for any precision $\epsilon>0$ if $s \geq 4ML_\psi/(q_{\rm min}\epsilon)$ and, in addition
    \begin{itemize}
        \item[(i)]  $\alpha < 1$ holds, and whenever
        \[
            M \geq \frac{\log{(2/\epsilon)}}{\log(1/\alpha)},
        \]
        holds, where $\alpha$ as in Proposition~\ref{prop:MarkovDobrushin},
        \item[(ii)] $\boldsymbol{\alpha} < 1$ holds, and whenever
        \[
           M \geq N\bigg\lceil\frac{\log{(2/\epsilon)}}{\log(1/\boldsymbol{\alpha})}\bigg\rceil,
        \]
        holds, where $\boldsymbol{\alpha}$ as in Proposition~\ref{prop:NMarkovDobrushin},
    \end{itemize}
    we get
    \[
        \sup_{r\geq M}\EE_{Y_{0:r}}\|P_r - P_r^{M,s}\|_\TV \leq \epsilon.
    \]
\end{corollary}

\begin{remark}
    All results in this section use predictive distributions in Definition~\ref{def:PredictiveDists} with the prior $\mathrm{Law}(Z_{r - M})$. However, it is possible to take an arbitrary prior $\tilde\pi_0 \in \cP(S)$; Assumption~\ref{a:BaseAssumptions} ensures that the conditions of \cite[Theorem~3.6]{mcdonald2020exponential} hold for $\tilde\pi_0$, and the same analysis applies for the predictive distributions initialized from $\tilde\pi_0$.
\end{remark}

\section{The Optimal Control Formulation of Transformers}\label{sec:OCFormulation}

We follow \cite{akman2026optimal}'s formulation; however, with a divergence: we do not consider positional encoding separately in the first lift, but rather encode it in the token embedding itself. This depicts how position is used in early Transformers (now it is done at the attention matrix level, see \citep{barbero2024round}, which requires generating a token from a static 'vocabulary' and then embedding it with positional information.

\subsection{Transformers as Discrete-Time Dynamical Systems with Open-Loop Control} We recall the definitions and conventions of \cite{akman2026optimal} for a self-contained discussion. Let $\XX$ be the token embedding space and $\UU$ be the space of network weights in one Transformer layer. We define the particle flow function $f: \XX  \times \PX \times \UU\to \XX$ for all $x \in \XX$, $\mu \in \PX$, and $U = (W_W, W_A, W_b, W_Q, W_K, W_V, E_E, E_U)\in\UU$ by
\[
    f(x, \mu, U) := W_W \sigma(W_Ax + W_b) + \dfrac{\displaystyle\int _ \XX \mu(\mathrm{d}z)\,e^{\beta \langle W_Qx, W_Kz \rangle}\, W_Vz}{\displaystyle\int _ \XX \mu(\mathrm{d}z)\,e^{\beta \langle W_Qx, W_Kz \rangle} }
\]
for a sharpening parameter $\beta \geq 0$. The feed-forward block is the left term, and the attention block is the right term. The number of layers is denoted by $T\in\NN$ where $\cT := \{0,\ldots, T-1\}$ and $\overline{\cT} := \cT \cup \{T\}$.

We encode the position $p \in \mathrm{PE}_{L}$ into $s \in S$ via $E_E:S\times \mathrm{PE}_{L} \to \XX$ to produce $x \in \XX$ where $L\in\NN$ is the input sequence length. Henceforth, we assume that positional information is embedded and retrievable in $x^i$ via $E_E$ for $i \in \mathcal L := \{1,\ldots,L\}$ including under transitions. That is, we assume the following:

\begin{assumption} \label{assumption:B1} There exists an unembedding function $E_U : \XX \to \RR ^ {|\Lambda_s|} \times \mathrm{PE}_{L}$ such that $E_U \circ E_E  = {\text{id}_{S\times \mathrm{PE}_{L}}}$ and $E_U(f(x, \mu, U)) = (., p) \iff E_U(x) = (., p)$.
\end{assumption}

In practice, this may correspond to a positional encoder and decoder where the particle flow function sufficiently preserves proximity to the original positional encoding such that projecting it onto the image of the embedded vocabulary $E_E(S\times \mathrm{PE}_{L})$ in $\XX$ retrieves the decoded position.

For our realization, for any $t\in\overline{\cT}$, we take $x^i_t \in \XX$, $U_t \in \UU$, and define $\mu_t := \frac{1}{L} \sum_{i=1}^L \delta_{x_t^i} \in \PX$ for $i \in \mathcal L$ and $t \in \mathcal T$. The weights of $U_t$ are specified as $W_{W,t} \in \RR^{d \times d_1}, W_{A,t} \in \RR^{d_1 \times d}, W_{b,t} \in \RR ^{d_1}, W_{Q,t}, W_{K,t} \in \RR ^ {d_2 \times d}, W_{V,t} \in \RR^{d \times d}$ for $d_1, d_2 \in \NN$. The first layer $U_0$ additionally includes the embedding matrix $E_E \in \RR ^ {d \times |\Lambda_s|}$ and the last layer $U_T$ includes the unembedding matrix $E_U \in \RR ^ {|\Lambda_s| \times d}$. The particle flow function represents the evolution of a single token embedding, such that
\begin{align*}
f(x_t^i, \mu_t, U_t) &= W_t \sigma(W_{A,t} x_t^i + W_{b,t}) + \frac{\int _\XX \mu_t(\diff z)e^{\beta \langle W_{Q,t} x_t^i, W_{K,t} z \rangle} W_{V,t} z}{\int _ \XX \mu_t(\diff z)e^{\beta \langle W_{Q,t} x_t^i, W_{K,t} z \rangle} }\\
&=W_t \sigma(W_{A,t} x_t^i + W_{b,t}) + \frac{\sum _{j = 1}^L e^{\beta \langle W_{Q,t} x_t^i, W_{K,t} x^{j, k}_t \rangle } W_{V,t} x^{j, k}_t}{\sum _{l = 1}^{L} e^{\beta \langle W_{Q,t} x_t^i, W_{K,t} x^{l, k}_t \rangle }} = x^i_{t + 1}.
\end{align*}

We define the \emph{measure flow map} $\Phi : \PX \times \UU \to \PX$ such that
\begin{equation}\label{eq:MeasureFlowMap}
    \Phi(\mu, U) := \mu \circ T^{-1}_{\mu, U}
\end{equation}
where $T_{\mu, U} : \XX \to \XX$ is given by $T_{\mu, U}(x) = f(x, \mu, U)$ for all $x \in \XX$. A change-of-measures argument yields that $\Phi(\mu_t, U) = \mu_{t+1}$ for all $t\in\mathcal{T}$ and $k\in\mathcal{K}$. This shows that $\Phi$ is indeed the right tool to describe measure-level Transformer dynamics. Under compactness assumptions that can be realizable via practical tools, $\Phi$ is jointly continuous on its domain. Moreover, $\{\mu_t\}_{t\in\overline{\cT}}$ defines a weak Feller Markov decision process, that is, the process $\{(\mu_t, U_t) : t\in\cT\}\cup\{\mu_T\}$ is controlled Markov chain with state space $\PX$, action space $\UU$, and with transition kernel
\[
    \eta(\,\cdot\mid\mu, U) := \delta_{\Phi(\mu, U)}(\cdot)
\]
which satisfies: For any continuous and bounded $c:\PX\to\RR$, the function
\[
    \PX\times\UU\ni(\mu, U)\mapsto\int_{\PX} \eta(\mathrm{d}\mu' \mid \mu, U)\,c(\mu')
\]
is jointly continuous on $\PX\times\UU$, where $\PX$ is given the weak$^*$-topology, metrized by the $2$-Wasserstein loss
\begin{equation}\label{eq:2Wasserstein}
    W_2(P, Q) = \inf _{\pi \in \Pi (P, Q)} \int _{S \times S}\|x-y\|_2^2\,\pi(\mathrm{d}x, \mathrm{d}y)
\end{equation}
where $\Pi (P, Q) := \{\zeta \in \mathcal{P}(S \times S) : \zeta(S \times\cdot) = P(\cdot) \text{ and } \zeta(\cdot\times S) = Q(\cdot)\}$.

Note that the Transformer dynamics define a deterministic Markov chain. The noise in Transformer training, or in general, neural network training is due to mini-batching process.

\subsection{Data Set Construction} We generate $\tilde K$ sequences of length $L+1$ from the quantized process $\{q^s(Y_r)\}_{r\in\ZZ_+}$. We split each sequence into a training pair by taking $\bm{x}^k = (x^{1, k}, \dots, x^{L, k})$ as the first $L$ elements of the sequence and $\tilde y^k$ as the last element, for $k \in \tilde {\mathcal{K}} = \{1, \dots , \tilde K\}$. We aggregate these prompts to an initial single state-labeled data set, given by
$$\tilde D = \{(\bm{x}^k, \tilde y^k):k \in \tilde {\mathcal{K}} \} \in (\Lambda_s^L \times \Lambda_s)^{\tilde K}$$
where $\bm{x}^k$ is the feature vector and $\tilde{y}^k \in \Lambda_s$ is the target token embedding. The initial data set $\tilde D$ has training pairs indexed by $\tilde {\mathcal{K}}$, with possibly some duplicate feature vectors, i.e., $\bm{x}^j = \bm{x}^l$ for $j\neq l$, and $j, l \in \tilde {\mathcal{K}}$.

We construct a new data set $D$ from $\tilde D$ with unique feature vectors and labels as probability measures instead of individual token embeddings.
We index the unique feature vectors with $\mathcal{K} = \{1, \dots, K\}$, picking one representative index from $\tilde {\mathcal{K}}$ for each unique feature vector. For $k \in \mathcal{K}$, we define the empirical measure $y^k \in \mathcal P(\Lambda_s)$ conditioned on $\bm{x}^k$ as
\[
    y^k = \frac{\sum_{l=1}^{\tilde {K}} \delta_{\tilde y^l} \mathbf{1}_{\{\bm{x}^l = \bm{x}^k\}}}{\sum_{l=1}^{\tilde {K}} \mathbf{1}_{\{\bm{x}^l = \bm{x}^k\}}}
\]

Then $D = \{(\bm{x}^k, y^k):k \in \mathcal{K}\} \in (\Lambda_s^L \times \mathcal P(\Lambda_s))^K$ has unique training samples $\bm{x}^k$, with labels $y^k$ as empirical probability measures. Note that the denominator is non-zero since it is the count of samples in $\hat D$ with feature vector $\bm{x}^k$.

A distinction between this data set construction and that of most large language models (LLMs) is that here we exclude elements of generated process appearing after the target label before training at the data set level, whereas most current pre-trained LLMs use the technique of masking of \cite{devlin2019bertpretrainingdeepbidirectional} during training, see \citep{zhao2024analysing}.

We then construct $\mu_0^k := \frac{1}{L}\sum_{i=1}^L \delta_{x^{i,k}}$ and push it through the time horizon via lifted Transformer dynamics $\Phi$ for all $k\in\cK := \{1,\ldots, K\}$. We write $\{\mu_t^k\}_{\overline{\cT}}$ for the resulting process. If we concatenate all of these processes over the data set, then we obtain another MDP with state process $\{\boldsymbol{\mu}_t\}_{t\in\overline{\cT}} = \{(\mu_t^1,\ldots,\mu_t^K)\}_{t\in\overline{\cT}}$ for which the state space is $\cP(\XX)^K$, the action space is $\UU$ and the transition kernel is
\begin{equation}\label{eq:EtaKernel}
    \boldsymbol{\eta}^K(\,\cdot\mid \boldsymbol{\mu}_t, U) := \delta_{\boldsymbol{\Phi}(\boldsymbol{\mu}_t, U)}(\cdot), \quad t\in\cT.
\end{equation}
This can be regarded as an ensemble control problem \citep{li2009ensemble}.

\subsection{The Loss Structure}
Having defined the particle dynamics and the data set $D$, we now specify how to train the policy $\gamma$. The Transformer outputs an unembedding score vector $E_Ux^{k, L}_{T, \gamma} \in \RR ^{|\Lambda_s|}$ for each prompt $\bm x^k$. We convert this into a probability measure and penalize its distance from the empirical label $y^k \in \mathcal P (\Lambda_s)$.

We employ the softmax function to obtain probability distributions from our predicted unembedding scores.  For a vector $z \in \RR^{|\Lambda_s|}$, $\text{softmax}: \RR^{|\Lambda_s|} \to \mathcal{P}(\Lambda_s)$ maps $z$ onto the probability simplex as
\[
    \mathrm{softmax}(z) = \left(\frac{e^{z_1}}{\sum _ {j = 1} ^ {|\Lambda_s|}e^{z_j}}, \dots ,  \frac{e^{z_{|\Lambda_s|}}}{\sum _ {j = 1} ^ {|\Lambda_s|}e^{z_{|\Lambda_s|}}}\right).
\]
This grants us $\text{softmax}(z)_j \in (0, 1)$ and $\sum_{j = 1}^{|\Lambda_s|}\text{softmax}(z)_j= 1$, providing us a valid probability measure in $\mathcal P(S)$. A useful property of softmax is shifting invariance, where $\text{softmax}(z) = \text{softmax}(z + c\mathbf{1})$, for $c\in \RR$ applied component-wise to $z$.

We define $\nu^k := \text{softmax}(E_U {x_{T, \gamma}^{k, L}}) \in \PS$, where $x^{k, L}_{T, \gamma} \in \XX$ is initialized at $x^{k, L}_{0, \gamma} \in \XX$ with its empirical measure $\mu^k_{0, \gamma}$ evolving according to the pushforward $\Phi$ with actions determined by a policy $\gamma : \PX^K \to \UU$. We aim to find a policy $\gamma \in \Gamma$ that minimizes the following loss structure.

Cross-Entropy loss is commonly used for machine learning classification tasks and autoregressive generative models \citep{losses}. It is well-suited for next-token prediction, which can be seen as classification over the entire vocabulary.
For a true probability measure $P_1$ and a predicted measure $P_2$ in $\mathcal P(\Lambda_s)$, with $P_2$ being absolutely continuous with respect to some reference measure $\lambda _ o$, the cross entropy is given by
$$
\mathrm{CE}(P_1, P_2) = \mathbb E_{P_1}[-\log p_2(x)] = - \int _{\Lambda_s} P_1(\diff v)\,\log p_2(v)
$$
where the density of $P_2$, $p_2 = \frac{\mathrm{d}P_2}{\mathrm{d}\lambda_o}$, is given by the Radon-Nikodym derivative. The continuous case has undesirable properties, such as potentially negative entropy, so we use the following discrete case, which is applicable for next-token prediction over a finite vocabulary. For discrete $P_1,P_2 \in \mathcal P(\Lambda_s)$, the cross entropy is given by
$$
\mathrm{CE}(P_1, P_2) = - \sum _{v \in \Lambda_s} P_1(v) \log P_2(v).
$$
For our formulation, we sum the prompt-wise losses over the entire data set to get our final cross entropy loss functional $V_{\rm CE} : (\Lambda_s^L \times \mathcal P (\Lambda_s))^K \times \Gamma \to \RR$ given by

\begin{equation}
    \label{eq:cross_entropy_loss}
    V_{\rm CE}(D, \gamma) = \sum_{k = 1} ^ K \frac{n_k}{\tilde K}\mathrm{CE}(y^k, \nu ^k) = -\sum _{k = 1} ^ K \frac{n_k}{\tilde K}\sum _ {v \in \Lambda_s} y^k(v) \log \nu ^k (v)
\end{equation}
where $n_k = \sum_{l = 1}^{\tilde{K}} \chi_{\mathbf{x}^k}(\mathbf{x}^{l})$ for all $k=1,\ldots,K$.

\eqref{eq:cross_entropy_loss} is equivalent to the empirical risk of the original sample:
\[
    \frac{1}{\tilde{K}}\sum_{l = 1}^{\Tilde{K}} -\log\nu(\tilde{y}^{l}\mid\mathbf{x}^l) = \sum_{k=1}^K\frac{n_k}{\tilde{K}}\bigg(\frac{1}{n_k}\sum_{\substack{l=1\\\mathbf{x}^l=\mathbf{x}^k}}^{\tilde{K}}-\log{\nu(\tilde{y}^l\mid\mathbf{x}^l)}\bigg) = \sum_{k = 1} ^ K \frac{n_k}{\tilde K}\mathrm{CE}(y^k, \nu ^k).
\]

The cross entropy loss has no notion of the underlying geometry of the state space. It depends only on where $p$ assigns the same probability mass as $q$, not the proximity of non-overlapping masses for $p, q \in \PX$. This property is useful for vocabularies where the token geometry doesn't reflect semantic information; adjacent tokens can have unrelated meanings.

We only compare the probability measure of the last position ($p_L = L \in \mathrm{PE}_{L}$) in the output vectors to our label. This is also a trick used in practice in Transformer fine-tuning, where features and labels are separated in this way, see \cite[Equation~3]{radford2018improving}.

\subsection{Finite Model Approximations to Transformers}\label{sub:approximation}

We select parameters $l, m, n \in \mathbb Z _+$ to quantize the above MDP into MDP$^{(l, m, n)}$ as follows: we quantize the values of the probability measure on the state space with $l$, the state space itself with $n$, and the action space with $m$.
We quantize the token embedding space $\XX$ as $\XX_n$. By the compactness of $\XX$, there exist some $M_\XX \in \RR_+$ that tightly bounds all coordinates of points $x = (x_1, x_2, \dots, x_d)$ in the embedding space $\XX \subset \RR^d$, i.e.
$$
M_\XX := \max _{x \in \XX, {j \in \{0, \dots, d\}}} ||x_j||.
$$
So $\XX \subset [-M_\XX, M_\XX] ^d$. We place a grid $\XX_n$ of $n^d$ many points uniformly over $[-M_\XX, M_\XX] ^d$ with spacing $\frac{2M_\XX}{n-1}$, which is finite for $n > 1$. Then, given $x \in \XX$, we can find some $\hat x \in \XX_n$ that satisfies
$$
\|x - \hat x\| \leq \frac{1}{2} \frac{2M_\XX}{n-1} \sqrt d =  \frac{M_\XX\sqrt d}{n-1}
$$
where $\sqrt{d}$ comes from taking the L2 norm in $d$ dimensions. We can arrive at such a grid arrangement by taking
$$
\XX_n = \{ (-M_\XX + \frac{2M_\XX \hat k_1}{n-1}, \dots , -M_\XX + \frac{2M_\XX\hat k _d}{n-1}) : \hat k _j \in \{0, \dots, n-1\} \}.
$$
Quantizing $\XX$ in this way ensures that the number of points is always $n^d$ which lends itself nicely for implementation.


We quantize the action space, imposing a grid on the entries of the parameter matrices. Due to the compactness of $\UU$, there exists some $M_\UU < \infty$ that is the tightest bound for all elements of the weights $(W, A, b, Q, K, V, E_E, E_U )$ of $U\in \UU$, i.e.
$$
M_{\UU} := \max _ {U \in \UU, \theta \in U; i,j} |\theta_{i, j}|
$$
where $\theta _ {i,j}$ are matrix entries of $\theta$. Thus, $ -M_\UU \leq \theta _ {i,j} \leq M_\UU $, and we place $m$ points uniformly over each parameter's matrix entries with spacing $\frac{2M_\UU}{m-1}$, which is finite for $m > 1$.

Then, given $U=(W_W, W_A, W_b, W_Q, W_K, W_V, E_E, E_U) \in \UU$, we can find some $\hat U = (\hat{W_W}, \hat{W_A}, \hat{W_b},$ $\hat{W_Q}, \hat{W_K}, \hat{W_V}, \hat E_E, \hat E_U)$ such that
\begin{align*}
|| U - \hat U||_F &\leq \frac{1}{2}\frac{2M_\UU}{m-1} \bigg (\sqrt{d \cdot d_1} +\sqrt{d_1 \cdot d} +\sqrt{d_1} +\sqrt{d_2 \cdot d}+\sqrt{d_2 \cdot d}\\
&\hspace{12em}+\sqrt{d \cdot d} + \sqrt {|\Lambda_s| \cdot d}  + \sqrt{ d \cdot |\Lambda_s|}\bigg)\\
&= \frac{M_\UU}{m-1} \bigg(2\sqrt{d \cdot d_1} + \sqrt{d_1} + 2\sqrt{d_2 \cdot d} + d + 2\sqrt{|\Lambda_s| \cdot d} \bigg )
\end{align*}
where the coefficient is determined from the dimensions of the parameter matrices and $||.||_F$ is the product-wise Frobenius norm (as squared):
\[
    \begin{split}
        \|U\|_F^2 &= \mathrm{Tr}(W W ^T) + \mathrm{Tr}(A A ^T) + \mathrm{Tr}(b b ^T)\\
        &\qquad+ \mathrm{Tr}(Q Q ^T) + \mathrm{Tr}(K K ^T) + \mathrm{Tr}(V V ^T) + \mathrm{Tr}(E_E E_E ^T) + \mathrm{Tr}(E_U E_U ^T).
    \end{split}
\]
Note that Transformers with $T>1$ layers do not have actions containing both $E_E$ and $E_U$; however, we keep them here to provide an upper bound. A quantization that adheres to the above conditions can be obtained by taking

\begin{align*}
    \UU_m = \big \{ (\hat{W_W}, \hat{W_A}, \hat{W_b},\hat{W_Q}, \hat{W_K}, \hat{W_V}, \hat E_E, \hat E_U) :&\,\theta _{i,j} = -M _\UU + \frac{2M_\UU \hat k_m}{m-1}, \\
    &\hspace{-12em}\theta \in \{\hat{W_W}, \hat{W_A}, \hat{W_b},\hat{W_Q}, \hat{W_K}, \hat{W_V}, \hat E_E, \hat E_U\}, \hat k_m \in \{0, \dots, m-1\} \big\}.
\end{align*}

Finally, we quantize $\PXn$ by taking the set of quantized probability measures on $\XX_n$ as
$$
\mathcal P_l(\XX_n) := \{ (p_\alpha)_{\alpha = 1}^ {|\XX_n|} : p_{\alpha} \in \frac{1}{l-1}\{0, \dots, l-1\}, \sum _{\alpha = 1} ^ {|\XX_n|} p_\alpha = 1\}.
$$

We define the quantizers $Q_n : \XX \to \XX_n$ mapping $x \in \XX$ to the closest point in $\XX_n$, and $R_l : \PXn \to \mathcal P_l(\XX _n)$ mapping $\hat \mu \in \PXn$ to the closest element in $\mathcal P_{l}(\XX _n)$ using \cite{reznik2011algorithm}'s algorithm. Quantizing the spaces in this fashion ensures a deterministic grid that is straightforward for implementation while guaranteeing the same $O(1/n)$ asymptotic error rate. 

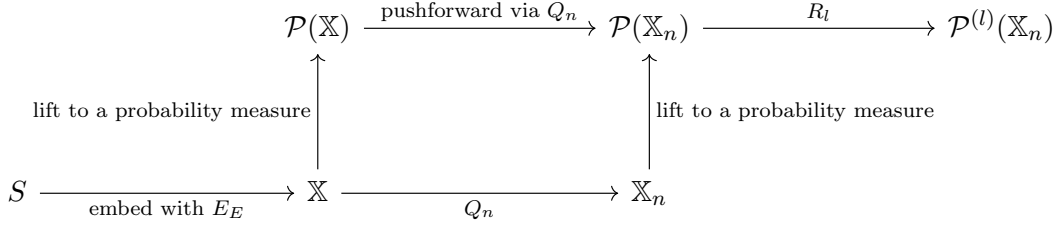
\begin{figure}[H]
    \centering
    \begin{tikzcd}[row sep=4em, column sep=8em]
    & \PX \arrow[r, "\text{pushforward via } Q_n"] & \PXn \arrow[r, "R_l"] & \mathcal{P}^{(l)}(\XX_n) \\
    S \arrow[r, "\text{embed with } E_{E}"'] & \XX \arrow[u, "\text{lift to a probability measure}"] \arrow[r, "Q_n"'] & \XX_n \arrow[u, "\text{lift to a probability measure}"']
    \end{tikzcd}
    \caption{Two-stage State Quantization}
    \label{fig:quantization}
\end{figure}

\subsection{Dynamic Programming Based Training}\label{sub:dp}

To guarantee that the quantized Transformer is trained optimally, we solve the finite state and action space $\text{MDP}^{(l, n, m)}$ exactly using backward dynamic programming rather than using gradient-based or other training schemes that lack such a guarantee. We define $\Phi ^ {(n)} : \PXn \times \UU \to \PXn$ as the pushforward on the quantized embedding space: $\Phi ^ {(n)} (\hat \mu, U) := \hat \mu \circ T_{\hat \mu, U}^{-1}$, where $T$ is the same as in the definition for $\Phi$, $\hat \mu \in \PXn$ and $U \in \UU$. The quantized deterministic ensemble pushforward $\bm \Phi^{(l, n, m)}$ maps $\PXn^K\times \UU_m \to \PXn^K$ by quantizing and applying the single-measure pushforward $\Phi ^{(n)} : \PXn \times \UU\to \mathcal P (\XX_n)$ point-wise and quantizing with $R_l$, that is, $\bm \Phi ^{(l, n, m)}(\hat {\bm \mu}, U):= (R_l(\Phi ^{(n)}(\hat \mu ^ 1, U), \dots , R_l(\Phi ^{(n)}(\hat \mu ^ K, U) $. It induces the transition kernel $\eta ^{(l, n, m)} : \PXn\times \UU_m \to \mathcal P(\PXn)$. This characterizes $\text{MDP}^{(l, n, m)}$ with a state space $\PXn^K$, action space $\UU_m$, transition kernel $\eta^{(l, n, m)}$ and the discrete cross-entropy loss $V_{CE}$ from ~\ref{eq:cross_entropy_loss}.

Note that this structure leverages all our predictive information, but does not necessarily translate into model performance in the following case: taking a small $l$ may artificially reduce the loss without improving performance due to coarseness.

For a given data set $D$, we pose the dynamic programming equations for $\text{MDP}^{(l, n, m)}$ as follows
\begin{equation}
    \label{eq:DP}
    \begin{split}
        \mathfrak{C}_{l, n, m, T}(\hat{\boldsymbol{\mu}}) &:= \frac{1}{K}\sum_{k = 1}^K \hat{\mathrm{CE}}(\hat \mu^{k, N}_T ,\hat y^k)\\
        \mathfrak C _ {l, n, m, t} (\hat {\bm \mu}) &:= \min_{U \in \UU_m} \sum _{\hat{\bm \mu}_{t +1} \in P^{(l)}(\XX_n)^K} \eta^{(l, n, m)}(\hat{\bm \mu}_{t + 1}\mid\hat{\bm \mu}_t, U)\,\mathfrak C _ {l, n, m, t + 1} (\hat{\bm \mu}_{t + 1})\\
        &\;=\min_{U \in \UU_m} \mathfrak{C}_{l, n, m, t+1} (\boldsymbol{\Phi}^{(l, n, m)}(\hat {\bm \mu}, U))
    \end{split}
\end{equation}

for all $\hat{\bm \mu} = (\hat \mu ^1, \dots \hat \mu ^K)\in \mathcal P^{(l)}(\XX _n)^K$ and $t \in \mathcal T$. The last equality is due to the deterministic nature of $\bm \Phi ^{(l, n, m)}$. Following \cite[Theorem~17]{akman2026optimal}, it follows that, under the compactness of $\XX$ and $\UU$, due to the weak continuity property of the constructed MDP, the quantized training procedure with the dynamic programming equations~\eqref{eq:DP} is asymptotically globally optimal for the original problem as quantization gets finer, which then results in a near-optimal open-loop controls under fine enough quantization.

\section{Construction of an Approximating Transformer}\label{sec:Approximation}

With the machinery of Section~\ref{sec:problem}, we relax the next token prediction problem for observations generated by latent ergodic stochastic processes into a filtering problem with finite-window and quantized data. The question then becomes: \emph{Does there exist a Transformer that can realize such filter?} The answer is affirmative and mostly expected since such predictive distribution is a lookup table. In this part, we give a step-by-step recipe for how to construct such Transformer.
 
Our construction in this section follows a strategy similar to the one used in \citep{yun2020are} to show that fixed-width Transformers with positional encodings are universal approximators of continuous sequence-to-sequence functions on compact domains, see also \citep{yun2020connections, geshkovski2024measure, jiang2024approximation, kajitsuka2024transformers, wang2024understanding, furuya2025transformers}. Their argument quantizes the input, computes a \emph{contextual mapping} with attention layers (which is a distinct identifier for each quantized input), and maps identifiers to outputs with feed-forward layers. Our inputs are already quantized, the gatherer layers of §\ref{sub:Gatherers} compute and collect the contextual mapping at the readout token, and the ReLU network of §\ref{sub:Interpolator} performs the value mapping by interpolation.

\begin{figure}[H]
    \centering
    \includegraphics[width=\linewidth]{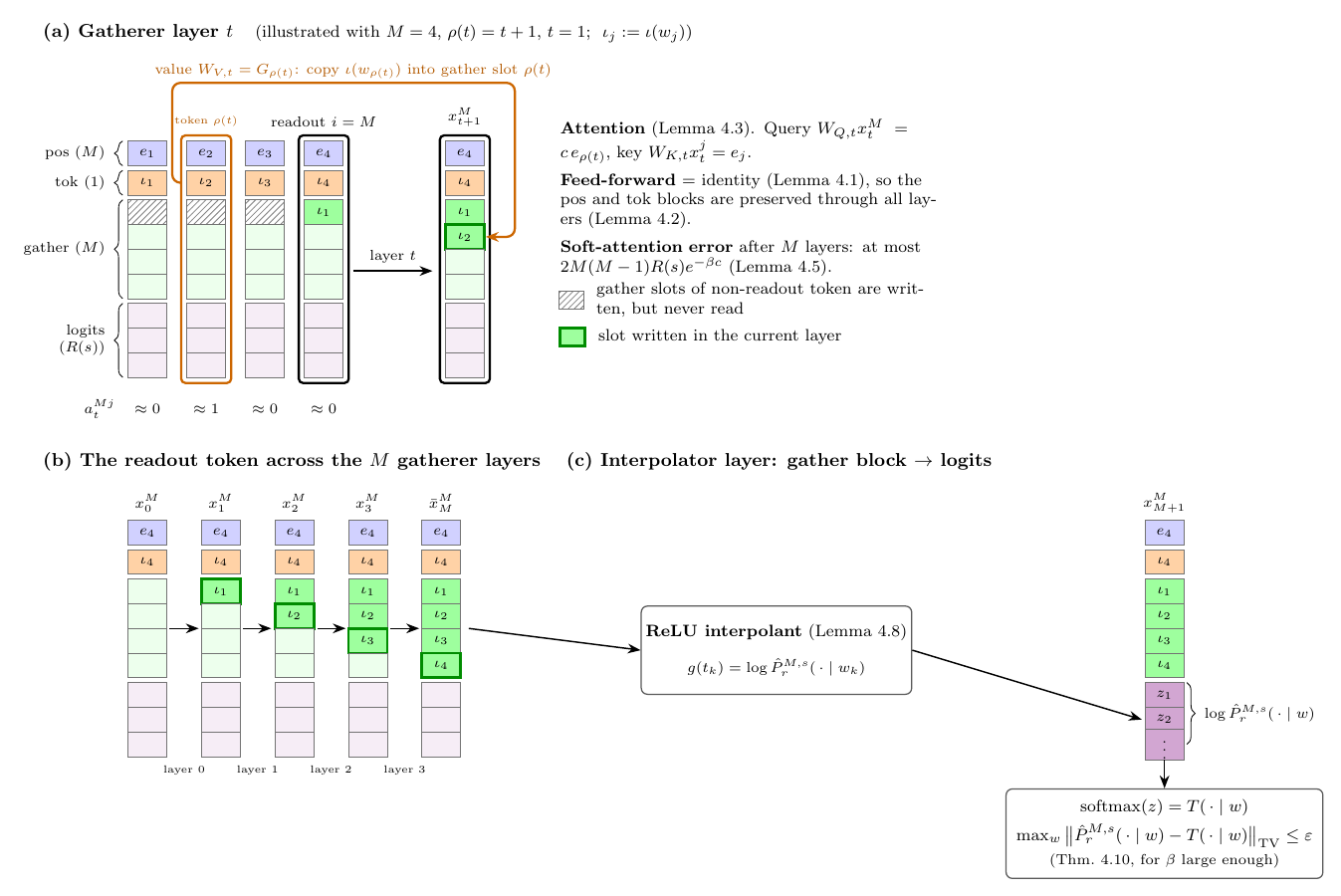}
    \caption{Description of the Transformer built in this section.}
    \label{fig:architecture}
\end{figure}

\subsection{Transformer Layers and Token Embeddings}

\subsubsection{Token Embeddings}\label{sub:TokenEmbeddings}

Throughout this section, fix the window length $M$ and the quantization level $s$. We fix an enumeration $\iota$ of the $1/s$-net $\Lambda_s$, meaning that $\iota:\Lambda_s\to\{1,\ldots,R(s)\}$ is a bijective map. Recall that $R(s)$ is the cardinality of $\Lambda_s$. This makes each window $w = (w_1,\ldots,w_M)$ an element of $\Lambda_s^M$, which has $R(s)^M$ elements. We set the token embedding dimension to $d = 2M + R(s) + 1$ with the following decomposition in mind
\begin{equation}
    \label{eq:decomp}
     d = \underbrace{M}_{\rm positional}\quad+\quad\underbrace{1}_{\rm token\,value}\quad+\quad\underbrace{M}_{\rm gather}\quad+\quad\underbrace{R(s)}_{\rm logits}.
\end{equation}
According to decomposition~\eqref{eq:decomp}, we split each token embedding into four parts: $x = (x^{\rm pos}, x^{\rm tok}, x^{\rm gat}, x^{\rm log})\in\RR^M\times\RR^1\times\RR^M\times\RR^{R(s)} = \RR^d$. Let $\Pi_{\bullet}$ be the coordinate projections from $\RR^d$ to $\bullet$ space where $\bullet\in\{{\rm pos}, {\rm tok}, {\rm gat}, {\rm log}\}$. Lastly, $\{e_i\}_{i=1}^M$ form the standard basis for $\RR^M$, which we use for positional encodings and \emph{gathering tokens}.

For any $v\in\Lambda_s$ at position $i\in\{1,\ldots,M\}$, we define
\begin{equation}
    \label{eq:embedding}
    E_E(v, p_i) := (e_i, \iota(v), 0_M, 0_R)\in\RR^d
\end{equation}
where $0_M$ and $0_R$ are the zero vectors of dimensions $M$ and $R(s)$, respectively. The map $E_E$ gives rise to an unembedding function satisfying the requirements of Assumption~\ref{assumption:B1} since $\Pi_{\rm pos}$ recovers the positional encodings. Using the embedding operator, for any window $w\in\Lambda_s^M$, we define $x_0^{i,(w)} = E_E(w_i, p_i)$ for all $i\in\{1,\ldots,M\}$ and $\mu_0^{(w)} := \frac{1}{M}\sum_{i=1}^M\delta_{x_0^{i,(w)}}$.

\subsubsection{Gatherer Layers}\label{sub:Gatherers}

Let us fix a bijection $\rho:\{0,\ldots,M-1\}\to\{1,\ldots,M\}$ sending each layer to a position. We need it to copy the token at position $\rho(t)$ into the gather coordinate of the readout token at each layer $t\in\{0,\ldots,M-1\}$. We form the first $M$ layers as\emph{gatherer layers}, which are responsible for collecting all the information in the window to the readout position at the $M-1^{\rm th}$ layer. As Transformers have feed-forward parts at each layer, we need to give role to them:

\begin{lemma}\label{lem:RoleOfFeedForward}
    Let $\sigma$ be {\tt ReLU} and assign
    \[
        W_A := \begin{bmatrix}
            I_d \\
            -I_d
        \end{bmatrix}\qquad
        W_b := 0_{2d}\qquad
        W_W := \begin{bmatrix}
            I_d & -I_d
        \end{bmatrix}.
    \]
    Then, $W_W\sigma(W_Ax+W_b) = x$ for all $x\in\RR^d$.
\end{lemma}

We apply $(W_W, W_A, W_b)$ in Lemma~\ref{lem:RoleOfFeedForward} for each gatherer layer, that is 
\[
    (W_{W,t}, W_{A,t}, W_{b,t}) = (W_W, W_A, W_b)
\]
for each $t\in\{0,\ldots,M-1\}$, where the right hand side as in Lemma~\ref{lem:RoleOfFeedForward}. Consequently, layer-wise updates take the form
\begin{equation}
    \label{eq:Updates}
    x_{t+1}^i = x_t^i + \sum_{j=1}^M\frac{\mathrm{e}^{\beta\langle W_{Q,t} x_t^i, W_{K,t}x_t^j\rangle}}{\sum_{j'=1}^M \mathrm{e}^{\beta\langle W_{Q,t} x_t^i, W_{K,t}x_t^{j'}\rangle}}W_{V,t}\,x_t^j
\end{equation}
for each $i\in\{1,\ldots, M\}$ and $t\in\{0,\ldots,M-1\}$.

We now specify the roles of the query, key and value matrices. For keys: $W_{K,t} = \Pi_{\rm pos}$ so that $W_{K,t}x_t^i = (x_t^i)^{\rm pos} = e_i$ for all $i\in\{1,\ldots, M\}$ and $t\in\{0,\ldots,M-1\}$. The task we assign to queries is a bit more specialized: For each $t\in\{0,\ldots,M-1\}$, let $A_t\in\RR^{M\times M}$ be a matrix defined on the standard basis $\{e_i\}_{i=1}^M$ of $\RR^M$, as described in §\ref{sub:TokenEmbeddings}, whose action is defined by
\[
    A_te_M = e_{\rho(t)} \qquad A_te_i := e_i\;(i\neq M).
\]
Then we set $W_{Q,t} := cA_t\Pi_{\rm pos}$ for some constant $c>0$ satisfying
\begin{equation}\label{eq:ConstantC}
    c = \max_{\substack{c'\|A_t\| \leq M_{\UU}\\0\leq t \leq M-1}} c'
\end{equation}
so that query matrix is admissible in $\UU$. For the value matrices: Let
\begin{equation}\label{eq:G_j}
    G_j :=  \begin{bmatrix}
                0_M & 0 & e_j & 0_{R(s)}
            \end{bmatrix}^T\Pi_{\rm tok}
\end{equation}
for each $j\in\{1,\ldots,M\}$, which is a rank one matrix (since $\Pi_{\rm tok}$ projects onto a $1$-dimensional space) reading the token-value coordinate and then writing it into the gather coordinate $j$. Then we set the value matrices $W_{V,t} := G_{\rho(t)}$ for each $j\in\{1,\ldots,M\}$.

To sum up, we separate the dynamics into two parts: the feed-forward layers are responsible for carrying the token value and the positional information across the layers. Since each token except the one at the readout position ($i=M$) has null gatherers and logits, the feed-forward flow does not interfere with gatherer and logit positions. On the other hand, by construction of the self-attention layer, it is only concerned with writing the correct information to the gatherer positions of the readout token at each layer by a prescribed manner. Note that $W_{V,t}$ reads the token information and writes it to the $\rho(t)^{\rm th}$ position of the $M^{\rm th}$ token. As a result, the self-attention layers do not change the positional and token values. Therefore, both layer types act as counterparts of each other. In short, the feed-forward layers carry the relevant information through the Transformer dynamics, while the self-attention layers use that information in a sequential manner. In the following, we make this high-level discussion rigorous.

\begin{lemma}\label{lem:Invariance}
    For all $w\in\Lambda_s^M$, $i\in\{1,\ldots,M\}$ and $t\in\{0,\ldots,M\}$, we have
    \[
        (x_t^i)^{\rm pos} = e_i \qquad \text{and} \qquad (x_t^i)^{\rm tok} = \iota(w_i).
    \]
\end{lemma}

Because of Lemma~\ref{lem:Invariance}, the decomposition in~\eqref{eq:decomp} and choices of network weights are almost clear. To support, this design choice also makes the error propagation analysis simpler: The attention scores depend only on $x^{\rm pos}$ while the values are determined by $x^{\rm tok}$ in an exact way at each layer. Before proving the complete characterization of the self-attention block under our construction, we define a helper function
\[
    \mathrm{sel}_t(i) := \begin{cases}
        \rho(t), & i = M,\\
        i, & i\neq M.
    \end{cases}
\]
which is defined for all $t\in\{0,\ldots,M-1\}$ and $i\in\{1,\ldots,M\}$.

\begin{lemma}\label{lem:Attention}
    For every $t\leq M-1$ and all $i,j\in\{1,\ldots,M\}$, we have
    \[
        \langle W_{Q,t} x_t^i, W_{K,t}x_t^j\rangle = c\,\chi_{j}(\mathrm{sel}_t(i)).
    \]
    Hence,
    \[
        a_t^{i\mathrm{sel}_t(i)} = \frac{e^{\beta c}}{e^{\beta c} + M - 1} \qquad\text{and}\qquad a_t^{ij} = \frac{1}{e^{\beta c} + M - 1}\;\;(j\neq\mathrm{sel}_t(i)),
    \]
    where $a_t^{ij}$'s are exactly as in~\eqref{eq:AttentionWeights} and $c$ is as in~\eqref{eq:ConstantC}.
\end{lemma}

\begin{remark}
    Note that $c$ is a design parameter, which is determined by the action space diameter $M_\UU$ per \eqref{eq:ConstantC}. Therefore, the attention scores are calculated using only system intrinsic values given an input. Also, we remind that attention scores are determined solely by the positional encodings by our construction.
\end{remark}

From Lemma~\ref{lem:Attention}, it is clear that we cannot recover the token information without any errors. Ideally, we want to read the token value from $\rho(t)^{\rm th}$ position and write it to the corresponding gatherer coordinate. However, the self-attention yields a weighted average of all tokens. The next lemma quantifies this error while providing computed attention scores across all particles.

\begin{lemma}\label{lem:Error}
    Let $\bar{x}_t^M$ be the readout token under (hard) selection, that is $\bar{x}_{t+1}^M = \bar{x}_t^M + G_{\rho(t)}\bar{x}_t^{\rho(t)}$ for all $t\in\{0,\ldots, M-1\}$ with $\bar{x}_0^M = x_0^M$. Then, the final error between the hard choice and the soft choice\footnote{See Remark~\ref{rem:Choice} for the distinction and what this buys for us in the implementation phase.} is quantified as follows:
    \[
        \|x_M^M - \bar{x}_M^M\|_2 \leq 2M(M-1)R(s)e^{-\beta c}.
    \]
    Moreover, $\bar{x}_M^M := (e_M, \iota(w_M), (\iota(w_1),\ldots,\iota(w_M)), 0_R)$. 
\end{lemma}

Remaining part of our construction is based on \emph{distinguishability} of different tokens parametrized by different initial windows, that is, if $w\neq w'\in\Lambda_s^M$, then the resulting $\bar{x}_M^{M,(w)}$ and $\bar{x}_M^{M,(w')}$ should be identifiable to be distinct in the metric. This is crucial because, at the end, we want to map the readout tokens to probability distributions (thereby we get conditional probability distribution) without any loss of information. The next corollary is a simple result of Lemma~\ref{lem:Error}'s second claim.

\begin{corollary}\label{cor:Separation}
    The map $\Lambda_s^M \ni w \mapsto \bar{x}_M^{M,(w)}$ is injective, and $\|\bar{x}_M^{M,(w)}-\bar{x}_M^{M,(w')}\|_2\geq 1$ whenever $w,w'\in\Lambda_s^M$ with $w\neq w'$.
\end{corollary}
\begin{proof}
    If $w\neq w'$, then they differ in at least one coordinate. Since $\iota$ is an integer-valued enumeration of $\Lambda_s^M$, then $|\iota(w) - \iota(w')| \geq 1$, which transfers to $\|\bar{x}_M^{M,(w)}-\bar{x}_M^{M,(w')}\|_2$ directly, also establishing injectivity.
\end{proof}

\begin{remark}\label{rem:Choice}
    The hard choice in Lemma~\ref{lem:Error} refers to choosing $\bar{x}_t^{\rho(t)}$ directly without considering attention weights, whereas soft choice is done by choosing the one with the highest attention score. Implementation-wise, we use soft choice unless in an occurrence of an extreme case.
\end{remark}

\subsubsection{The Interpolator Layer}\label{sub:Interpolator}

The machinery of §\ref{sub:Gatherers} collects all data in the window into the gather block in the decomposition~\eqref{eq:decomp}. In this part, we show that it is possible to map the information in logits part to suitable conditional distributions. We achieve this task doing some {\tt ReLU} calculus to obtain a single layer {\tt ReLU} network applied to the readout tokens $x_M^M$ and interpolate them into vectors in $\RR^{R(s)}$.

\begin{lemma}
    \label{lem:relu_interpolation}
    Let $\{\zeta_k\}_{k=1}^n$ be distinct set of features in $\RR^d$ and $\{v_k\}_{k=1}^n$ be arbitrary targets in $\RR^{R(s)}$. There exists a one-layer ${\tt ReLU}$ network $\phi:\RR^d\to\RR^{R(s)}$ of width $n-1$ with $\phi(\zeta_k) = v_k$ for all $k\in\{1,\ldots,n\}$ with Lipschitz constant $L_\phi\in\RR_+$.
\end{lemma}

The proof of Lemma~\ref{lem:relu_interpolation} chooses a generic $a$, on which the whole analysis is built. In the next result, we give an explicit choice for $a$ and see its effect on the Lipschitz constant $L_\phi$ of the {\tt ReLU} network.

\begin{lemma}\label{lem:Explicit_a}
    Take $a=[0_M\; 0\; \langle R(s)\rangle_M \; 0_R]$, where $\langle R(s) \rangle_M = (1, R(s), R(s)^2, \ldots, R(s)^{M-1})$, in the proof of Lemma~\ref{lem:relu_interpolation}. Then the following assertions are true:
    \begin{itemize}
        \item[(i)] $\langle a, \bar{x}_M^{M,(w)}\rangle = \sum_{j=1}^M \iota(w_j)R(s)^{j-1}$, $\|a\|_2 \leq R(s)^{M-\frac{1}{2}}$ and $t_{k+1} - t_k \geq 1$ for all $k=1,\ldots,R(s)-1$.
        \item[(ii)] Under Assumptions~\ref{a:BaseAssumptions}-(A2,4), with $n = R(s)^M$ and targets $v_w = \hat{P}_r^{M,s}(\,\cdot\mid w)$ in Lemma~\ref{lem:relu_interpolation}'s statement, we have 
        \[
            L_\phi \leq 4\log{\frac{1}{q_{\rm min}\mu_{\rm min}^s}}R(s)^M(R(s)^M-1).
        \]
    \end{itemize}
\end{lemma}

\ka{}

\subsection{The Approximation Theorem}
Now that we have established how we construct a candidate Transformer, all is left to complete the description with an error analysis, which constitutes the purpose of this subsection.

\begin{theorem}\label{thm:ApproximatingTransformer}
    For any $\epsilon \in (0, 1/2]$, $r\geq M$ and all
    \[
        \beta \geq \frac{1}{c}\bigg(\log{(4M(M-1)R(s))} + \log{L_\phi} + \log{\frac{1}{\epsilon}}\bigg),
    \]
    the Transformer described through §\ref{sub:TokenEmbeddings}-\ref{sub:Interpolator} satisfies:
    \begin{itemize}
        \item[(i)] $\max_{w\in\Lambda_s^M}\|\hat{P}_r^{M,s}(\,\cdot\mid w) - \mathsf{T}(\,\cdot\mid w)\|_\TV \leq \epsilon$.
        \item[(ii)] $\min_{v\in\Lambda_s\, w\in\Lambda_s^M}\mathsf{T}(v\mid w) \geq q_{\rm min}\,\mu_{\rm min}^se^{-2\epsilon}$.
    \end{itemize}
    In this theorem, $\mathsf{T}(\,\cdot\mid w)$ denotes the readout distribution of the Transformer given input sequence $w$.
\end{theorem}

\section{Main Results: Performance of Trained Transformers}\label{sec:Main}
 
In §\ref{sec:problem}, we have studied the effect of memory in next-token prediction for observations of latent stationary ergodic and $N^{\rm th}$ order Markovian sources, with quantization aspects for implementable learning algorithms. Then, §\ref{sec:Approximation} is devoted to construct a Transformer that can realize finite-memory quantized-data predictive distribution in Definition~\ref{def:PredictiveDists}.
 
In this section, we complete our theoretical excursion by deriving a complete error bound for Transformers trained to perform next-token prediction tasks by the dynamic programming based training scheme presented in §\ref{sub:dp} building on \cite{akman2026optimal}. 

The final error bound consists of (i) optimization error due to training, (ii) filter stability bounds due to finite-memory, (iii) sample complexity bounds due to training, and (iv) insufficiency of architecture.
 
\def\sT{\mathsf{T}}
\def\cE{\mathcal{E}}
We need to prove some auxiliary results that lead us to the ultimate theorem. We commence with establishing some notation. Throughout this section, the Transformer has the architecture of §\ref{sec:Approximation}: it reads prompts of length $L = M$, and it has $T = M+1$ layers, namely the $M$ gatherer layers of §\ref{sub:Gatherers} followed by the interpolator layer of §\ref{sub:Interpolator}. Hence, an open-loop policy is a sequence $\gamma = (U_0,\ldots,U_M)\in\UU_m^{M+1}$. For such $\gamma$ and $w\in\Lambda_s^M$, we write $\sT^{\gamma}(\,\cdot\mid w)$ for the readout distribution of the Transformer with weights given by $\gamma$ and the input sequence $w$, run as $\mathrm{MDP}^{(l,n,m)}$ of §\ref{sub:dp}. That is, the state quantizer $Q_n$ and the measure quantizer $R_l$ act after every layer, and $\sT^{\gamma}(\,\cdot\mid w) = \mathrm{softmax}(E_Ux_{M+1}^M)$, where $x_{M+1}^M$ is the token at position $M$ after the last layer. Let $\gamma^*$ be the optimal open-loop policy that is obtained via the optimal closed-loop policy given by the dynamic programming recursions~\eqref{eq:DP}. Since $\boldsymbol{\Phi}^{(l,n,m)}$ is deterministic, $\gamma^*$ minimizes $V_{\rm CE}(D,\cdot)$ over all open-loop policies in $\UU_m^{M+1}$.
 
We also need to specify how the data set is sampled. Fix $r\geq M$. We assume that the $\tilde K$ sequences of the data set construction in Appendix~\ref{sec:OCFormulation} are independent copies of $(q^s(Y_{r-M+1}),\ldots,q^s(Y_{r+1}))$. Equivalently, the training pairs $(\bm{x}^l,\tilde y^l)$, $l=1,\ldots,\tilde K$, are independent copies of $(q^{M,s}(Y_{r-M+1:r}), q^s(Y_{r+1}))$. If $\{Z_r\}_{r\in\ZZ_+}$ is stationary, the law of this pair does not depend on $r$. For each window $w\in\Lambda_s^M$, we write
\begin{equation}\label{eq:WindowLaw}
    p_r(w) := \Pr\big(q^{M,s}(Y_{r-M+1:r}) = w\big),\qquad n_w := \sum_{l=1}^{\tilde K}\mathbf{1}_{\{\bm{x}^l = w\}},
\end{equation}
and, whenever $n_w\geq 1$, $y^w := \frac{1}{n_w}\sum_{l=1}^{\tilde K}\delta_{\tilde y^l}\mathbf{1}_{\{\bm{x}^l = w\}}\in\cP(\Lambda_s)$. In words, $p_r$ is the law of a test window, $n_w$ is the number of occurrences of $w$ in the data set and $y^w$ is the empirical distribution of its labels. If $w = \bm{x}^k$, then $n_w = n_k$ and $y^w = y^k$ in the notation of Appendix~\ref{sec:OCFormulation}.
 
Finally, the Transformer of §\ref{sec:Approximation} has to be representable on the grids of §\ref{sub:approximation}. To this end, we take $\XX = [-M_\XX, M_\XX]^d$ and let $\UU$ be the set of weights whose entries lie in $[-M_\UU, M_\UU]$, where $M_\XX, M_\UU\in\NN$ satisfy
\begin{equation}\label{eq:RangeConditions}
    M_\XX\geq\max\Big\{R(s),\;1+\log\frac{1}{q_{\rm min}\mu_{\rm min}^s}\Big\}\qquad\text{and}\qquad M_\UU\geq\max\Big\{2R(s)^M,\;2\log\frac{1}{q_{\rm min}\mu_{\rm min}^s}\Big\}.
\end{equation}
The first condition keeps the positional encodings, the token values, the gather coordinates and the logits of the readout token inside $[-M_\XX, M_\XX]$. The second one makes the weights of the interpolator layer admissible; see the proof of Lemma~\ref{lem:TransformerPerformance}.
 
The cross-entropy loss can be decomposed into two parts: 
\[
    \mathrm{CE}(P_1,P_2) = H(P_1) + D_{\rm KL}(P_1\parallel P_2)
\]
for two probability measures $P_1,P_2\in\cP(\Lambda_s)$, where $H(P_1) := -\sum_{v\in\Lambda_s}P_1(v)\log P_1(v)$ is the Shannon entropy. Therefore, unless the true distribution, in this case $P_1$, is deterministic (a Dirac mass), $\mathrm{CE}(P_1,P_2) > 0$. However $P_1$ is often modeled as a non-deterministic distribution. For this reason, we want to get rid of the entropy term, and define the \emph{excess loss}
\[
    \cE(\gamma) := V_{\rm CE}(D, \gamma) - \sum_{k=1}^{K}\frac{n_k}{\tilde K}H(y^k) = \sum_{k=1}^{K}\frac{n_k}{\tilde K}D_{\rm KL}\big(y^k\parallel \sT^{\gamma}(\,\cdot\mid \bm{x}^k)\big),
\]
where $\gamma\in\UU_m^{M+1}$ is an open-loop policy. $\cE$ is always a nonnegative quantity due to the aforementioned decomposition of the cross-entropy loss. Since we do not consider any other data set, we also do not explicitly write the data set dependency for a simpler notation. Note that the optimization task is not changed since the data set is not changed: For fixed data set $D$, $\cE$ is an affine transformation of $V_{\rm CE}(D, \cdot)$. Therefore, any minimizer of $\cE$ is a minimizer of $V_{\rm CE}(D, \cdot)$ and vice versa. In particular, $\gamma^*$ minimizes $\cE$ over $\UU_m^{M+1}$.

\subsection{The Unified Bound for Stationary Ergodic Sources}\label{sub:StationaryErgodicBound}
\begin{theorem}\label{thm:Soft}
    Suppose Assumptions~\ref{a:BaseAssumptions} and~\ref{a:Ergodicity} hold. Let $Y_1,Y_2,\dots$ be a single trajectory of the source, and
    let $\underline{U}^*_{\tilde K}$ be an open-loop policy returned by~\eqref{eq:DP} for the training pairs
    $\big(q^{M,s}(Y_{j:j+M-1}),\,q^s(Y_{j+M})\big)$, $j=1,\dots,\tilde K$. Let the test sequence $Y_{0:r}$
    be independent of this trajectory. Then, almost surely,
    \[
        \limsup_{M\to\infty}\;\limsup_{s\to\infty}\;\limsup_{(n,m,\beta)\to\infty}\;\limsup_{\tilde K\to\infty}\;\sup_{r\geq M}\EE_{Y_{0:r}}\Big\|P_r - \Xi_s\big(\mathsf{T}^{\underline{U}^*_{\tilde K}}(\,\cdot\mid q^{M,s}(Y_{r-M+1:r}))\big)\Big\|_\TV = 0,
    \]
\end{theorem}

\subsection{The Unified Bound for Markovian Sources}\label{sub:Bound}
 
\begin{theorem}\label{thm:Unified}
    Suppose that Assumption~\ref{a:BaseAssumptions} holds. Let $r\geq M\geq N$, let $s\in\NN$ with $R(s)\geq 2$, let $\delta\in(0,1)$, and let the data set be sampled as above with
    \[
        \tilde K \geq 8(q_{\rm min}\mu_{\rm min}^s)^{-M}\Big(M\log{R(s)} + \log{\frac{3}{\delta}}\Big).
    \]
    Let $(n,m,l)\in\NN^3$ be aligning quantization levels (Definition~\ref{def:AligningQuantization}), let $M_\XX$ and $M_\UU$ satisfy~\eqref{eq:RangeConditions}, and let $\beta>0$ be such that $\eta(n,m,\beta)\leq 1$, where
    \begin{equation}\label{eq:EtaDef}
        \eta(n,m,\beta) := 2L_\phi M(M-1)R(s)e^{-\beta c} + \big(1+L_\phi M\sqrt{d}\big)\frac{M_\XX}{n-1} + M_\XX\sqrt{d}R(s)^{\frac{3M-1}{2}}\frac{M_\UU d_s}{m-1}.
    \end{equation}
    Here, $L_\phi$ is the Lipschitz constant of the interpolator layer (Lemmata~\ref{lem:relu_interpolation} and~\ref{lem:Explicit_a}), $c = M_\UU$ as in Definition~\ref{def:AligningQuantization}-(ii), $d = 2M+R(s)+1$ is the embedding dimension in~\eqref{eq:decomp} and $d_s := R(s)^{\frac{M+1}{2}}$. Then, with probability at least $1-\delta$ over the data set,
    \begin{equation}\label{eq:UnifiedBound}
        \begin{split}
            &\EE_{Y_{0:r}}\Big\|P_r - \Xi_s\big(\sT^{\gamma^*}(\,\cdot\mid q^{M,s}(Y_{r-M+1:r}))\big)\Big\|_\TV\\
            &\qquad\leq \underbrace{\boldsymbol{\alpha}^{\lfloor M/N\rfloor}}_{\text{memory}} + \underbrace{\Big(1+\frac{2M}{q_{\rm min}}\Big)\frac{L_\psi}{s}}_{\text{quantization}} + \underbrace{(1+\sqrt{2})\sqrt{\frac{R(s)^MC_\delta}{q_{\rm min}\,\mu_{\rm min}^s\,\tilde K}}}_{\text{sample complexity}} + \underbrace{\frac{\eta(n,m,\beta)}{\sqrt{2}}}_{\text{architecture and training}},
        \end{split}
    \end{equation}
    where $Y_{0:r}$ is independent of the data set, $C_\delta := R(s)\log 2 + M\log R(s) + \log(3/\delta)$, and $\boldsymbol{\alpha} = (1-\delta(\mathbf{T}))(2-\delta(\mathbf{Q}))$ is as in Proposition~\ref{prop:NMarkovDobrushin}.
\end{theorem}

\subsection{Asymptotic Performance under Filter Stability and Weak Feller Continuity}

Theorem~\ref{thm:Unified} is quantitative in its nature. If one only asks that the error vanish as $M$, $s$, $(n,m,\beta)$ and
$\tilde K$ grow, then neither geometric filter stability ($\alpha<1$) nor the Lipschitz condition (Assumption~\ref{a:BaseAssumptions}-(A2))is needed, and two changes in the proof suffice.

\emph{On memory.} Let $\hat\nu$ be the uniform distribution on $S$, let $\pi^\lambda_n$ denote the filter
started from $Z_0\sim\lambda$, and set
$\epsilon^{\mathrm{fs}}_n:=\sup_{\lambda\in\mathcal P(S)} \EE^\lambda\|\pi^\lambda_n-\pi^{\hat\nu}_n\|_{\mathrm{TV}}$.
Conditioning on $Y_{0:r-M}$ as in the proof of Proposition~2.9 gives
$\sup_{r\ge M} \EE\|P_r-P^M_r\|_{\mathrm{TV}}\leq2\varepsilon^{\mathrm{fs}}_{M-1}$, without
assuming stationarity.

Under (A1), $\epsilon^{\mathrm{fs}}_n\to0$ whenever the filter is stable in total variation in
expectation for every prior \citep[Definition I.5(iii)]{mcdonald2018stability}. Indeed, since
$D(\lambda\|\hat\nu)\leq\log|S|$, Theorem II.20 of \citet{mcdonald2018stability} upgrades this to
stability in relative entropy. The expected relative entropy is non-increasing in $n$ and continuous
in $\lambda$, so Dini's theorem on the compact simplex $\mathcal P(S)$ makes the convergence uniform,
and Pinsker's inequality returns to total variation.

By Theorems II.7 and II.19 of \citet{mcdonald2018stability}, such stability holds whenever the model is
observable in their sense (which in turn builds on \cite{Ramon2008discrete}). For finite $S$, this means that the laws
$\Pr(Y_{1:N}\in\cdot\mid Z_1=z)$, $z\in S$, are linearly independent for some $N$; linearly
independent observation densities suffice. As an additional remark, $N$th-order sources the sliding lift can be used here,
since no contraction coefficient is involved.

\emph{On quantization.} If Assumption~\ref{a:BaseAssumptions}-(A2) is replaced by continuity of $\psi(\cdot\mid z)$ with modulus
$\omega_\psi$, the arguments of Proposition~\ref{prop:Quantization} and Lemma~\ref{def:QuantizationBound} hold with $2L_\psi/s$ replaced by
$\omega_\psi(2/s)$.

Consequently, under Assumptions~\ref{a:BaseAssumptions}-(A1,A3),  the bound of Theorem~\ref{thm:Unified} holds with
$\alpha^{\lfloor M/N\rfloor}$ and $(1+2M/q_{\min})L_\psi/s$ replaced by
$2\epsilon^{\mathrm{fs}}_{M-1}$ and $(\tfrac12+M/q_{\min})\,\omega_\psi(2/s)$, respectively. This
bound is uniform in $r$ and vanishes in the iterated limit.

This covers cases outside the scope of Theorem~\ref{thm:Unified}. A deterministic cyclic transition kernel has
$\delta(\kappa)=0$, hence $\alpha\ge1$. Yet the observation matrices
$\frac{\theta}{k}J+(1-\theta)A$ -see §\ref{sub:mixing_experiment} are invertible for every $\theta<1$, so the model is
observable for all $0<\theta<1$. Finally, since $S$ is finite, the weak Feller-type conditions in
\citet{mcdonald2018stability} hold automatically, so continuity is needed only in the observation
variable.

\section{Numerical Experiments}\label{sec:Experiments}

For the following trials, unless stated otherwise, we use $T=2$ layers, $l = 9$ probability quantization levels, $ n = 5$ state space quantization levels ranging from $-2.5$ to $2.5$, attention temperature $\beta = 0.3$, $N = 1$, $L = 2$ with $\tilde K = 400$ training samples. We use $m = 5$ action quantization levels, ranging entries in $W, A, b, Q, V$ from $-2.5$ to $2.5$, keeping $K = I$ as done by \cite{geshkovski2025mathematical}, since we are using scalars in $\RR$ where $Q$ absorbs $K$ by commutativity. We take $E_E = I$ and $E_U = e_U \cdot I$ for a trainable parameter $e_U$. This ensures that Assumption~\ref{assumption:B1} holds and allows us to vary the divergence of the predictive distribution via $e_U$ to approximate the empirical distribution. We therefore have $m^6$ trainable parameters per layer. We use cross-entropy with a floor of $\epsilon = 10^{-3}$ to cap the maximum prompt-wise loss at $-\log(\epsilon) \approx 6.91$ nats, to avoid taking $\log(0)$. First, we test the effect of memory on prediction by varying the observation channel noise, which affects the decay rate of $\Delta_M$. For the stability tests, we run three trials: varying the sources mixing rate $\alpha = (1-\delta(T))(2 - \delta (Q))$ by sweeping $T, Q$ values, varying the granularity of the action space via $n$, and varying the granularity of the state space via $m$.

\subsection{Memory of data in Transformer training}
We conduct three trials with varying Markov characteristics of the input-generating processes by setting a garbling rate $\xi = 0, 0.5, 1$.

For each trial, we test the model on a distinct data set from the training data set, with the same number of data points and quantization levels. The testing data generation uses the same transition kernel as the training data, which the model is expected to learn. We plot the last index input (before), predicted (after) and label states on the $y$-axis for each input state on the x-axis, which can be seen in the top row of Figure \ref{fig:xi}. Note that the $x$-axis has fewer than $\tilde K = 98$ ticks in all cases when we combine identical input labels by converting state labels into distribution labels. We also plot the difference between the label state and the last index input state (before), and the label state and the predicted state (after), which can be seen in the bottom row of Figure \ref{fig:xi}.

\begin{figure}[H]
    \centering
    \includegraphics[width=0.37\linewidth]{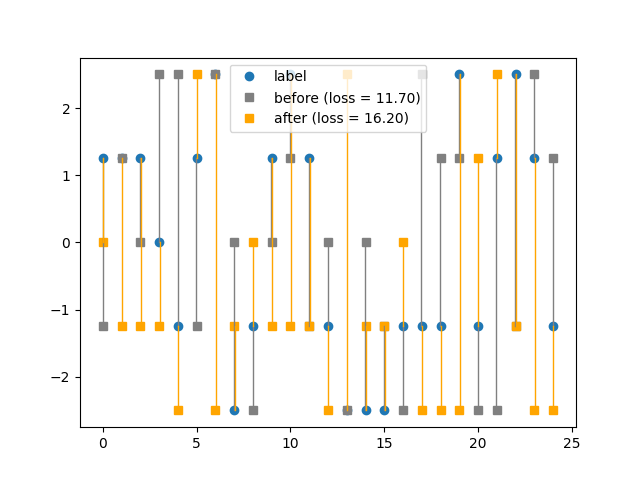}
    \includegraphics[width=0.37\linewidth]{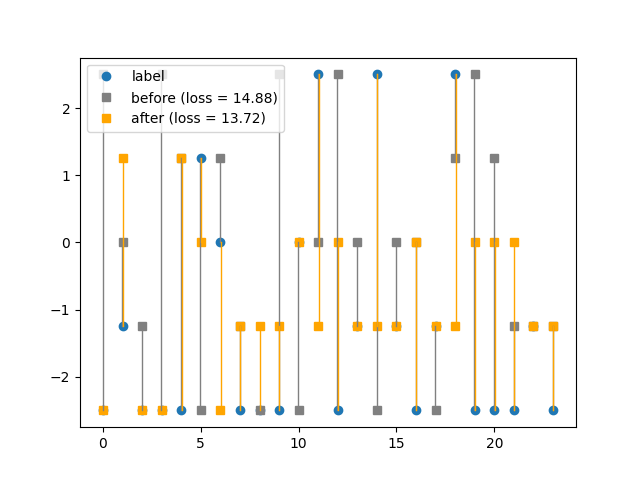}
    \includegraphics[width=0.37\linewidth]{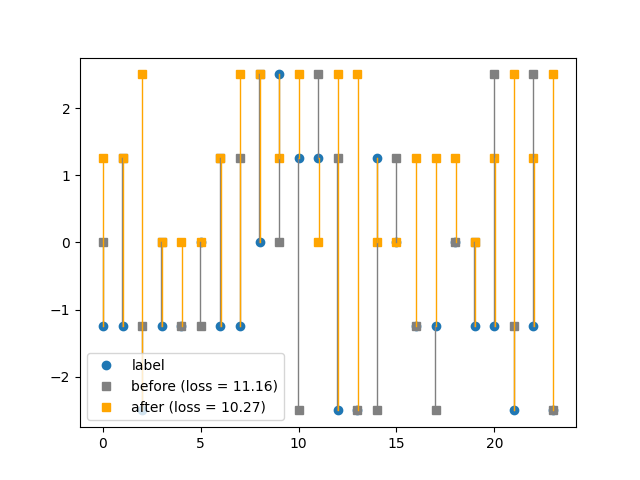}
\end{figure}
\begin{figure}[H]
    \centering
    \includegraphics[width=0.37\linewidth]{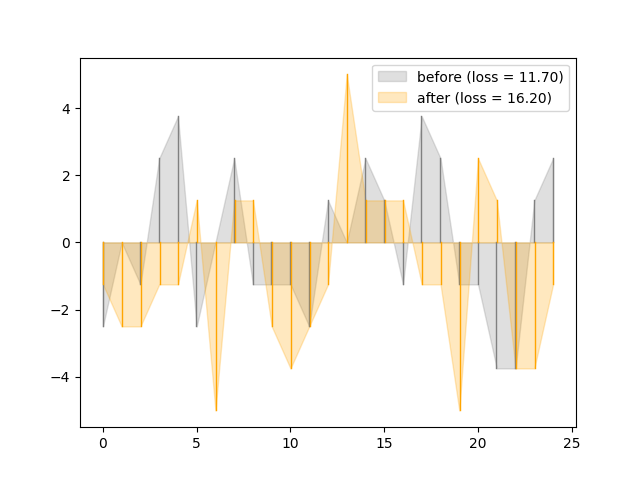}
    \includegraphics[width=0.37\linewidth]{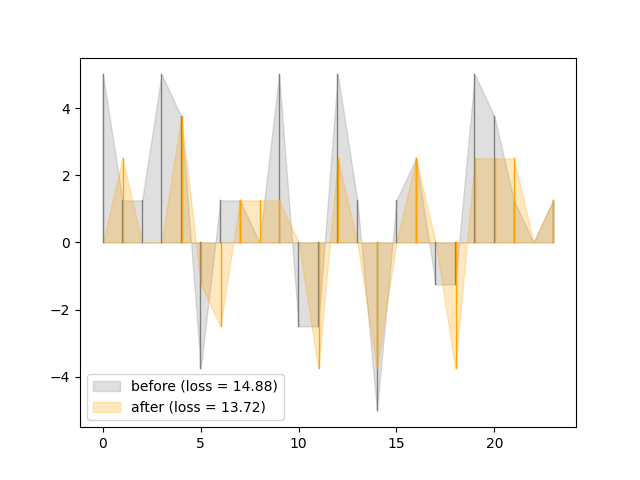}
    \includegraphics[width=0.37\linewidth]{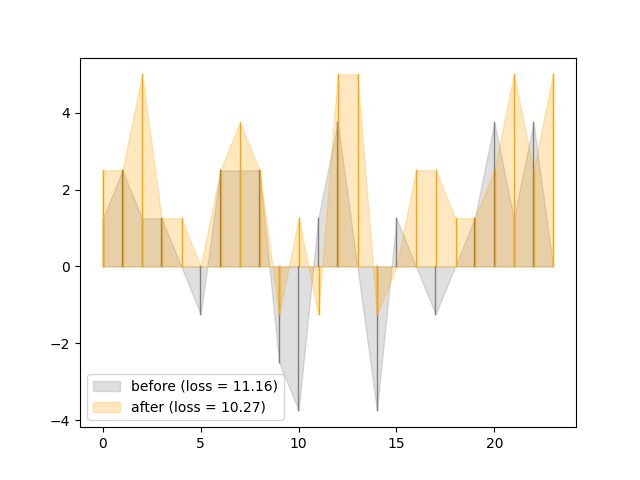}
    \caption{The first group shows the value of the input, prediction, and the label. The second group illustrates the losses calculated from discrepancy between the predicted states and the labels.}
    \label{fig:xi}
\end{figure}

Final loss decreases as the source becomes more Markovian. However, the variation between the initial data set losses show that the difference between the original loss and the final loss is highest (even increases after training) in the completely non-Markovian case, but remains relatively unchanging between the partially non-Markovian and completely Markovian case. 

\begin{figure}[H]
    \centering
    \includegraphics[width=0.5\linewidth]{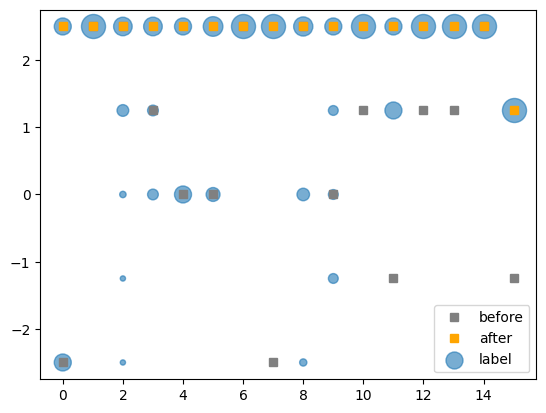}
    \caption{The radius of the label circle corresponds to occurrences in the distribution}
    \label{fig:placeholder}
\end{figure}

\subsection{Filtering via Transformers Trained with Optimal Control}

\subsubsection{Effect of mixing rate} \label{sub:mixing_experiment}
Our goal is to control the Markovianness of our generated input sequence $\{Y_t\}_{t = 0}^{\tilde K}$ via $\alpha$ described by \cite{mcdonald2020exponential}. $\alpha = (1-\delta(T))(2-\delta(Q))$ signifies the rate of agreement between two posterior estimates beginning at different starting state distributions given observations $Y_{[0,t]}$, where $\alpha \approx 1$ has slow or no convergence and $\alpha << 1$ has fast convergence to a single estimate. 

We hand-pick matrices for transition and measure kernels $T$ and $Q$ on finite spaces to obtain the desired Dobrushin coefficients $\delta(T)$ and $\delta(Q)$, determining $\alpha$. We took $T, Q \in \RR^{3\times 3}$, mapping between $n=3$ quantized states. Several examples for $T$ are shown below:

$$\delta \left(\begin{bmatrix} 0 & 1 & 0 \\ 0 & 0 & 1 \\ 1 & 0 & 0 \\ \end{bmatrix} \right) = 0,\qquad
\delta \left (\begin{bmatrix} 1/6 & 2/3 & 1/6 \\ 1/6 & 1/6 & 2/3 \\ 2/3 & 1/6 & 1/6 \\ \end{bmatrix} \right) = \frac{1}{2},\qquad
\delta \left (\begin{bmatrix} 1/3 & 1/3 & 1/3 \\ 1/3 & 1/3 & 1/3 \\ 1/3 & 1/3 & 1/3 \\ \end{bmatrix} \right) = 1
$$
with the matrices used for different Dobrushin coefficients $\delta(T)$ following a similar pattern of uniformly distributing mass outside the single-step transition matrix, which is shown above on the left. We construct the $Q$ matrix entry values to obtain the desired Dobrushin coefficients $\delta(Q)$ with a similar matrix pattern but different row ordering.

This setup allows us to vary $\delta(T)$ and $\delta(Q)$ independently by selecting the corresponding matrix to observe the independent effect of the transition and observation kernel on the Transformer's performance. We test six different values for each $\delta(T), \delta(Q)$, for 36 total trials. 

We generate an input sequence by taking observation $Y_t \sim Q \delta_{X_t}$ of the underlying Markov process $(X_t)_{t = 0}^{\tilde K}$, beginning at $X_0$ and evolving as $X_{t+1} \sim T \delta _ {X_t}$. 
The numerical experiment ran for 90 hours. Each training set was comprised of $\tilde K = 98$ (prompt, label) pairs. The sweep varied the transition and observation kernel independently over the irregularly spaced Dobrushin coefficient set $\{0, \frac{1}{6}, \frac{1}{3}, \frac{1}{2}, \frac{2}{3}, 1\}$ selected to show interesting behaviour at small Dobrushin coefficients. The matrices ranged from fully deterministic cyclic to fully uniform/mixing. We conduct 3 trials and average the loss for each value of $\delta(T)$ and $\delta(Q)$. In doing so, we obtain the losses for Dobrushin coefficients ranging from 0 to 1, as shown in Figure \ref{fig:colourmap}.

\begin{figure}[h]
    \centering
    \includegraphics[width=0.5\linewidth]{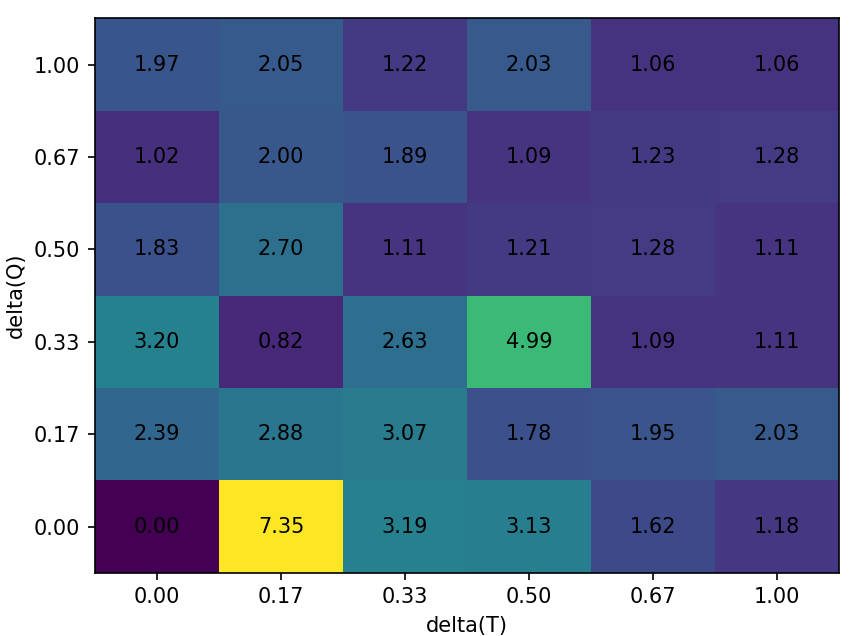}
    \caption{Test prediction loss for varying values of $\delta(Q)$ and $\delta(T)$}
    \label{fig:colourmap}
\end{figure}

Note that the right-most column corresponds to $\alpha = 0$, the bottom-left cell to $\alpha = 2$, and other values can be found via the equation for $\alpha$. The loss is the lowest at purely non-mixing, deterministic kernels at $(\delta(T), \delta(Q)) = (0, 0)$. At $\delta(T) \geq 0.67$ the loss is consistently low regardless of $Q$. 

The near-deterministic region is prone to occasional confident incorrect predictions, which is strongly punished by cross-entropy loss. This is seen at $(0.17, 0.00)$ and $(0.50, 0.33) $ with loss 7.35 and 4.99, respectively. These occasional errors may be due to the relatively low sample size of 3; thus, regional trends may be more indicative of Transformer performance.   

Overall, the bottom-left quadrant (high $\alpha$) attains the highest average loss (2.84), while the top-right quadrant (low $\alpha$) attains the lowest (1.26), suggesting a positive relationship between stability and Transformer performance. While this trend is partially driven by the outlier at $(\delta(T), \delta(Q)) = (0.17, 0)$, if we exclude it, the bottom left quadrant average is still elevated (2.42), so the overall direction still holds. 

Thus, as $\delta(T), \delta(Q)$ increase and $\alpha$ decreases, loss tends to decrease while stability increases. This experiment suggests that stability is positively related to Transformer performance.

\subsubsection{Effect of action space quantization}

For this trial, we keep $\delta(Q)=\delta(T) = \frac{1}{3}$ to attain $\alpha = \frac{10}{9}$, and sweep $m$, the quantization level of the action space parameters, through $\{1, 2, 3, 4, 5\}$. We keep the action space constant with $n = 3$ for all trials. The resulting loss is displayed in Figure \ref{fig:m-loss}. 

\begin{figure}[H]
    \centering
    \includegraphics[width=0.45\linewidth]{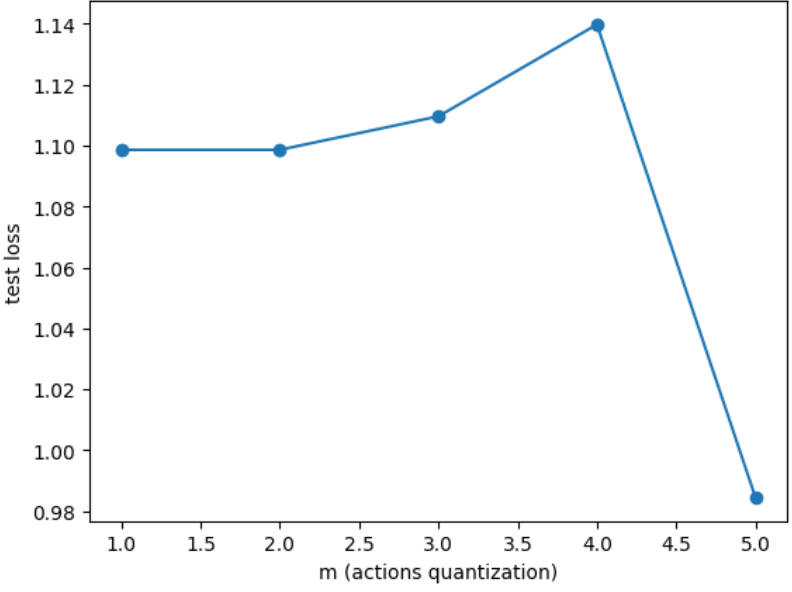}
    \caption{Test prediction loss for varying values of m}
    \label{fig:m-loss}
\end{figure}

There is a nonlinear dependence on the granularity of the action space for smaller values of $m$. Loss is higher for very coarse action spaces, and actually increases as we refine the action space, only to decrease for the highest tested value of $m = 5$. Note that due to computational challenges of using dynamic programming, we only tested relatively coarse action spaces. 

A possible explanation for the non-monotonic pattern is the following: for coarse action spaces ($m \in \{1, \dots, 4\}$), the model struggles to cleanly fit to the training data set, and distributes the output for uniformly to hedge against badly punished incorrect guesses. All coarse-spaced losses cluster around $1.12 \pm 0.02$, so the difference is small, and any increase with $m$ may reflect overfitting with insufficient generalization due to the small $m$. That is, when more fine-grained models make more confident, non-distributional guesses, they end up even more incorrect because the parameters overfit to loose data trends. For $m=5$, the model begins to confidently and correctly model more localized data set tendencies, not solely distributional. 

\subsubsection{Effect of state space quantization}

We now sweep varying values of $n$, the quantization level of the state space, keeping the transition kernel Dobrushin coefficients $\delta(T)$ and $\delta(Q)$, and the action space constant. We take $m = 5$, $\tilde K = 400$, and $\delta(Q)=\delta(T) = \frac{1}{3}$ for all trials. Note that although $\delta(Q)$ and $\delta(T)$ are constant, the transition kernels have different coefficients due to the difference in state space shape. We create a general kernel matrix $\Omega$ using the equation shown below. Given an action space with $n$ levels and a desired Dobrushin coefficient $\delta(\Omega)$, we take $\Omega = \frac{\delta(\Omega)}{n}J + (1-\delta(\Omega))A$ where $J$ is the all-ones matrix and $A$ is any permutation matrix. Note that constructing $\Omega$ this way ensures that the desired $\delta(\Omega)$ is always obtained. Note also that this corresponds to the $T, Q$ used in \ref{sub:mixing_experiment} with $n = 3$, taking
\[
    A = \begin{bmatrix} 0&1&0 \\ 0 & 0 & 1\\ 1 & 0 & 0 \end{bmatrix}.
\]
Figure \ref{fig:n predictions} shows the trained model's predicted distributions for varying state space quantizations, where the size of the squares corresponds to the mass in the predicted measure. 

\begin{figure}[H]
    \centering
    \includegraphics[width=0.45\linewidth]{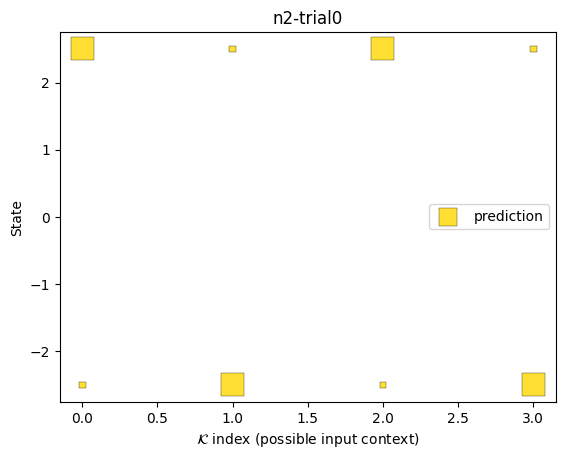}\includegraphics[width=0.45\linewidth]{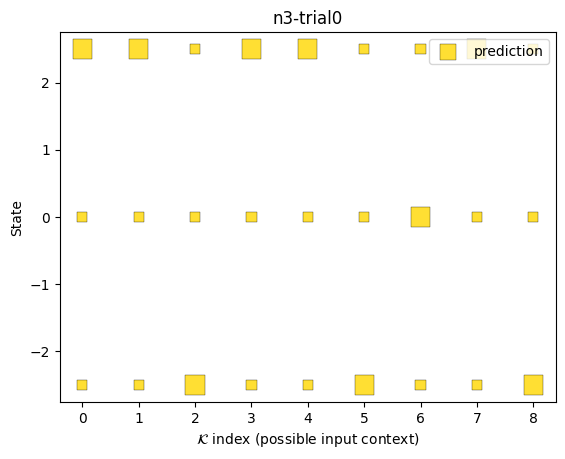}
\end{figure}
\begin{figure}[H]
    \centering
    \includegraphics[width=0.45\linewidth]{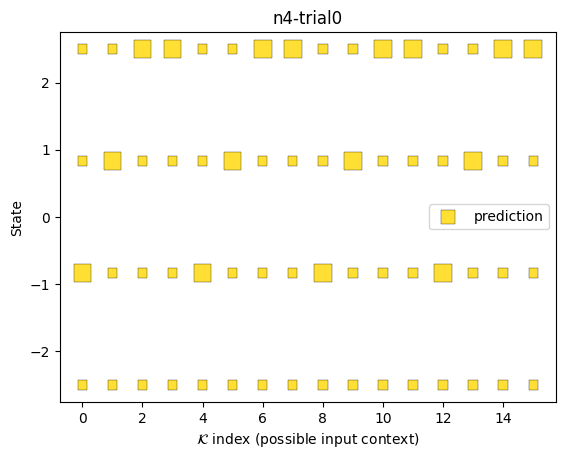}\includegraphics[width=0.45\linewidth]{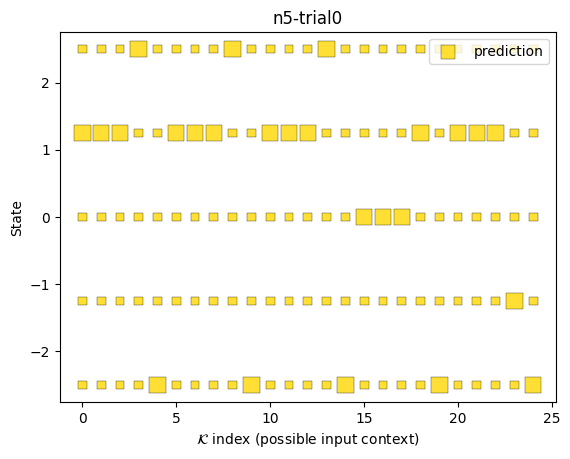}\\
    \includegraphics[width=0.45\linewidth]{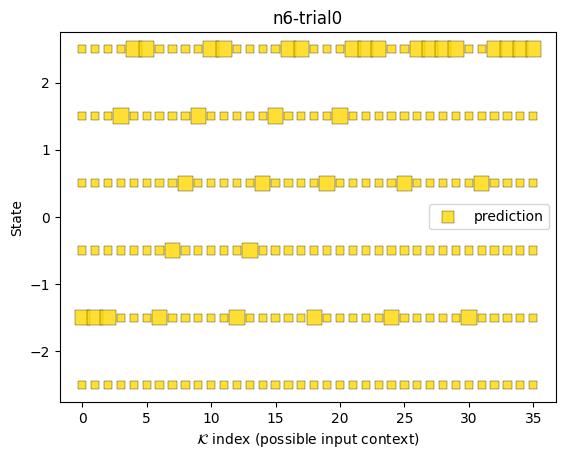}
    \caption{Shape of the predicted distributions with respect to varying state quantization level $n$.}
    \label{fig:n predictions}
\end{figure}

These predictions resulted in actual test losses and per-state losses plotted with linear interpolation in Figure \ref{fig:n-loss} for varying state quantization levels:

\begin{figure}[H]
    \centering
    \includegraphics[width=0.48\linewidth]{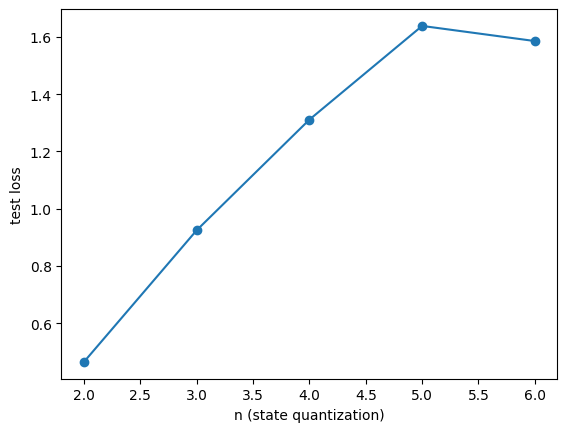}
    \hfill
    \includegraphics[width=0.48\linewidth]{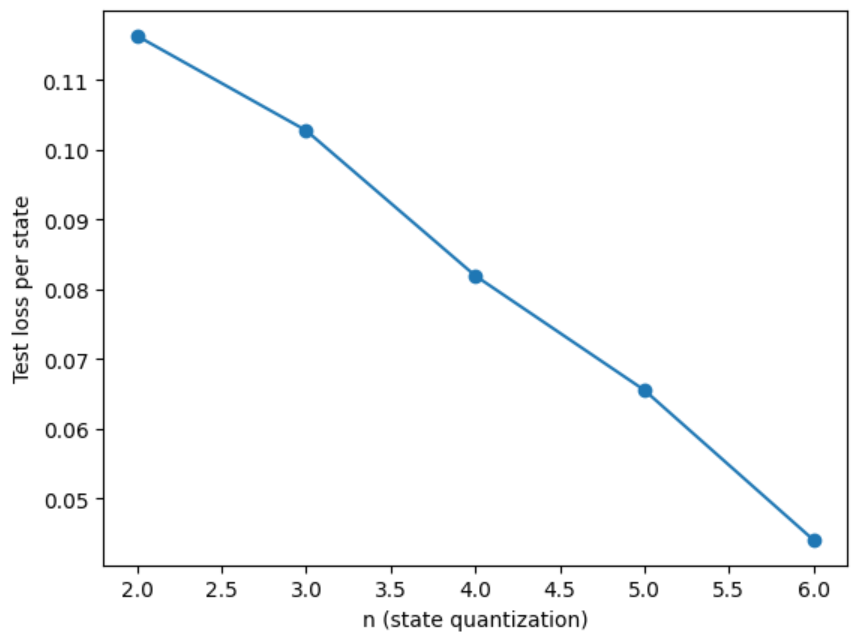}
    \caption{Test prediction loss (left) and per-state loss (right) for varying values of n}
    \label{fig:n-loss}
\end{figure}

Actual test loss tends to increase with $n$ in Figure
\ref{fig:n-loss} (left), except at $n = 6$, where the trend breaks and test loss decreases. However, since cross-entropy has no intrinsic normalization for the size of the state space, this loss increase reflects the increased difficulty in the task of classification, which is expected since there are more possible states to classify. Once we normalize by the number of possible states ($n^N$) in Figure \ref{fig:n-loss} (right), per-state loss monotonically decreases near-linearly. Thus, model performance outpaces the growth rate of the state space itself.

The cumulative training runs took 75 hours. The runtime breakdown for each trial is shown in Figure \ref{fig:n-runtime}.
\begin{figure}[H]
    \centering
    \includegraphics[width=0.7\linewidth]{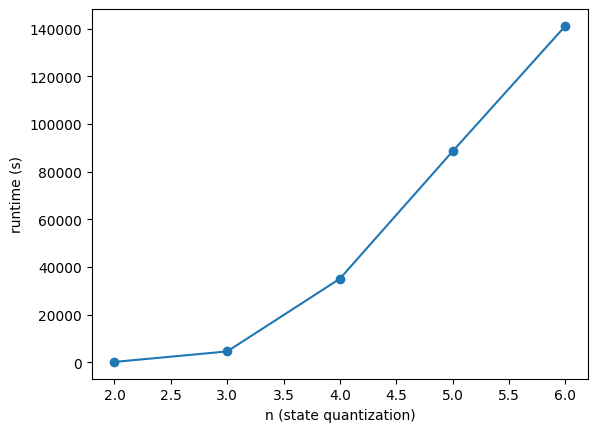}
    \caption{Numerical experiment runtime for varying values of n}
    \label{fig:n-runtime}
\end{figure}

A least-squares quadratic fit of the runtime in seconds as a function of $n$ yields $t_\text{runtime}(n) = 8542.9n^2 - 31702.9 n + 27200.0$ with $R^2 \approx 0.994$.

\section{Conclusion}\label{sec:Conclusion}

We studied next-token prediction by optimally trained Transformers for observations generated by a latent process. We truncated the context to a length $M$, and quantified the resulting error with entropy gaps for stationary ergodic process; with filter stability results for $N$th order Markov sources with the explicit rate $\boldsymbol\alpha^{\lfloor M/N \rfloor}$, where $\boldsymbol\alpha$ is a mixing rate determined by the Dobrushin coefficients of the transition kernel and the observation kernel. We then constructed a Transformer that realizes the quantized finite-window predictor, trained it with dynamic programming over a quantized set of weights. We finally showed that the resulting learning error can be decomposed into a sum of four factors: loss due to finite memory, quantization error, sample complexity error, and error due to insufficiency of architecture. Finally, through extensive experiments, we measured how important the memory of data is in Transformer training, and the effect of the mixing rate in filtering via Transformers.

\section*{Funding}

The first author was supported by the Scientific and Technological Research Council of Türkiye (TÜBİTAK) through the 2211-E program. The third author was supported by Turkish Academy of Sciences through the 2025 GEBIP Award. The last author was supported in part by the Natural Sciences and Engineering Research Council (NSERC) of Canada.
\newpage

\appendix

\section{The Training Algorithm}

\begin{algorithm}[H]
\caption{Next Token Prediction with Dynamic Programming}
\begin{algorithmic}[1]

\State \textbf{Input:} data set $\tilde D = \{(\bm x ^k, \tilde y^k)\}_{k=1}^ {\tilde K}$, horizon $T$, quantization levels $(l, n, m)$

\State Construct embedding space grid $\XX_n$ from $\XX$ and for $x\in \XX$ define 
$Q_n(\cdot) = \arg \min_{x'\in \XX_n}||\cdot-x'||$

\State Construct finite action set $\UU_m = \{U^1, \dots U^m\} \subset \UU$

\State Construct measure grid $\mathcal P^{(l)} (\XX_n)$ and for $\nu \in \PXn$ define 
$$R_l(\nu)  = \arg \min _{\hat \nu \in \mathcal P^{(l)} (\XX_n)} W_2(\nu, \hat \nu)$$

\State \textbf{for} $k = 1, \dots , \tilde K$ \textbf{do}

\State \quad $\hat \mu_0^k \leftarrow R_l (\frac{1}{N}\sum _{i=1}^N\delta_{Q_n(p_i, x^{i,k})})$

\State \quad $\hat y^k \leftarrow R_l( {\frac{
\sum_{j=1}^{\tilde {K}} \delta _ {\tilde y^j} \mathbbm{1}_{\{\bm{x}^j = \bm{x}^k\}}}{\sum_{j=1}^{\tilde {K}} \mathbbm{1}_{\{\bm{x}^j = \bm{x}^k\}}}})$

\State \textbf{end for}

\State Construct data set $D = \{(\bm x ^k, \hat y^k)\}_{k=1}^K$ 

\State Define dynamics for $U \in \UU _m$
$$\Phi ^{(l, n, m)}(\hat \mu, U) = R_l(\Phi^{(n)}(\hat \mu, U))$$

\State Define terminal cost $$\mathfrak C _ {l, n, m, T} (\hat {\bm \mu}) = \frac{1}{K}\sum_{k = 1} ^ K \mathrm{CE}(y^k, \text{softmax}(E_U {x_{T, \gamma}^{k, L}}))$$

\State \textbf{for} $t = T-1$ down to $0$ \textbf{do}
\State \quad \textbf{for all} $\hat{\bm \mu} = (\hat \mu ^1, \dots , \hat \mu ^K) \in \mathcal P^l(\XX_n)^K$ \textbf{do} 

\State \quad \quad 
$$\mathfrak C _ {l, n, m, t} (\hat {\bm \mu}) = \min_{U \in \UU_m} \mathfrak C _ {l, n, m, t+1} (\Phi ^{(l, n, m)}(\hat {\mu}^1, U), \dots ,\Phi ^{(l, n, m)}(\hat {\mu}^K, U))$$

\State \quad \quad 
$$\gamma ^*_t(\hat {\bm \mu}) = \arg \min _{U \in \UU_m} \mathfrak C _ {l, n, m, t+1} (\Phi ^{(l, n, m)}(\hat {\mu}^1, U), \dots ,\Phi ^{(l, n, m)}(\hat {\mu}^K, U))$$

\State \quad \textbf{end for}
\State \textbf{end for}
\State \textbf{Forward Pass:}

\State \textbf{for} $t = 0, \dots , T-1$ \textbf{do}

\State \quad $U^*_t \leftarrow \gamma ^*_t(\hat \mu _t ^1, \dots , \hat \mu _t ^K)$

\State \quad \textbf{for} $k = 1, \dots K$ \textbf{do}

\State \quad \quad $\hat \mu^k_{t+1} \leftarrow \Phi ^{(l, n, m)}(\hat \mu^W_{K,t}, U^*_t)$

\State \quad \textbf{end for}
\State \textbf{end for}
\State \textbf{Output:} $(U^*_t) ^{T -1}_{t = 0}$

\end{algorithmic}
\end{algorithm}

\section{Proofs of the Results in Section~\ref{sec:problem}}

\subsection{Auxiliary Results}

\begin{lemma}\label{def:QuantizationBound}
    Under Assumptions~\ref{a:BaseAssumptions}-(A2, A3), for all $r\in\ZZ_+$ and $s\in\NN$, $\max_{Y_{0:r}\in\YY^r}\|P_r(\,\cdot\mid Y_{0:r}) - \Xi_s(q^s_\#P_r)(\,\cdot\mid Y_{0:r})\|_\TV \leq L_{\psi}/s$.
\end{lemma}
\begin{proof}
    Taking the Radon-Nikodym derivative of $P_r$ with respect to $\mu$, we get a $\mu$-integrable function $p_r(y) := \sum_{z\in S}\psi(y\mid z)\,\Pr(Z_{r+1} = z\mid Y_{0:r})$. Then, for all $y,y'\in\YY$,
    \begin{align*}
        |p_r(y) - p_r(y)| &= \big|\sum_{z\in S}(\psi(y\mid z) - \psi(y'\mid z))\,\Pr(Z_{r+1} = z\mid Y_{0:r})\big|\\
        &\leq \sum_{z\in S}|\psi(y\mid z) - \psi(y'\mid z)|\Pr(Z_{r+1} = z\mid Y_{0:r}\\
        &\leq L_\psi\rho(y, y')\underbrace{\sum_{z\in S}\Pr(Z_{r+1} = z\mid Y_{0:r})}_{1} \tag{by Assumption~\ref{a:BaseAssumptions}-(A2)}\\
        &= L_\psi\rho(y, y').
    \end{align*}
    So, $p_r$ is also $L_\psi$-Lipschitz. For each $i=1,\ldots,R(s)$, on the cell $C(s,i)$, the $\mu$-density of $\Xi_s(q^s_\#P_r)(\,\cdot\mid Y_{0:r})$, denoted $\hat{p}_r$, is constant and equal to the $\mu$-average of $p_r$. Hence,
    \[
        \sup_{y\in C(s,i)}|p_r(y) - \hat{p}_r(y)| \leq L_\psi\operatorname{diam}(C(s,i)) \leq \frac{2L_\psi}{s}.
    \]
    We use the density representation of the total variation distance:
    \[
        \|P_r(\,\cdot\mid Y_{0:r}) - \Xi_s(q^s_\#P_r)(\,\cdot\mid Y_{0:r})\|_\TV = \frac{1}{2}\int_\YY \mu(\de y)\,|p_r(y) - \hat{p}_r(y)| \leq \frac{1}{2}\cdot\frac{2L_\psi}{s} = \frac{L_\psi}{s}.
    \]
    The above bound does not depend on the choice of trajectory $Y_{0:r}$ at any step of this proof, thereby, the claim holds and the proof is now concluded.
\end{proof}

\begin{lemma}\label{lem:Helper1}
    Under Assumptions~\ref{a:BaseAssumptions}-(A1) and (A2), for all $\mathbf{y} = (y_1,\ldots,y_M), \mathbf{y}'=(y_1',\ldots,y_M')\in\YY^M$, for all $r\geq M$, we have
    \[
        \begin{split}
            \|\Pr(Z_{r-M+1:r}\in\cdot\mid Y_{r-M+1:r} = \mathbf{y}) - \Pr(Z_{r-M+1:r}\in\cdot\mid &Y_{r-M+1:r} = \mathbf{y}')\|_\TV\\
            &\leq \frac{L_\psi}{q_{\rm min}}\sum_{i=1}^M\rho(y_i,y_i').
        \end{split}
    \]
\end{lemma}
\begin{proof}
    Let $\tilde P(\mathbf{z}) := \Pr\{Z_{r-M+1:r} = \mathbf{z}\}$ for all $\mathbf{z}\in S^M$. Conditional independence of $Y_{r-M+1:r}$ gives the likelihood
    \[
        \boldsymbol{\psi}_{\tilde{\mathbf{y}}}(\mathbf{z}) := \prod_{i=1}^M \psi(\tilde{y}_i\mid z_i)
    \]
    for all $\tilde{\mathbf{y}} := (\tilde{y}_1,\ldots,\tilde{y}_M)\in\YY^M$. The posterior is then
    \[
        \mu_{\tilde{\mathbf{y}}}(\mathbf{z}) := \frac{\tilde{P}(\mathbf{z})\,\boldsymbol{\psi}_{\tilde{\mathbf{y}}}(\mathbf{z})}{\sum_{\mathbf{z}'\in S^M}\tilde{P}(\mathbf{z}')\,\boldsymbol{\psi}_{\tilde{\mathbf{y}}}(\mathbf{z}')}.
    \]
    We continue after the following \emph{intermezzo}: For any two vectors $u,v\in\RR_+^d\setminus\{0\}$ for some $d\in\NN$, we denote their normalizations by $\hat{u} = u/\|u\|_1$ and $\hat{v} = v/\|v\|_1$. Then, we can write
    \begin{equation}\label{eq:Helper1-1}
        \hat{u} - \hat{v} = \frac{u-v}{\|u\|_1} + v\left( \frac{1}{\|u\|_1} - \frac{1}{\|v\|_1} \right).
    \end{equation}
    The term on the right hand side has the norm
    \begin{equation}\label{eq:Helper1-2}
        \left\|v \left(\frac{1}{\|u\|_1} - \frac{1}{\|v\|_1}\right)\right\|_1 \leq \frac{|\|v\|_1 - \|u\|_1}{\|u\|_1} \leq \frac{\|v-u\|_1}{\|u\|_1}.
    \end{equation}
    Therefore, by \eqref{eq:Helper1-1} and \eqref{eq:Helper1-2}, we have
    \begin{equation}
        \label{eq:Helper1-3}
        \|\hat{u} - \hat{v}\|_\TV = \frac{1}{2}\|\hat{u} - \hat{v}\|_1 \leq \frac{\|v-u\|_1}{\|u\|_1}.
    \end{equation}
    We resume from where we left: Let $\mathbf{a},\mathbf{b}\in\YY^M$ differ at the coordinate $k$ with $\rho(a_k, b_k) = \delta$. We define $c(\mathbf{z}) := \tilde{P}(\mathbf{z})\prod_{i\neq k}\psi(a_i\mid z_i) \geq 0$ (as a matter of fact, Assumption~\ref{a:BaseAssumptions}-(A1) gives that $\|c\|_1 \geq q_{\rm min}^{M-1}>0$). Writing $u(\mathbf{z}) = c(\mathbf{z})\psi(a_k\mid z_k)$ and $v(\mathbf{z}) = c(\mathbf{z})\psi(b_k\mid z_k)$. We have
    \begin{itemize}
        \item from Assumption~\ref{a:BaseAssumptions}-(A2) that $\|u-v\|_1 = \sum_{\mathbf{z}\in S^M} c(\mathbf{z})|\psi(a_k\mid z_k) - \psi(b_k\mid z_k)| \leq L_\psi\delta\|c\|_1$.
        \item from Assumption~\ref{a:BaseAssumptions}-(A1) that $\|u\|_1 = \sum_{\mathbf{z}\in S^M} c(\mathbf{z}) \psi(a_k\mid z_k) \geq q_{\rm min} \|c\|_1$.
    \end{itemize}
    Since $c$ is \emph{non-degenerate} ($\|c\|_1>0$), from the observation above and from \eqref{eq:Helper1-3}, we have
    \begin{equation}\label{eq:Helper1-4}
        \|\mu_{\mathbf{a}} - \mu_{\mathbf{b}}\|_\TV \leq \frac{L_\psi\delta}{q_{\rm min}}.
    \end{equation}

    Now define $\mathbf{y}^{(0)} = \mathbf{y}$, $\mathbf{y}^{(k)} = (y_1',\ldots,y_k',y_{k+1},\ldots,y_M)$ for all $k=1,\ldots,M$ and $\mathbf{y}^{(M)} = \mathbf{y}'$. Then, from \eqref{eq:Helper1-4}, we have
    \[
        \|\mu_{\mathbf{y}} - \mu_{\mathbf{y}'}\|_\TV \leq \sum_{k=1}^M \|\mu_{\mathbf{y}^{(k-1)}} - \mu_{\mathbf{y}^{(k)}}\|_\TV \leq \frac{L_\psi}{q_{\rm min}} \sum_{k=1}^M \rho(y_k, y_k').
    \]
    Seeing $\mu_{\mathbf{a}}$ is equivalent with $\Pr(Z_{r-M+1:r}\in\cdot\mid Y_{r-M+1:r} = \mathbf{a})$ concludes this proof.
\end{proof}

\begin{lemma}\label{lem:ErgodicityIsInherited}
    Under Assumption~\ref{a:Ergodicity}, $\{Y_r\}_{r\in\ZZ_+}$ is stationary and ergodic.
\end{lemma}
\begin{proof}
    By the randomization lemma \cite[Lemma~2.22]{kallenberg2021foundations}, there exists $H:S\times[0,1]\to\YY$ and an i.i.d. $[0,1]$-valued noise process $\{V_r\}_{r\in\ZZ_+}$ independent of $\{Z_r\}_{r\in\ZZ_+}$ such that $Y_r = H(Z_r, V_r)$ for all $r\in\ZZ_+$. Since mixing of an ergodic process with a noise is still ergodic \cite[Theorem~1.23]{walters2000introduction}, the claim follows.
\end{proof}

\begin{lemma}\label{lem:pomp}
    The block process $\{(\mathbf{X}_i, \mathbf{Y}_i)\}_{i\geq 0}$ is a partially observed Markov process with the state transition kernel $\mathbf{T}$ and the observation channel $\mathbf{Q}$.
\end{lemma}
\begin{proof}
    Fix $i\geq 0$. From \eqref{eq:block_state} and the definition of $\mathbf{T}$, $\mathbf{X}_{i+1}$ is generated from $\mathbf{X}_i$ via the kernel $\mathbf{T}$, hence it is Markov. By the stochastic realization results, each $Y_i$ is a function of $Z_i$ and an independent noise $V_i$. As we block, $\mathbf{Y}_i$ is then generated only by $\mathbf{X}_i$ via a noise block $\mathbf{V}_i$ for all $i$. Hence, $\{(\mathbf{X}_i, \mathbf{Y}_i)\}_{i\geq 0}$ is indeed a POMP.
\end{proof}

\subsection{Proof of Proposition~\ref{prop:Quantization}}
Since the prediction kernel is non-expansive in total variation distance (as all stochastic kernels are), writing $\mathbf{y}' := q^{M,s}(\mathbf{y})$ we get
\[
    \begin{split}
        \|P_r^M(\,\cdot\mid\mathbf{y}) - P_r^{M,s}(\,\cdot\mid\mathbf{q}^{M,s}(\mathbf{y}))\|_\TV &= \|\Pr(Z_{r-M+1:r}\in\cdot\mid Y_{r-M+1:r} = \mathbf{y})\mathrm{Pred}\\
        &\qquad- \Pr(Z_{r-M+1:r}\in\cdot\mid Y_{r-M+1:r} = \mathbf{y}')\mathrm{Pred}\|_\TV\\
        &\leq \|\Pr(Z_{r-M+1:r}\in\cdot\mid Y_{r-M+1:r} = \mathbf{y})\\
        &\qquad - \Pr(Z_{r-M+1:r}\in\cdot\mid Y_{r-M+1:r} = \mathbf{y}')\|_\TV.
    \end{split}
\]
Then, we apply Lemma~\ref{lem:Helper1} with the fact that $\rho(y_i, q^s(y_i)) \leq 1/s$ for all $i=1,\ldots,M$. The factor $2$ is due to each cell having diameter at most $2/s$ because of quantization.\hfill$\blacksquare$

\subsection{Proof of Proposition~\ref{prop:Info}}
By Lemma~\ref{lem:ErgodicityIsInherited}, for all $r\geq M$, $H(Y_{r+1}\mid Y_{r-M+1:r}) = H(Y_M\mid Y_{0:M-1}) = H_M$ and $H_{r+1} = H(Y_{r+1}\mid Y_{0:r})$ by definition. Then $\Delta_{M,r} = H_M - H_{r+1}$ for all $r\geq M$. So, $\sup_{r\geq M}\Delta_{M,r} = H_M - H_\infty$. Note that $H_M \geq H_\infty$ for all $M\in\NN$ and $H_M\to H_\infty$ as $M\to\infty$. Then, by monotone convergence theorem for a real-valued non-increasing sequence, $\Delta_M\downarrow 0$.

For the second claim, for any $r\geq M$, we have from Pinsker's inequality \cite[Theorem~7.10]{polyanskiy2025information}
\begin{align*}
    \EE_{Y_{0:r}}\,\|P_r - P_r^M\|_\TV &\leq \EE_{Y_{0:r}}\,\sqrt{\frac{1}{2}D_{\rm KL}(P_r\parallel P_r^M)} \\
    &\leq \sqrt{\frac{1}{2}\EE_{Y_{0:r}}\,D_{\rm KL}(P_r\parallel P_r^M)} = \sqrt{\frac{1}{2}\Delta_{M,r}},
\end{align*}
where the last inequality follows from Jensen's inequality for concave functions applied to $\sqrt{\cdot}$ and the last equality holds by \eqref{eq:CondMutualInfo} and \eqref{eq:CondMutualInfoIdentity} applied to $A=Y_{r+1}$, $B=Y_{0:r-M}$ and $C=Y_{r-M+1:r}$ and by the definition of $\Delta_{M,r}$ in \eqref{eq:MemoryGap}. Taking the suprema over $r\geq M$ of both sides concludes this proof.\hfill$\blacksquare$

\subsection{Proof of Proposition~\ref{prop:MarkovDobrushin}}
Define the one-step Bayesian update map
\[
    F(\lambda, y)(z) := \frac{\psi(y \mid z) (\lambda \kappa(z))}{\sum _{z' \in S} \psi(y \mid z')(\lambda \kappa(z'))}.
\]
We define the initializations of filters:
\[
    \begin{aligned}
        &\pi_0(\cdot) := \Pr(Z_{r-M}\in\cdot\mid Y_{0:r-M})\\
        &\pi_0' := \pi_Z
    \end{aligned}
\]
using which we write the filter update equations
\[
    \begin{aligned}
        &\pi_i := F(\pi_{i-1}, Y_{r-M+i})\\
        &\pi_i' := F(\pi_{i-1}', Y_{r-M+i})
    \end{aligned}\quad \text{ for all } i\in\{1,\ldots, M\}
\]
so that we have
\[
    \begin{aligned}
        &\pi_M(\cdot) := \Pr(Z_r\in\cdot\mid Y_{0:r})\\
        &\pi_M'(\cdot) := \Pr(Z_r\in\cdot\mid Y_{r-M+1:r})
    \end{aligned}
\]
If we define $\mathrm{Pred}(y\mid z) := \sum_{z'\in S}Q(y\mid z')\,\kappa(z'\mid z)$ for all $z\in S$ and $y\in\YY$, then we have $P_r = \pi_M\mathrm{Pred}$ and $P_r^M = \pi_M'\mathrm{Pred}$.

Using non-expansiveness of stochastic kernel in total variation distance, we get
\[
    \|P_r - P_r^M\|_\TV \leq \|\pi_M\mathrm{Pred} - \pi_M'\mathrm{Pred}\|_\TV \leq \|\pi_M - \pi_M'\|_\TV.
\]

Since $\pi_0' = \EE\pi_0$, then $\pi_0' \ll \pi_0$ almost surely, and hence $\pi_0'\kappa \ll \pi_0\kappa$ a.s. Now, conditioning on $\sigma(Y_{0:r-M})$, the process $\{(Z_{r-M+1+i}, Y_{r-M+1+i})\}_{i\geq 0}$ is then a partially observed Markov process whose initial state $Z_{r-M+1}$ has the law $\pi_0\kappa$. Then the sequences $\{\pi_i\}_{i\geq 1}$ and $\{\pi_i'\}_{i\geq 1}$ are its filters started from the true prior $\pi_0\kappa$ and the false prior $\pi_0'\kappa$. Then, \cite[Corollary~3.7]{mcdonald2020exponential} applies:
\[
    \EE[\|\pi_M - \pi_M'\|_\TV\mid Y_{0:r-M}] \leq \alpha^M \|\pi_0 - \pi_0'\|_\TV \leq \alpha^M.
\]
Taking the expectations of both sides concludes this proof. $\hfill\blacksquare$

\subsection{Proof of Proposition~\ref{prop:NMarkovDobrushin}}\label{proof:NMarkovDobrushin}
Define the one-step Bayesian update map
\[
    F(\lambda, \mathbf{y})(\mathbf{z}) := \frac{\boldsymbol{\psi}(\mathbf{y}\mid\mathbf{z})(\lambda\mathbf{T}(\mathbf{z}))}{\sum_{\mathbf{z}'\in S^N}\boldsymbol{\psi}(\mathbf{y}\mid\mathbf{z}')(\lambda\mathbf{T}(\mathbf{z}'))}
\]
for all $\lambda\in\mathcal{P}(S^N)$ and $\mathbf{y}\in\YY^N$ where $\boldsymbol{\psi}(\mathbf{y}\mid\mathbf{z}) := \prod_{i=1}^N \psi(y_i\mid z_i)$. We also define $\lambda_0' := \Pr(\mathbf{X}_0\in\cdot\mid\mathbf{Y}_{r-M+1:\res})$ and $\lambda_0 := \Pr(\mathbf{X}_0\in\cdot\mid\mathbf{Y}_{0:\res})$. Using these, we write the filter update equations:
\[
    \pi_i := F(\pi_{i-1},\mathbf{y}_i) \text{ with } \pi_0 = \lambda_0
\]
and
\[
    \pi_i' := F(\pi_{i-1}',\mathbf{y}_i) \text{ with } \pi_0' = \lambda_0'.
\]
Conditionally on $\mathbf{X}_0$, the past $Y_{0:\res}$ is independent of $(\mathbf{X}_{1:m},\mathbf{Y}_{1:m})$. Hence,
\[
    \pi_m = \Pr(\mathbf{X}_m\in\cdot\mid Y_{0:r}) \text{ and } \pi_m' = \Pr(\mathbf{X}_m\in\cdot\mid Y_{r-M+1:r}).
\]
As $Y_{r+1}$ depends on the past only through $\mathbf{X}_m$, the predictive distributions take the shape $P_r = \pi_m\Pred$ and $P_r^M = \pi_m'\Pred$. As stochastic kernels are non-expansive in total variation, we have
\[
    \|P_r-P_r^M\|_{\rm\scriptscriptstyle TV} = \|\pi_m\Pred-\pi_m'\Pred\|_{\rm\scriptscriptstyle TV} \leq \|\pi_m -\pi_m'\|_{\rm\scriptscriptstyle TV}.
\]

Condition on $Y_{0:\res}$ and apply the same steps in the proof of Proposition~\ref{prop:MarkovDobrushin} (see §\ref{proof:NMarkovDobrushin}): applying \cite[Corollary~3.7]{mcdonald2020exponential}, we get:
\[
    \EE_{\lambda_0'}\,\left[\|\pi_m-\pi_m'\|_{\rm\scriptscriptstyle TV} \mid Y_{0:\res}\right] \leq \boldsymbol{\alpha}^{\lfloor M/N\rfloor}
\]
Taking the expectation and writing $m=\lfloor M/N\rfloor$ conclude this proof.\hfill$\blacksquare$

\subsection{Proof of Theorem~\ref{thm:Main}}
For each $r\geq M$, we decompose $\EE\,\|P_r - P_r^{M,s}\|_\TV$ into two parts:
\[
    \EE_{Y_{0:r}}\,\|P_r - P_r^{M,s}\|_\TV \leq \EE_{Y_{0:r}}\,\|P_r - P_r^M\|_\TV + \EE_{Y_{0:r}}\,\|P_r^M - P_r^{M,s}\|_\TV.
\]
On the right hand side, the first summand is bounded by Proposition~\ref{prop:MarkovDobrushin} for (i) and by Proposition~\ref{prop:NMarkovDobrushin} for (ii). The second summand is handled by Proposition~\ref{prop:Quantization}. As a note, expectation in the second term is not needed since the objects vary only because of the quantized data. Since the bound holds independent of $r\geq M$, the statement follows.\hfill$\blacksquare$

\subsection{Proof of Corollary~\ref{cor:MemoryAndQuantization}}
We allocate $\epsilon$ equally to the two summand in the bound of Theorem~\ref{thm:Main}. First, we solve the inequality
\[
    \frac{\epsilon}{2} \geq \boldsymbol{\alpha}^{\lfloor M/N\rfloor}
\]
for $M$ to get:
\[
    M \geq N\frac{\log{(2/\epsilon)}}{\log(1/\boldsymbol{\alpha})}.
\]
Then we solve 
\[
    \frac{\epsilon}{2} \geq \frac{2ML_\psi}{q_{\rm min}s}
\]
for $s$. This concludes the proof for (ii). The argument for (i) is same with the exception that $N$ does not appear since the latent source is Markovian, i.e. $N=1$.\hfill$\blacksquare$

\section{Proofs of the Results in Section~\ref{sec:Approximation}}

\subsection{Proof of Lemma~\ref{lem:RoleOfFeedForward}}
Write $x_+ := \sigma(x)$ and $x_- := \sigma(-x)$ for the coordinate-wise positive and negative parts of $x$, respectively. A sequence of simple calculations shows that
\begin{align*}
    W_W\sigma(W_Ax+W_b) &= W_W\,\sigma\bigg(\begin{bmatrix}
        x\\
        -x
    \end{bmatrix}\bigg)\\
    &= \begin{bmatrix}
        I_d & -I_d
    \end{bmatrix}\begin{bmatrix}
        x_+\\
        x_-
    \end{bmatrix} = x_+ - x_- = x,
\end{align*}
which concludes this proof.\hfill$\blacksquare$

\subsection{Proof of Lemma~\ref{lem:Invariance}}
We prove by induction. At $t=0$, this is true by construction in \eqref{eq:embedding}. Suppose that the claim is true for $t\in\{0,\ldots,M-1\}$. Then we get
\[
    x_{t+1}^i = x_t^i + \sum_{j=1}^Ma_t^{ij}G_{\rho(t)}\,x_t^j
\]
as in \eqref{eq:Updates} for all $i\in\{1,\ldots,M\}$, where
\begin{equation}\label{eq:AttentionWeights}
    a_t^{ij} := \frac{\mathrm{e}^{\beta\langle W_{Q,t} x_t^i, W_{K,t}x_t^j\rangle}}{\sum_{j'=1}^M \mathrm{e}^{\beta\langle W_{Q,t} x_t^i, W_{K,t}x_t^{j'}\rangle}},\quad i,j\in\{1,\ldots,M\}.
\end{equation}
By the construction of $G_{\rho(t)}$ in~\eqref{eq:G_j}, $\sum_{j=1}^M a_t^{ij}G_{\rho(t)}\,x_t^j$ has no non-zero component in positional encodings and token values, regardless of the attention scores. Therefore, the positional encodings and token values of $x_{t+1}^i$ are equal to those of $x_t^i$.\hfill$\blacksquare$

\subsection{Proof of Lemma~\ref{lem:Attention}}
Using the definitions of $W_{Q,t}$, $W_{K,t}$, $A_t$ and $\mathrm{sel}_t$, we have
\[
   \begin{split}
        \langle W_{Q,t} x_t^i, W_{K,t}x_t^j\rangle &= \langle cA_t\Pi_{\rm pos} x_t^i, \Pi_{\rm pos} x_t^j\rangle\\
        &= c\langle A_te_i, e_j\rangle = \begin{cases}
        c \langle e_{\rho(t)}, e_j\rangle, & i = M, \\
        c \langle e_i, e_j\rangle, & i\neq M.
    \end{cases} = c\,\chi_j(\mathrm{sel}_t(i)).
   \end{split}
\]
Consequently, the attention scores are in the following form: If $j = \mathrm{sel}_t(i)$, then
\[
    \begin{split}
        a_t^{i\mathrm{sel}_t(i)} &= \frac{\exp(\beta c\,\chi_j(\mathrm{sel}_t(i)))}{\sum_{j'=1}^M\exp(\beta c\,\chi_{j'}(\mathrm{sel}_t(i)))}\\ 
        &= \frac{\exp(\beta c\,\chi_j(\mathrm{sel}_t(i)))}{\exp(\beta c\,\chi_j(\mathrm{sel}_t(i))) + \sum_{j'=1, j'\neq j}^M\exp(\beta c\,\chi_{j'}(\mathrm{sel}_t(i)))} = \frac{\exp(\beta c)}{\exp(\beta c) + M - 1}.
    \end{split}
\]
If $j\neq\mathrm{sel}_t(i)$, then there exists $j'\in\{1,\ldots,M\}$ with $j'\neq j$ such that $j' = \mathrm{sel}_t(i)$ since $\rho$ is a bijection. Hence, we have
\[
    a_t^{ij} = \frac{\exp(\beta c)}{\exp(\beta c) + M - 1}.
\]
This conclude the proof.\hfill$\blacksquare$

\subsection{Proof of Lemma~\ref{lem:Error}}
Fix $t\in\{0,\ldots,M-1\}$. By Lemma~\ref{lem:Invariance}, the values $G_{\rho(t)}x_t^j = \iota(w_j) (0_M, 0, e_j, 0_R)$ are deterministically obtained for all $j\in\{1,\ldots,M\}$, and they are bounded by $R(s)$ in $2$-norm. Then the single layer can be calculated by
\begin{align*}
    \bigg\|\sum_{j=1}^M a_t^{Mj}G_{\rho(t)}x_t^j - G_{\rho(t)}x_t^{\rho(t)}\bigg\|_2 &= \bigg\| (a_t^{M\rho(t)} - 1) G_{\rho(t)}x_t^{\rho(t)} + \sum_{\substack{j=1\\j\neq\rho(t)}}^M a_t^{Mj}G_{\rho(t)}x_t^j\bigg\|_2\\
    &\leq(1 - a_t^{M\rho(t)})\|G_{\rho(t)}x_t^{\rho(t)}\|_2 + \sum_{j=1,\,j\neq\rho(t)}^M a_t^{Mj}\|G_{\rho(t)}x_t^{j}\|_2
    \intertext{By Lemma~\ref{lem:Invariance}, $\|G_{\rho(t)}x_t^j\|_2 = \iota(w_j) \leq R(s)$ for all $j\in\{1,\ldots,M\}$, $\sum_{\substack{j=1\\j\neq\rho(t)}}^M a_t^{Mj} = 1 - a_t^{M\rho(t)}$, and $1-a_t^{M(t)} \leq (M-1)e^{-\beta c}$ by Lemma~\ref{lem:Attention}, then}
    &\leq R(s)(M-1)e^{-\beta c} + R(s)(M-1)e^{-\beta c}\\
    &= 2(M-1)R(s)e^{-\beta c}.
\end{align*}
Since both recursions of the form $z_{t+1} = z_t + \star$, the errors are added:
\[
    x_M^M - \bar{x}_M^M = \sum_{t=0}^{M-1}\sum_{j=1}^M a_t^{Mj}G_{\rho(t)}x_t^j - G_{\rho(t)}x_t^{\rho(t)}.
\]
Passing to the norm and applying the triangle inequality, we obtain
\[
    \|x_M^M - \bar{x}_M^M\|_2 \leq 2M(M-1)R(s)e^{-\beta c},
\]
which verifies the statement's claim.\hfill$\blacksquare$

\subsection{Proof of Lemma~\ref{lem:relu_interpolation}}
We write $\sigma$ for ${\tt ReLU}$. Since $\zeta_k$'s are distinct, the set $H_{k,j}:=\{a\in\RR^d : \langle a, \zeta_k - \zeta_j\rangle = 0\}$ is a hyperplane in $\RR^d$, therefore, is of $p$-dimensional Lebesgue measure $0$. So, the finite union $H = \bigcup_{k\neq j}H_{k,j}$ is also of $d$-dimensional Lebesgue measure $0$. Let us pick $a\in\RR^d\setminus H$ so that $t_k$'s defined by $t_k := \langle a, \zeta_k\rangle$ are pairwise distinct. Relabel them so that $t_1 < \cdots < t_n$.

We define the slopes
\[
    s_k = \frac{v_{k+1} - v_k}{t_{k+1} - t_k} \in \RR^{R(s)}, \quad t\in\RR_+
\]
for all $k=1,\ldots,n$ with $s_0 = 0$. Set
\[
    g(t) := v_1 + \sum_{k=1}^{n-1}(s_k-s_{k-1})\,\sigma(t - t_k).
\]
Then, with an induction argument on $k$, one can easily show that $g(t_k) = v_k$ for every $k=1,\ldots,n-1$.

Finally, we compose $\phi(\zeta) := g(\langle a, \zeta\rangle)$, which gives the desired interpolation property $\phi(\zeta_k) = g(t_k) = v_k$. For the Lipschitz constant, take $\zeta, \zeta'\in\RR^d$
\[
    \begin{split}
        \|\phi(\zeta) - \phi(\zeta')\|_2 &\leq \sum_{k=0}^{n-1}\|s_k - s_{k-1}\|_2\,|\sigma(\langle a,\zeta\rangle) - \sigma(\langle a,\zeta'\rangle)|\\
        &\leq \sum_{k=0}^{n-1}\|s_k - s_{k-1}\|_2\,|\langle a,\zeta\rangle-\langle a,\zeta'\rangle|\quad (\star)\\
        &\leq\sum_{k=0}^{n-1}\|s_k - s_{k-1}\|_2\,|\langle a, \zeta - \zeta'\rangle| \leq \sum_{k=0}^{n-1}\|s_k - s_{k-1}\|_2\,\|a\|_2\,\|\zeta-\zeta'\|_2.
    \end{split}
\]
Therefore $L_\phi := \|a\|_2\sum_{k=0}^{n-1}\|s_k - s_{k-1}\|_2 < + \infty$ since it is a finite sum. In above, $(\star)$ is due to $1$-Lipschitzness of ${\tt ReLU}$ and the last majorization is due to the Cauchy-Bunyakovsky-Schwarz inequality.\hfill$\blacksquare$

\subsection{Proof of Lemma~\ref{lem:Explicit_a}}
We prove step-by-step:

\noindent{\it Proof of (i).} By Lemma~\ref{lem:Error}, we have $\bar{x}_M^M := [e_M\;\iota(w_M)\;(\iota(w_1),\ldots,\iota(w_M))\; 0_R]$. With the choice of $a$, the result is evident. For the norm:
\[
    \|a\|_2^2 = \sum_{j=1}^M (R(s)^2)^{j-1} = \frac{R(s)^{2M} - 1}{R(s)^2-1} \leq \frac{R(s)^{2M} - 1}{R(s)} \leq R(s)^{2M-1}.
\]
This gives that $\|a\|_2 \leq R(s)^{M - 1/2} \leq R(s)^M$. Finally, considering that $\iota(w) \mapsto \iota(w) - 1$ is an injective map from $\{1,\ldots,R(s)^M\}$ to $\{0,\ldots,R(s)^M-1\}$ for any $w\in\Lambda_s^M$, $t_{k+1} - t_k \geq 1$ can be obtained via the same argument in Corollary~\ref{cor:Separation}. 

\noindent{\it Proof of (ii).} First we note that $\hat{P}_r^{M,s}$ satisfies: $\min_{\hat{y}\in\Lambda_s,\, \hat{w}\in\Lambda_s^M}\hat{P}_r^{M,s}(\hat{y}\mid\hat{w}) \geq q_{\rm min}\,\mu_{\rm min}^s$ since for any $\hat{y}\in\Lambda_s$ and $w\in\Lambda_s^M$ we have
\begin{align*}
    \hat{P}_r^{M,s}(\hat{y}\mid \hat{w}) &= \EE[\Pr(Y_{r+1}\in C(s,i)\mid Z_{r+1})\mid q^{M,s}(Y_{r-M+1:r}) = \hat{w}]\\
    &= \int_{C(s,i)}\mu(\de y)\,\psi(y\mid Z_{r+1})\\
    &\geq q_{\rm min}\int_{C(s,i)}\mu(\de y)\tag{by Assumption~\ref{a:BaseAssumptions}-(A2)}\\
    &\geq q_{\rm min}\,\mu_{\rm min}^s.\tag{by Assumption~\ref{a:BaseAssumptions}-(A4)}
\end{align*}
Since the targets are of the form $v_w = \log{\hat{P}_r^{M,s}(\,\cdot\mid w)}$, we have $\{(v_w)_y\}_{y\in\Lambda_s}\in[-\log{\frac{1}{q_{\rm min}\,\mu_{\rm min}^s}}, 0]$. Fix $k=1,\ldots,R(s)^M-1$, then
\[
    \|v_w\|_2 = \sqrt{\sum_{y\in\Lambda_s} (v_w)_y^2} \leq \log{\frac{1}{q_{\rm min}\,\mu_{\rm min}^s}}\sqrt{R(s)}.
\]
From here, by an application of the triangle inequality, we obtain
\[
    \|v_{k+1} - v_k\|_2 \leq 2\log{\frac{1}{q_{\rm min}\,\mu_{\rm min}^s}}\sqrt{R(s)}.
\]
With $t_{k+1} - t_k\geq1$ from (i), it follows that $\|s_k\|_2 \leq 2\log{\frac{1}{q_{\rm min}\,\mu_{\rm min}^s}}\sqrt{R(s)}.$ Likewise, this leads us to $\|s_{k+1} - s_k\|_2 \leq 4\log{\frac{1}{q_{\rm min}\,\mu_{\rm min}^s}}\sqrt{R(s)}$. Hence,
\[
    \begin{split}
        L_\phi &\leq \|a\|_2 \sum_{k=1}^{R(s)^M-1}\|s_{k+1} - s_k\|_2 \leq R(s)^{M - 1/2}(R(s)^M-1)4\sqrt{R(s)}\log{\frac{1}{q_{\rm min}\,\mu_{\rm min}^s}}\\
        &= 4R(s)^M(R(s)^M-1)\log{\frac{1}{q_{\rm min}\,\mu_{\rm min}^s}},
    \end{split}
\]
which gives the bound in the statement.\hfill$\blacksquare$

\subsection{Proof of Theorem~\ref{thm:ApproximatingTransformer}}
Fix $w\in\Lambda_s^M$. Let $z(w) := \Pi_{\rm log}x_{M+1}^M$ for the logits produced from the interpolator layer and $\bar{z}(w) := \log{\hat{P}_r^{M,s}(\,\cdot\mid w)}$ for the log-targets. Let $\{x_t^{M,(w)}\}_{t=0}^{M+1}$ be the hard-choice trajectory as in Lemma~\ref{lem:Error}. By Corollary~\ref{cor:Separation}, the readouts are separated by at least $1$ in norm, which enables us to invoke Lemma~\ref{lem:relu_interpolation} to construct a one layer {\tt ReLU} network $\phi: \RR^d\to\RR^{R(s)}$ with $\phi(\bar{x}_M^{M,(w)}) = \bar{z}(w)$. Since $\hat{P}_r^{M,s}(\,\cdot\mid w)$ is a probability vector, we get
\[
    \mathrm{softmax}(\bar{z}(w))_v = \frac{\exp(\log{\hat{P}_r^{M,s}(v\mid w)})}{\sum_{v'\in\Lambda_s}\exp(\log{\hat{P}_r^{M,s}(v'\mid w)})} = \frac{\hat{P}_r^{M,s}(v\mid w)}{\sum_{v'\in\Lambda_s}\hat{P}_r^{M,s}(v'\mid w)} = \hat{P}_r^{M,s}(v\mid w)
\]
exactly. Let us compute the error now: By Lemma~\ref{lem:Error}, we already know
\[
    \|x_M^{M,(w)} - \bar{x}_M^{M,(w)}\|_2 \leq 2M(M-1)R(s)e^{-\beta c}.
\]
Since $z(w) = \phi(x_M^{(w)})$ and $\bar{x}(w) = \phi(\bar{x}_M^{M,(w)})$, 
\begin{equation}\label{eq:Readouts}
    \begin{split}
        \|z(w) - \bar{z}(w)\|_\infty &\leq \|z(w) - \bar{z}(w)\|_2 \\
        &\leq L_\phi\|x_M^{M,(w)} - \bar{x}_M^{M,(w)}\|_2\\
        &\leq 8\log{\frac{1}{q_{\rm min}\mu_{\rm min}^s}}R(s)^{M+1}(R(s)^M-1)M(M-1)e^{-\beta c}
    \end{split}
\end{equation}
by Lemmata~\ref{lem:Explicit_a} and~\ref{lem:relu_interpolation}, which are invoked to obtain the fact that the {\tt ReLU} network is Lipschitz and the explicit Lipschitz constant $L_\phi$. We choose $\beta$ large enough to ensure that
\begin{equation}
    \label{eq:BetaChoice}
    8\log{\frac{1}{q_{\rm min}\mu_{\rm min}^s}}R(s)^{M+1}(R(s)^M-1)M(M-1)e^{-\beta c} \leq \frac{\epsilon}{2}.
\end{equation}
Solving for $\beta$ gives exactly the same bound in the hypothesis.

Recall that $\mathsf{T}(\,\cdot\mid w) = \mathrm{softmax}(z(w))$ and $\hat{P}_r^{M,s}(\,\cdot\mid w) := \mathrm{softmax}(\bar{z}(w))$. Then
\begin{align*}
    \|\hat{P}_r^{M,s}(\,\cdot\mid w) - \mathsf{T}(\,\cdot\mid w)\|_\TV &= \|\mathrm{softmax}(\bar{z}(w)) - \mathrm{softmax}(z(w))\|_\TV \\
    &= \frac{1}{2}\|\mathrm{softmax}(\bar{z}(w)) - \mathrm{softmax}(z(w))\|_1\\
    &= \frac{1}{2}\sum_{v\in\Lambda_s} |\mathrm{softmax}(\bar{z}(w))_v - \mathrm{softmax}(z(w))_v|\\
    &= \frac{1}{2}\sum_{v\in\Lambda_s} \mathrm{softmax}(\bar{z}(w))_v\bigg|\frac{\mathrm{softmax}(z(w))_v}{\mathrm{softmax}(\bar{z}(w))_v} - 1\bigg|\\
    &=\frac{1}{2}\sum_{v\in\Lambda_s} \mathrm{softmax}(\bar{z}(w))_v\bigg|\frac{e^{z(w)_v}}{\sum_{u\in\Lambda_s}e^{z(w)_u}}\cdot\frac{\sum_{u\in\Lambda_s}e^{\bar{z}(w)_u}}{e^{\bar{z}(w)_v}} - 1\bigg|\\
    &=\frac{1}{2}\sum_{v\in\Lambda_s} \mathrm{softmax}(\bar{z}(w))_v\bigg|\underbrace{e^{z(w)_v - \bar{z}(w)_v}}_{(a)}\underbrace{\frac{\sum_{u\in\Lambda_s}e^{\bar{z}(w)_u}}{\sum_{u\in\Lambda_s}e^{z(w)_u}}}_{(b)} - 1\bigg|
    \intertext{($\star$): Writing $\epsilon' := 8\log{\frac{1}{q_{\rm min}\mu_{\rm min}^s}}R(s)^{M+1}(R(s)^M-1)M(M-1)e^{-\beta c}$ from \eqref{eq:Readouts}, it follows that $(a)\in[e^{-\epsilon'}, e^{\epsilon'}]$, which also gives $e^{-\epsilon'}e^{\bar{z}(w)_u} \leq e^{z(w)_u} \leq e^{\epsilon'}e^{\bar{z}(w)_u}$. Summing over $u$ and taking the reciprocals yield the fact hat $(b)\in[e^{-\epsilon'}, e^{\epsilon'}]$ as well. Thus,}
    &\leq \max\{e^{2\epsilon'} - 1, 1 - e^{2\epsilon'}\} = e^{2\epsilon'} - 1.
\end{align*}
By the convexity of the exponential function $e^{2\epsilon'} - 1\leq 2\epsilon'$. Hence,
\[
    \|\hat{P}_r^{M,s}(\,\cdot\mid w) - \mathsf{T}(\,\cdot\mid w)\|_\TV \leq 2\epsilon'.
\]
From \eqref{eq:BetaChoice} and the definition of $\epsilon'$, we have 
\[
    \|\hat{P}_r^{M,s}(\,\cdot\mid w) - \mathsf{T}(\,\cdot\mid w)\|_\TV \leq \epsilon,
\]
which concludes the part (i) since the bounds are independent of the choice of $w\in\Lambda_s^M$.

{\it Proof of (ii).} By ($\star$), $\mathsf{T}(v\mid w) = \mathrm{softmax}(z(w))_v \geq \mathrm{softmax}(\bar{z}(w))_ve^{-2\epsilon'}$ for all $v\in\Lambda_s$. Since $\mathrm{softmax}(\bar{z}(w)) = \hat{P}_r^{M,s}(\,\cdot\mid w)$ and from the proof of Lemma~\ref{lem:Explicit_a},
\[
    \mathsf{T}(v\mid w)\geq  \mathrm{softmax}(\bar{z}(w))_ve^{-2\epsilon'} = \hat{P}_r^{M,s}(v\mid w)e^{-2\epsilon'} \geq q_{\rm min}\mu_{\rm min}^s e^{-2\epsilon'} \geq q_{\rm min}\mu_{\rm min}^s e^{-2\epsilon}.
\]
The bound is uniform for all $v\in\Lambda_s$, thereby, the claim follows.\hfill$\blacksquare$

\section{Proofs of the Results in Section~\ref{sec:Main}}

\subsection{Auxiliary Results}\label{sub:Auxiliary} In this part, we give all the necessary tools for proving Theorem~\ref{thm:Unified}. Throughout, we fix $M$ and $s$ in accordance with Corollary~\ref{cor:MemoryAndQuantization} and the conditions therein, that are; under Assumption~\ref{a:BaseAssumptions} and $\boldsymbol{\alpha} = (1-\delta(\mathbf{T}))(2-\delta(\mathbf{Q})) < 1$, $M$ and $s$ have to satisfy:
\begin{itemize}
    \item $M\geq N\lceil \log(2/\epsilon)/\log(1/\boldsymbol{\alpha})\rceil$,
    \item $s\geq 4ML_\psi/(q_{\rm min}\epsilon)$.
\end{itemize}
 
\begin{definition}\label{def:AligningQuantization}
    Let $M_\XX, M_\UU\in\NN$. We say that the triple $(n,m,l)\in\NN^3$ are \emph{aligning quantization levels} if the following hold:
    \begin{itemize}
        \item[(i)] $\frac{2M_\XX}{n-1} = \frac{1}{K_\XX}$ for some $K_\XX\in\NN$;
        \item[(ii)] $\frac{2M_\UU}{m-1} = \frac{1}{K_\UU}$ for some $K_\UU\in\NN$ and $c = M_\UU$, where $c$ is \eqref{eq:ConstantC};
        \item[(iii)] $M$ divides $l-1$.
    \end{itemize}
\end{definition}
 
Due to Definition~\ref{def:AligningQuantization}, we have that $\XX_n \supset ([-M_\XX, M_\XX] \cap \ZZ)^d$ and the grid of $\UU_m$ contains $[-M_\UU, M_\UU]\cap\ZZ$ in each entry. Indeed, since $M_\XX\in\NN$, every integer $p\in[-M_\XX, M_\XX]$ is the grid point $-M_\XX + \frac{2M_\XX\hat k}{n-1}$ with $\hat k = (p+M_\XX)K_\XX\in\{0,\ldots,n-1\}$, and the same argument applies to $\UU_m$. In other words, the state and action quantization preserve integer values within the diameter of the ambient spaces. This is an important feature since the weights in the gatherer layers in §\ref{sub:Gatherers} are $\{-1, 0, c, 1\}$-valued. If $(n,m,l)$ are aligning quantization levels, then the weights of §\ref{sub:Gatherers} live in $\UU_m$. On the other hand, the condition (iii) results in no error due to quantization of measure-states since $M\mid l-1$ ensures that any mass, which is an integer multiple of $\frac{1}{M}$, is a grid value of $\mathcal P_l(\XX_n)$, and hence remains invariant under the measure quantization.
 
\begin{lemma}\label{lem:GridExactness}
    Assume that $(n,m,l)\in\NN^3$ are chosen so that they are aligning quantization levels and that $M_\XX\geq R(s)$. Then, along the trajectory of the constructed Transformer, for any $w\in\Lambda_s^M$, the following assertions are true:
    \begin{itemize}
        \item[(a)] The weights defined in §\ref{sub:Gatherers} live in $\UU_m$;
        \item[(b)] $\Pi_{\rm pos}Q_n(x) = \Pi_{\rm pos} x$ and $\Pi_{\rm tok}Q_n(x) = \Pi_{\rm tok} x$ for every token $x$ along the trajectory;
        \item[(c)] the attention scores are exactly $c\chi_{\mathrm{sel}_t(i)}(j)$, and the measure states passing through the Transformer layers have atoms carrying mass exactly $1/M$;
        \item[(d)] the only quantization error is due to $Q_n$ being applied on the gather coordinates, which is of size at most $M_\XX\sqrt{d}/(n-1)$ per layer.
    \end{itemize}
\end{lemma}
\begin{proof}
    We prove step-by-step:
    
    \noindent{\it Proof of (a).} Recall that $W_W = [I_d\; -I_d]$ and $W_A = W_W^T$, therefore, all components of $W_W$ and $W_A$ are in $\{-1,0,1\}$. In addition, $W_b$ is the zero vector. $W_{K,t} = \Pi_{\rm pos}$ is a coordinate projection, hence, all components are $0,1$ as well. The value matrix $W_{V,t} = G_{\rho(t)} = [0_M\; 0\; e_{\rho(t)}\; 0_{R(s)}]^T\Pi_{\rm tok}$ has components with values in $\{0,1\}$ too. Finally, $W_{Q,t}$ is $cA_t\Pi_{\rm pos}$ where $A_t$ is a $0,1$ matrix. Because of the factor $c$, $W_{Q,t}$ is component-wise $\{0,c\}$-valued. By Definition~\ref{def:AligningQuantization}-(ii), $c = M_\UU$, so $-1,0,1,c$ are integers in $[-M_\UU,M_\UU]$, hence grid points of $\UU_m$. Therefore, $(W_{W,t}, W_{A,t}, W_{b,t}, W_{Q,t}, W_{K,t}, W_{V,t})\in\UU_m$.
 
    \noindent{\it Proof of (b).} From Definition~\ref{def:AligningQuantization}, $\XX_n$ contains all integers in $[-M_\XX, M_\XX]$ in each coordinate. Since $M_\XX\geq R(s)$ and $\iota(w_i) \in \{1,\ldots,R(s)\}$ for all $w_i\in\Lambda_s$, the token values are such integers, and so are the positional coordinates, which lie in $\{0,1\}$. These coordinates are not changed by the attention blocks (Lemma~\ref{lem:Invariance}, whose argument applies verbatim to the quantized trajectory). Since $\XX_n$ is a product grid, $Q_n$ acts coordinate-wise. Therefore the state quantizer $Q_n$ must fix positional and token value components of each token.
 
    \noindent{\it Proof of (c).} By (b), the attention scores are computed from the exact positional blocks with exact $W_{Q,t}, W_{K,t}$ by (a). Therefore, Lemma~\ref{lem:Attention} applies directly. For the weights, each measure is supported on $M$ atoms of mass $1/M$, since the positional encodings of the tokens are distinct, and $1/M = \frac{(l-1)/M}{l-1}$ is a grid value of $\mathcal P_l(\XX_n)$ by Definition~\ref{def:AligningQuantization}-(iii). Therefore $R_l$ fixes the measure states.
 
    \noindent{\it Proof of (d).} Each gather coordinate is of the form $\sum_{j=1}^M \hat{a}_t^{ij}\iota(w_j)$, which is not an integer in general. Therefore, $Q_n$ does not fix it, but moves it by at most the grid spacings, which is $M_\XX\sqrt{d}/(n-1)$.
\end{proof}
 
\begin{lemma}\label{lem:Stability}
    Assume that $(n,m,l)\in\NN^3$ are chosen so that they are aligning quantization levels and that $M_\XX, M_\UU$ satisfy~\eqref{eq:RangeConditions}. Then for every $w\in\Lambda_s^M$,
    \[
        \|\hat{z}(w) - \bar{z}(w)\|_\infty \leq M_\XX\sqrt{d}R(s)^{\frac{3M-1}{2}}\frac{M_\UU d_s}{m-1} + L_\phi\bigg(M\frac{M_\XX\sqrt{d}}{n-1} + 2M(M-1)R(s)e^{-\beta c}\bigg),
    \]
    where $\bar{z}$ is as in Theorem~\ref{thm:ApproximatingTransformer} and $\hat{z}$ is the logits produced by the quantized weights before the last application of $Q_n$.
\end{lemma}
\begin{proof}
    Let $\bar{x}_t^i$ be the unquantized trajectory given by Lemma~\ref{lem:Error}, $\hat{x}_t^i$ is given by the quantized weights, and let $x_t^i$ be the usual trajectory (neither a hard choice as in Lemma~\ref{lem:Error} nor under quantization like $\hat x_t^i$).
    
    There are two immediate observations: (i) The attention scores are computed from the positional encodings. From Lemma~\ref{lem:GridExactness}-(b,c) and Lemma~\ref{lem:Attention}, we know that both the quantized and the unquantized trajectories are computed under the same attention scores. We write $a_t^{ij}$ for the common attention scores. (ii) From Lemma~\ref{lem:GridExactness}-(b), the token coordinates are invariant under the state quantizer $Q_n$. Therefore $W_{V,t}\bar x_t^i = W_{V,t}\hat x_t^i$ for all $t\in\{0,\ldots,M-1\}$ and $i\in\{1,\ldots,M\}$. 
    
    For convenience, we write $A_i = \sum_{j=1}^M a_t^{ij}W_{V,t}\bar{x}_t^j$ for the $i$th token's attention part, which is same for both $\bar x_t^i$ and $\hat x_t^i$ by (i) and (ii). Then
    \[
        \begin{split}
             x_{t+1}^i - \hat x_{t+1}^i &= ( x_t^i + A_i) - Q_n(\hat x_t^i + A_i) \\
            &=  x_t^i + A_i - Q_n(\hat x_t^i + A_i) + \hat x_t^i - \hat x_t^i\\
            &= ( x_t^i - \hat x_t^i) + \big((\hat x_t^i + A_i) - Q_n(\hat x_t^i + A_i)\big).
        \end{split}
    \]
    Passing to the norm, we get
    \begin{equation}\label{eq:RecursiveBound}
        \begin{split}
            \| x_{t+1}^i - \hat x_{t+1}^i\| &\leq \| x_t^i - \hat x_t^i\| + \underbrace{\|(\hat x_t^i + A_i) - Q_n(\hat x_t^i + A_i)\|}_{\leq \frac{M_\XX\sqrt{d}}{n-1}},
        \end{split}
    \end{equation}
    for all $t\in\{0,\ldots,M-1\}$ and all $i\in\{1,\ldots,M\}$. Initial tokens given by $[e_i \; \iota(w_i) \; 0_M \; 0_{R(s)}]$ have integer coordinates, hence, invariant under $Q_n$. So $\| x_{0}^i - \hat x_{0}^i\| = 0$ for all $i\in\{1,\ldots,M\}$. Consequently, the error in the last gatherer layer's readout position is upper bounded as
    \begin{equation}\label{eq:Part1}
            \|\hat{x}_M^M - \bar x_M^M\| \leq \|\hat x_M^M - x_M^M\| + \|x_M^M - \bar x_M^M\| \leq M\frac{M_\XX\sqrt{d}}{n-1} + 2M(M-1)R(s)e^{-\beta c},
    \end{equation}
    where the first summand in the last bound is due to \eqref{eq:RecursiveBound} accumulated over $M$ layers and the the second summand is due to Lemma~\ref{lem:Error}.
 
    Now we arrive at the interpolator layer, in which we have $W_{V,M}$ is the zero matrix, hence the attention part does not contribute the final token. The logit part of the final layer is then given in terms of a one-layer feed-forward network of the form $\phi(\zeta) = W_{W,M}\sigma(W_{A,M}\zeta + W_{b,M})$. Let $\hat\phi$ be the corresponding feed-forward with quantized weights, i.e. $\hat\phi(\zeta) = \hat W_{W,M}\sigma(W_{A,M}\zeta + W_{b,M})$. We deliberately write $W_{A,M}$ and $W_{b,M}$ without hats because of Lemma~\ref{lem:Explicit_a}, which implies that their entries are integers in $[-M_\UU, M_\UU]$ by~\eqref{eq:RangeConditions}; hence, they are not affected by the quantization due to Definition~\ref{def:AligningQuantization}-(ii) (see the proof of Lemma~\ref{lem:TransformerPerformance}). So, for any $\zeta\in\XX$
    \[
        \hat\phi(\zeta) - \phi(\zeta) = (W_{W,M} - \hat W_{W,M})\sigma(W_{A,M}\zeta + W_{b,M}).
    \]
    In norm, we have
    \[
        \|\hat\phi(\zeta) - \phi(\zeta)\| \leq \|W_{W,M} - \hat W_{W,M}\|_F \bigg(\sum_{k=1}^{R(s)^M}\sigma(\langle a, \zeta\rangle - t_k)^2\bigg)^{\frac{1}{2}}.
    \]
    $\|W_{W,M} - \hat W_{W,M}\|_F \leq \frac{M_\UU d_s}{m - 1}$ due to action quantization. Indeed, the only entries of $W_{W,M}$ that are not integers are those of its logit block, which has $R(s)$ rows and $R(s)^M$ columns, so that there are at most $d_s^2 = R(s)^{M+1}$ of them. By the proof of Lemma~\ref{lem:Explicit_a}-(ii), they lie in $[-2\log\frac{1}{q_{\rm min}\mu_{\rm min}^s}, 2\log\frac{1}{q_{\rm min}\mu_{\rm min}^s}]\subset[-M_\UU, M_\UU]$, and each of them is moved by at most half of the grid spacing $\frac{2M_\UU}{m-1}$. On the other hand,
    \[
        \bigg(\sum_{k=1}^{R(s)^M}\sigma(\langle a, \zeta\rangle - t_k)^2\bigg)^{\frac{1}{2}} \leq \sqrt{R(s)^M}\|a\|_2 M_\XX \sqrt{d}. 
    \]
    By Lemma~\ref{lem:Explicit_a}, $a = (1,R(s),\ldots,R(s)^{M-1})$, hence $\|a\|_2 \leq R(s)^{M - 1/2}$, which implies 
    \begin{equation}\label{eq:Part2}
        \|\phi - \hat\phi\|_\infty \leq M_\XX\sqrt{d}R(s)^{\frac{3M-1}{2}}\frac{M_\UU d_s}{m-1}.
    \end{equation}
    By definition, $\hat z(w) = \hat\phi(\hat x_M^M)$ and $\bar z(w) = \phi(\bar x_M^M)$. With an ``add and subtract'' trick:
    \[
        \|\hat z(w) - \bar z(w)\|_\infty \leq \|\hat z(w) - \phi(\hat x_M^M)\|_\infty + \|\phi(\hat x_M^M) - \bar z(w)\|_\infty.
    \]
    For the first term, we have 
    \[
        \|\hat z(w) - \phi(\hat x_M^M)\|_\infty \leq M_\XX\sqrt{d}R(s)^{\frac{3M-1}{2}}\frac{M_\UU d_s}{m-1}
    \]
    from \eqref{eq:Part2}, and the second term is bounded as
    \[
        \|\phi(\hat x_M^M) - \bar z(w)\|_\infty \leq L_\phi\bigg(M\frac{M_\XX\sqrt{d}}{n-1} + 2M(M-1)R(s)e^{-\beta c}\bigg)
    \]
    by Lemma~\ref{lem:relu_interpolation} and \eqref{eq:Part1}. Therefore, we get
    \[
         \|\hat z(w) - \bar z(w)\|_\infty \leq M_\XX\sqrt{d}R(s)^{\frac{3M-1}{2}}\frac{M_\UU d_s}{m-1} + L_\phi\bigg(M\frac{M_\XX\sqrt{d}}{n-1} + 2M(M-1)R(s)e^{-\beta c}\bigg).\qedhere
    \]
\end{proof}
 
\begin{lemma}\label{lem:TransformerPerformance}
    Assume that $(n,m,l)\in\NN^3$ are chosen so that they are aligning quantization levels, that $M_\XX, M_\UU$ satisfy~\eqref{eq:RangeConditions}, that $R(s)\geq 2$, and let Assumptions~\ref{a:BaseAssumptions}-(A1, A3) hold. Let $\beta>0$ be such that $\eta(n,m,\beta)\leq 1$, where $\eta(n,m,\beta)$ is as in~\eqref{eq:EtaDef}. Then there exists an open-loop policy $\underline{\hat U} \in \UU_m^{M+1}$ such that, for every $w\in\Lambda_s^M$, $\sT^{\underline{\hat U}}(\,\cdot\mid w) = \mathrm{softmax}(\hat z_{\rm out}(w))$ for a logit vector $\hat z_{\rm out}(w)\in\RR^{R(s)}$ with
    \begin{equation}\label{eq:LogitError}
        \|\hat z_{\rm out}(w) - \bar z(w)\|_\infty\leq\eta(n,m,\beta),
    \end{equation}
    where $\bar z(w) = \log\hat P_r^{M,s}(\,\cdot\mid w)$ is as in Theorem~\ref{thm:ApproximatingTransformer}. Consequently,
    \[
        \max_{w\in\Lambda_s^M}\|\hat{P}_r^{M,s}(\,\cdot\mid w) - \mathsf{T}^{\underline{\hat U}}(\,\cdot\mid w)\|_\TV \leq \frac{\eta(n,m,\beta)}{2}.
    \]
    In particular, if $\beta$ is as in Theorem~\ref{thm:ApproximatingTransformer} for some $\epsilon_0\in(0,1/2]$, then
    \[
        \begin{split}
            \max_{w\in\Lambda_s^M}\|\hat{P}_r^{M,s}(\,\cdot\mid w) - \mathsf{T}^{\underline{\hat U}}(\,\cdot\mid w)\|_\TV &\leq \frac{\epsilon_0}{4} + \big(1+L_\phi M\sqrt{d}\big)\frac{M_\XX}{2(n-1)}\\
            &\qquad\qquad+ M_\XX\sqrt{d}R(s)^{\frac{3M-1}{2}}\frac{M_\UU d_s}{2(m-1)}.
        \end{split}
    \]
\end{lemma}
\begin{proof}
    We prove step-by-step:
 
    \noindent{\it Step 1.} Let $\underline{\hat U} = (\hat U_0,\ldots,\hat U_M)$ be the weights of §\ref{sec:Approximation} in which only the output matrix of the interpolator layer is moved onto the grid. That is, for $t\in\{0,\ldots,M-1\}$, $\hat U_t$ consists of the gatherer weights of §\ref{sub:Gatherers}, and $\hat U_0$ also holds the embedding $E_E$ in~\eqref{eq:embedding}. The last layer $\hat U_M$ consists of $W_{V,M} = 0$, so that its attention block is switched off, of the matrices $W_{A,M}$ and $W_{b,M}$ of the {\tt ReLU} network $\phi$ of Lemma~\ref{lem:relu_interpolation} with the choice of $a$ in Lemma~\ref{lem:Explicit_a}, of $\hat W_{W,M}$, the entry-wise closest grid point to $W_{W,M}$, and of the unembedding $E_U = \Pi_{\rm log}$. Since the feed-forward block has no output bias, the constant term $v_1$ of the interpolant $g$ in the proof of Lemma~\ref{lem:relu_interpolation} is produced by one additional hidden unit with zero input weights and bias $1$, whose output column is $v_1$. In addition, the interpolator layer carries the positional block through: $W_{A,M}$ has the additional rows $(e_i^T, 0, 0_M^T, 0_R^T)$ and $W_{W,M}$ the corresponding columns $(e_i, 0, 0_M, 0_R)$, $i=1,\ldots,M$. Since the positional coordinates are nonnegative, $\sigma$ acts on them as the identity. These additional weights do not interact with the logit block, and they keep Assumption~\ref{assumption:B1} in force after the last layer, so that the readout token is identified by its position.
 
    \noindent{\it Step 2.} We show that $\underline{\hat U}\in\UU_m^{M+1}$. The gatherer weights live in $\UU_m$ by Lemma~\ref{lem:GridExactness}-(a), and the entries of $E_E$, of $E_U$ and of the positional entries are in $\{0,1,\ldots,R(s)\}$. The entries of $W_{A,M}$ are $0$, $1$ or $R(s)^{j-1}\leq R(s)^{M-1}$, $j=1,\ldots,M$, the latter coming from $a$ in Lemma~\ref{lem:Explicit_a}. The entries of $W_{b,M}$ are $1$ (for the constant unit) and $-t_k$, where, by Lemma~\ref{lem:Explicit_a}-(i) and $R(s)\geq2$,
    \[
        0\leq t_k\leq\sum_{j=1}^M R(s)\,R(s)^{j-1} = R(s)\frac{R(s)^M-1}{R(s)-1}\leq 2R(s)^M.
    \]
    Therefore, all these entries are integers in $[-M_\UU, M_\UU]$ by~\eqref{eq:RangeConditions}, hence grid points of $\UU_m$ by Definition~\ref{def:AligningQuantization}-(ii). Finally, $\hat W_{W,M}$ is a grid point by construction. So, $\underline{\hat U}\in\UU_m^{M+1}$.
 
    \noindent{\it Step 3.} Fix $w\in\Lambda_s^M$. The positional carry uses exact integer weights, so Lemma~\ref{lem:Stability} applies without change to the logit vector $\hat z(w) = \hat\phi(\hat x_M^M)$ produced by the interpolator layer before the last application of $Q_n$. By the definition~\eqref{eq:EtaDef}, we have
    \[
        \|\hat z(w) - \bar z(w)\|_\infty\leq\eta(n,m,\beta) - \frac{M_\XX}{n-1}.
    \]
    Since $\bar z(w)\in[-\log\frac{1}{q_{\rm min}\mu_{\rm min}^s}, 0]^{R(s)}$ by the proof of Lemma~\ref{lem:Explicit_a}-(ii) and $\eta(n,m,\beta)\leq1$, the logit coordinates of the readout token lie in $[-M_\XX, M_\XX]$ by~\eqref{eq:RangeConditions}. Hence, the last application of $Q_n$ moves each of them by at most half of the grid spacing, that is, by at most $M_\XX/(n-1)$, and it fixes the positional block by Lemma~\ref{lem:GridExactness}-(b). The atoms of the last measure state still carry mass exactly $1/M$ since their positional encodings are distinct, so $R_l$ does not move them either by Lemma~\ref{lem:GridExactness}-(c). Consequently, $\sT^{\underline{\hat U}}(\,\cdot\mid w) = \mathrm{softmax}(\hat z_{\rm out}(w))$ with $\hat z_{\rm out}(w) := \Pi_{\rm log}Q_n(\hat x_{M+1}^M)$, and
    \[
        \|\hat z_{\rm out}(w) - \bar z(w)\|_\infty\leq\|\hat z_{\rm out}(w) - \hat z(w)\|_\infty + \|\hat z(w) - \bar z(w)\|_\infty\leq\frac{M_\XX}{n-1} + \eta(n,m,\beta) - \frac{M_\XX}{n-1} = \eta(n,m,\beta),
    \]
    which is~\eqref{eq:LogitError}.
 
    \noindent{\it Step 4.} Since $\hat P_r^{M,s}(\,\cdot\mid w) = \mathrm{softmax}(\bar z(w))$ (see the proof of Theorem~\ref{thm:ApproximatingTransformer}), Lemma~\ref{lem:Softmax}-(i) below and~\eqref{eq:LogitError} give
    \[
        \|\hat P_r^{M,s}(\,\cdot\mid w) - \sT^{\underline{\hat U}}(\,\cdot\mid w)\|_\TV\leq\frac{1}{2}\|\bar z(w) - \hat z_{\rm out}(w)\|_\infty\leq\frac{\eta(n,m,\beta)}{2}
    \]
    for every $w\in\Lambda_s^M$. Finally, if $\beta$ is as in Theorem~\ref{thm:ApproximatingTransformer} with $\epsilon=\epsilon_0$, then $\beta c\geq\log(4M(M-1)R(s)L_\phi/\epsilon_0)$, that is, $2L_\phi M(M-1)R(s)e^{-\beta c}\leq\epsilon_0/2$. Plugging this into~\eqref{eq:EtaDef} gives the last claim.
\end{proof}
 
\def\cG{\mathcal{G}}
\begin{lemma}\label{lem:SampleComplexity}
    Suppose Assumptions~\ref{a:BaseAssumptions}-(A1, A3) hold and that the data set is sampled as in the beginning of this section. Then the following assertions are true:
    \begin{itemize}
        \item[(i)] $p_r(w)\geq(q_{\rm min}\mu_{\rm min}^s)^M$ and $\min_{v\in\Lambda_s}\hat P_r^{M,s}(v\mid w)\geq q_{\rm min}\mu_{\rm min}^s$ for all $w\in\Lambda_s^M$.
        \item[(ii)] Let $\delta\in(0,1)$ and $\tilde K \geq 8(q_{\rm min}\mu_{\rm min}^s)^{-M}(M\log{R(s)} + \log{\frac{3}{\delta}})$. Then, with probability at least $1-\delta$, the following two events occur simultaneously:
        \[
            \cG_1 := \Big\{n_w\geq\frac{1}{2}\tilde Kp_r(w)\ \text{ for all } w\in\Lambda_s^M\Big\},
        \]
        \[
            \cG_2 := \bigg\{\|y^w - \hat P_r^{M,s}(\,\cdot\mid w)\|_\TV\leq\sqrt{\frac{C_\delta}{2n_w}}\ \text{ for all } w\in\Lambda_s^M \text{ with } n_w\geq1\bigg\},
        \]
        where $C_\delta := R(s)\log2 + M\log R(s) + \log(3/\delta)$. On $\cG_1$, every window $w\in\Lambda_s^M$ appears in the data set; hence, $K = R(s)^M$ and $\{\bm{x}^k : k\in\cK\} = \Lambda_s^M$.
        \item[(iii)] On $\cG_1\cap\cG_2$, we have
        \[
            \sum_{w\in\Lambda_s^M}p_r(w)\|y^w - \hat P_r^{M,s}(\,\cdot\mid w)\|_\TV\leq\sqrt{\frac{R(s)^MC_\delta}{\tilde K}}
        \]
        and
        \[
            \sum_{w\in\Lambda_s^M}\frac{n_w}{\tilde K}\|y^w - \hat P_r^{M,s}(\,\cdot\mid w)\|_\TV^2\leq\frac{R(s)^MC_\delta}{2\tilde K}.
        \]
    \end{itemize}
\end{lemma}
\begin{proof}
    We prove step-by-step:
 
    \noindent{\it Proof of (i).} Conditionally on $Z_{0:r+1}$, the observations $Y_0,\ldots,Y_{r+1}$ are independent with $Y_t\sim Q(\,\cdot\mid Z_t)$; see the representation $Y_t = H(Z_t, V_t)$ with i.i.d.\ noise independent of $\{Z_r\}_{r\in\ZZ_+}$ in the proof of Lemma~\ref{lem:ErgodicityIsInherited}. For every $v\in\Lambda_s$, the preimage $(q^s)^{-1}(v)$ is one of the cells $C(s,i)$, so that, for all $z\in S$,
    \[
        Q\big((q^s)^{-1}(v)\mid z\big) = \int_{(q^s)^{-1}(v)}\mu(\de y)\,\psi(y\mid z)\geq q_{\rm min}\,\mu\big((q^s)^{-1}(v)\big)\geq q_{\rm min}\,\mu_{\rm min}^s
    \]
    by Assumptions~\ref{a:BaseAssumptions}-(A1) and (A3). Therefore, for $w = (w_1,\ldots,w_M)\in\Lambda_s^M$,
    \[
        p_r(w) = \EE\bigg[\prod_{j=1}^MQ\big((q^s)^{-1}(w_j)\mid Z_{r-M+j}\big)\bigg]\geq(q_{\rm min}\,\mu_{\rm min}^s)^M.
    \]
    In particular, $\{q^{M,s}(Y_{r-M+1:r}) = w\}$ has positive probability, and, by the same conditional independence,
    \[
        \hat P_r^{M,s}(v\mid w) = \EE\big[Q\big((q^s)^{-1}(v)\mid Z_{r+1}\big)\,\big|\, q^{M,s}(Y_{r-M+1:r}) = w\big]\geq q_{\rm min}\,\mu_{\rm min}^s
    \]
    for all $v\in\Lambda_s$, as in the proof of Lemma~\ref{lem:Explicit_a}-(ii).
 
    \noindent{\it Proof of (ii).} Since the prompts are independent copies of $q^{M,s}(Y_{r-M+1:r})$, $n_w$ is binomially distributed with $\tilde K$ trials and success probability $p_r(w)$. The multiplicative Chernoff bound \cite[Corollary~4.2]{mulzer2018five} and (i) give
    \[
        \Pr\Big(n_w<\frac{1}{2}\tilde Kp_r(w)\Big)\leq\exp\Big(-\frac{\tilde Kp_r(w)}{8}\Big)\leq\exp\Big(-\frac{\tilde K(q_{\rm min}\mu_{\rm min}^s)^M}{8}\Big).
    \]
    A union bound over the $R(s)^M$ windows and the lower bound on $\tilde K$ yield
    \[
        \Pr(\cG_1^c)\leq R(s)^M\exp\Big(-\frac{\tilde K(q_{\rm min}\mu_{\rm min}^s)^M}{8}\Big)\leq R(s)^M\exp\Big(-M\log R(s) - \log\frac{3}{\delta}\Big) = \frac{\delta}{3}.
    \]
    For $\cG_2$, since the training pairs are independent and identically distributed, conditionally on the prompts $\bm{x}^1,\ldots,\bm{x}^{\tilde K}$, the labels are independent and $\tilde y^l\sim\hat P_r^{M,s}(\,\cdot\mid\bm{x}^l)$. Hence, conditionally on the prompts, for every $w$ with $n_w\geq1$, the label counts $(n_wy^w(v))_{v\in\Lambda_s}$ form a multinomial vector with $n_w$ trials and cell probabilities $\hat P_r^{M,s}(\,\cdot\mid w)$. By the Bretagnolle-Huber-Carol inequality \cite[Proposition~A.6.6]{vandervaart1996weak}, such a vector satisfies
    \[
        \Pr\bigg(\sum_{v\in\Lambda_s}\big|n_wy^w(v) - n_w\hat P_r^{M,s}(v\mid w)\big|\geq 2\sqrt{n_w}\lambda\bigg)\leq 2^{R(s)}e^{-2\lambda^2},\quad\lambda>0.
    \]
    Since $\|\cdot\|_\TV = \frac{1}{2}\|\cdot\|_1$ on $\cP(\Lambda_s)$, the event on the left hand side is $\{\|y^w - \hat P_r^{M,s}(\,\cdot\mid w)\|_\TV\geq\lambda/\sqrt{n_w}\}$. Choosing $\lambda = \sqrt{C_\delta/2}$, we get
    \[
        \Pr\bigg(\|y^w - \hat P_r^{M,s}(\,\cdot\mid w)\|_\TV>\sqrt{\frac{C_\delta}{2n_w}}\;\bigg|\;\bm{x}^1,\ldots,\bm{x}^{\tilde K}\bigg)\leq 2^{R(s)}e^{-C_\delta} = \frac{\delta}{3R(s)^M}.
    \]
    A union bound over at most $R(s)^M$ windows and taking the expectation give $\Pr(\cG_2^c)\leq\delta/3$. Therefore, $\Pr(\cG_1\cap\cG_2)\geq 1 - 2\delta/3\geq 1-\delta$. Finally, on $\cG_1$, we have $n_w\geq\tilde Kp_r(w)/2>0$ for every $w\in\Lambda_s^M$ by (i), that is, every window appears in the data set.
    \smallskip
    
    \noindent{\it Proof of (iii).} On $\cG_1\cap\cG_2$, we have $n_w\geq\tilde Kp_r(w)/2$, and therefore
    \[
        \|y^w - \hat P_r^{M,s}(\,\cdot\mid w)\|_\TV\leq\sqrt{\frac{C_\delta}{2n_w}}\leq\sqrt{\frac{C_\delta}{\tilde Kp_r(w)}}
    \]
    for every $w\in\Lambda_s^M$. By the Cauchy-Bunyakovsky-Schwarz inequality and $\sum_wp_r(w) = 1$,
    \[
        \sum_{w\in\Lambda_s^M}p_r(w)\|y^w - \hat P_r^{M,s}(\,\cdot\mid w)\|_\TV\leq\sqrt{\frac{C_\delta}{\tilde K}}\sum_{w\in\Lambda_s^M}\sqrt{p_r(w)}\leq\sqrt{\frac{C_\delta}{\tilde K}}\sqrt{R(s)^M}.
    \]
    For the second claim,
    \[
        \sum_{w\in\Lambda_s^M}\frac{n_w}{\tilde K}\|y^w - \hat P_r^{M,s}(\,\cdot\mid w)\|_\TV^2\leq\sum_{w\in\Lambda_s^M}\frac{n_w}{\tilde K}\frac{C_\delta}{2n_w} = \frac{R(s)^MC_\delta}{2\tilde K}.\qedhere
    \]
\end{proof}
 
The next lemma collects three elementary facts on the softmax map and the Kullback-Leibler divergence, which convert errors in the logits into errors in total variation and in the excess loss.
 
\begin{lemma}\label{lem:Softmax}
    Let $z,z'\in\RR^{R(s)}$, indexed by $\Lambda_s$, and let $\nu,P,P'\in\cP(\Lambda_s)$. Then the following assertions are true:
    \begin{itemize}
        \item[(i)] $\|\mathrm{softmax}(z) - \mathrm{softmax}(z')\|_\TV\leq\frac{1}{2}\|z-z'\|_\infty$;
        \item[(ii)] $D_{\rm KL}(\nu\parallel\mathrm{softmax}(z'))\leq D_{\rm KL}(\nu\parallel\mathrm{softmax}(z)) + 2\|\nu - \mathrm{softmax}(z)\|_\TV\|z-z'\|_\infty + \frac{1}{2}\|z-z'\|_\infty^2$;
        \item[(iii)] if $\min_{v\in\Lambda_s}P'(v)\geq p>0$, then $D_{\rm KL}(P\parallel P')\leq\frac{4}{p}\|P-P'\|_\TV^2$.
    \end{itemize}
\end{lemma}
\begin{proof}
    Throughout, write $h := z'-z$ and $\pi_\tau := \mathrm{softmax}(z+\tau h)$ for $\tau\in[0,1]$. For $\pi\in\cP(\Lambda_s)$, write $\langle\pi,h\rangle := \sum_v\pi(v)h_v$ and $\mathrm{Var}_\pi(h) := \sum_v\pi(v)(h_v - \langle\pi,h\rangle)^2$ for the mean and the variance of $h$ under $\pi$. We have that
    \begin{equation}
        \label{eq:HelperBound}
        \mathrm{Var}_\pi(h) \leq \langle \pi, h^2\rangle \leq \|h\|_\infty^2
    \end{equation}
    since $\pi$ is a probability vector and $h$ is bounded.
    
    Moreover, a direct computation gives $\frac{\de}{\de\tau}\pi_\tau(v) = \pi_\tau(v)(h_v - \langle\pi_\tau,h\rangle)$ for all $v\in\Lambda_s$.
 
    \noindent{\it Proof of (i).} By the fundamental theorem of calculus, the Cauchy-Bunyakovsky-Schwarz inequality and~\eqref{eq:HelperBound},
    \[
       \begin{split}
            \|\mathrm{softmax}(z') - \mathrm{softmax}(z)\|_1&\leq\int_0^1\sum_{v\in\Lambda_s}\pi_\tau(v)\,|h_v - \langle\pi_\tau,h\rangle|\,\de\tau\\
            &\leq\int_0^1\mathrm{Var}_{\pi_\tau}(h)^{\frac{1}{2}}\,\de\tau\leq\|h\|_\infty.
       \end{split}
    \]
    Since $\|\cdot\|_\TV = \frac{1}{2}\|\cdot\|_1$ on $\cP(\Lambda_s)$, (i) follows.
 
    \noindent{\it Proof of (ii).} Write $\mathrm{LSE}(u) := \log\sum_ve^{u_v}$, so that $\log\mathrm{softmax}(u)_v = u_v - \mathrm{LSE}(u)$. Since the softmax map has full support, both divergences in (ii) are finite, and
    \[
        \begin{split}
            D_{\rm KL}(\nu\parallel\mathrm{softmax}(z')) - D_{\rm KL}(\nu\parallel\mathrm{softmax}(z)) &= \sum_{v\in\Lambda_s}\nu(v)\log\frac{\mathrm{softmax}(z)_v}{\mathrm{softmax}(z')_v}\\
            &= -\langle\nu,h\rangle + \mathrm{LSE}(z+h) - \mathrm{LSE}(z).
        \end{split}
    \]
    The map $\tau\mapsto\mathrm{LSE}(z+\tau h)$ has the first derivative $\langle\pi_\tau,h\rangle$ and, by the computation above, the second derivative $\mathrm{Var}_{\pi_\tau}(h)\leq\|h\|_\infty^2$. By Taylor's theorem, we have the first order approximation: $\mathrm{LSE}(z+h) - \mathrm{LSE}(z)\leq\langle\pi_0,h\rangle + \frac{1}{2}\|h\|_\infty^2$ with $\pi_0 = \mathrm{softmax}(z)$. Therefore,
    \[
        \begin{split}
            D_{\rm KL}(\nu\parallel\mathrm{softmax}(z')) - D_{\rm KL}(\nu\parallel\mathrm{softmax}(z)) &\leq\langle\mathrm{softmax}(z) - \nu, h\rangle + \frac{1}{2}\|h\|_\infty^2\\
            &\leq\|\mathrm{softmax}(z) - \nu\|_1\|h\|_\infty + \frac{1}{2}\|h\|_\infty^2,
        \end{split}
    \]
    and $\|\cdot\|_1 = 2\|\cdot\|_\TV$ concludes (ii).
 
    \noindent{\it Proof of (iii).} By Jensen's inequality for the concave logarithm,
    \[
        \begin{split}
            D_{\rm KL}(P\parallel P') &= \sum_{v:\,P(v)>0}P(v)\log\frac{P(v)}{P'(v)}\\
            &\leq\log\sum_{v\in\Lambda_s}\frac{P(v)^2}{P'(v)}\\
            &= \log\bigg(1 + \sum_{v\in\Lambda_s}\frac{(P(v) - P'(v))^2}{P'(v)}\bigg),
        \end{split}
    \]
    where the last equality follows from expanding $P(v)^2 = (P(v)-P'(v))^2 + 2P(v)P'(v) - P'(v)^2$ and summing over $v$. Using $\log(1+x)\leq x$, $P'(v)\geq p$ and $\sum_v(P(v)-P'(v))^2\leq(\sum_v|P(v)-P'(v)|)^2$, we obtain
    \[
        D_{\rm KL}(P\parallel P')\leq\frac{1}{p}\bigg(\sum_{v\in\Lambda_s}|P(v) - P'(v)|\bigg)^2 = \frac{4}{p}\|P-P'\|_\TV^2.\qedhere
    \]
\end{proof}

\subsection{Proof of Theorem~\ref{thm:Unified}}\label{sub:Proof}
 
Fix $r\geq M$ and write $ w_r := q^{M,s}(Y_{r-M+1:r})$ for the quantized test window. Then $P_r^{M,s} = P_r^{M,s}(\,\cdot\mid w_r)$ and $\hat P_r^{M,s} = \hat P_r^{M,s}(\,\cdot\mid w_r)$ in Definition~\ref{def:PredictiveDists}, and the law of $ w_r$ is $p_r$ in~\eqref{eq:WindowLaw}. The data set enters only through $\gamma^*$ and it is independent of $Y_{0:r}$; so, in Steps~1--3, we fix the data set and take the expectation over $Y_{0:r}$ only. We prove step-by-step:
 
\noindent{\it Step 1.} For every realization of $Y_{0:r}$, all measures below are probability measures on $\YY$, and the triangle inequality gives
\begin{equation}\label{eq:ProofDecomposition}
    \begin{split}
        \|P_r - \Xi_s(\sT^{\gamma^*}(\,\cdot\mid w_r))\|_\TV &\leq \|P_r - P_r^{M,s}\|_\TV + \|P_r^{M,s}(\,\cdot\mid w_r) - \Xi_s(\hat P_r^{M,s}(\,\cdot\mid w_r))\|_\TV\\
        &\qquad + \|\Xi_s(\hat P_r^{M,s}(\,\cdot\mid w_r)) - \Xi_s(\sT^{\gamma^*}(\,\cdot\mid w_r))\|_\TV.
    \end{split}
\end{equation}
By Theorem~\ref{thm:Main}-(ii), the first term satisfies
\[
    \EE_{Y_{0:r}}\|P_r - P_r^{M,s}\|_\TV\leq\boldsymbol{\alpha}^{\lfloor M/N\rfloor} + \frac{2ML_\psi}{q_{\rm min}s}.
\]
 
\noindent{\it Step 2.} Fix $w\in\Lambda_s^M$. Conditionally on $Z_{r+1}$, the next observation $Y_{r+1}$ is independent of $Y_{r-M+1:r}$ and has the law $Q(\,\cdot\mid Z_{r+1})$. Hence, the $\mu$-density of $P_r^{M,s}(\,\cdot\mid w)$ is
\[
    p_r^{M,s}(y\mid w) := \sum_{z\in S}\psi(y\mid z)\,\Pr\big(Z_{r+1} = z\mid q^{M,s}(Y_{r-M+1:r}) = w\big),\qquad y\in\YY,
\]
which is a convex combination of $\{\psi(\,\cdot\mid z)\}_{z\in S}$. Exactly as for $p_r$ in the proof of Lemma~\ref{def:QuantizationBound}, Assumption~\ref{a:BaseAssumptions}-(A2) shows that $p_r^{M,s}(\,\cdot\mid w)$ is $L_\psi$-Lipschitz. Since $\hat P_r^{M,s}(\,\cdot\mid w) = q^s_\#P_r^{M,s}(\,\cdot\mid w)$ by Definition~\ref{def:PredictiveDists}, the $\mu$-density of $\Xi_s(\hat P_r^{M,s}(\,\cdot\mid w))$ is constant on each cell $C(s,i)$ and equal to the $\mu$-average of $p_r^{M,s}(\,\cdot\mid w)$ over $C(s,i)$. The remaining lines of the proof of Lemma~\ref{def:QuantizationBound} then apply verbatim and give
\[
    \|P_r^{M,s}(\,\cdot\mid w) - \Xi_s(\hat P_r^{M,s}(\,\cdot\mid w))\|_\TV\leq\frac{L_\psi}{s},\qquad w\in\Lambda_s^M.
\]
So, the second term in~\eqref{eq:ProofDecomposition} is at most $L_\psi/s$ for every realization of $Y_{0:r}$.
 
\noindent{\it Step 3.} For $\nu,\nu'\in\cP(\Lambda_s)$ and a Borel set $A\subset\YY$, Definition~\ref{def:RevertQuantization} gives
\[
    \Xi_s(\nu)(A) - \Xi_s(\nu')(A) = \sum_{i=1}^{R(s)}\big(\nu(y_{s,i}) - \nu'(y_{s,i})\big)\frac{\mu(A\cap C(s,i))}{\mu(C(s,i))}.
\]
Each ratio on the right hand side is in the interval $[0,1]$, and $\sum_{i}(\nu(y_{s,i}) - \nu'(y_{s,i})) = 0$. Therefore, the absolute value of the right hand side is at most $\sum_{i}(\nu(y_{s,i}) - \nu'(y_{s,i}))^+ = \|\nu-\nu'\|_\TV$, where $x^+ := \max\{x,0\}$. Hence, $\|\Xi_s(\nu) - \Xi_s(\nu')\|_\TV\leq\|\nu-\nu'\|_\TV$, and the third term in~\eqref{eq:ProofDecomposition} is at most $\|\hat P_r^{M,s}(\,\cdot\mid w_r) - \sT^{\gamma^*}(\,\cdot\mid w_r)\|_\TV$. Taking the expectation and collecting Steps~1--3, we have shown
\begin{equation}\label{eq:ProofReduction}
    \begin{split}
        &\EE_{Y_{0:r}}\|P_r - \Xi_s(\sT^{\gamma^*}(\,\cdot\mid w_r))\|_\TV\\
        &\qquad\leq\boldsymbol{\alpha}^{\lfloor M/N\rfloor} + \Big(1 + \frac{2M}{q_{\rm min}}\Big)\frac{L_\psi}{s} + \sum_{w\in\Lambda_s^M}p_r(w)\|\hat P_r^{M,s}(\,\cdot\mid w) - \sT^{\gamma^*}(\,\cdot\mid w)\|_\TV.
    \end{split}
\end{equation}
 
\noindent{\it Step 4.} Let $\mathcal{G}_1$ and $\mathcal{G}_2$ be the events of Lemma~\ref{lem:SampleComplexity}-(ii). Their intersection has probability at least $1-\delta$, and we fix a realization of the data set in $\mathcal{G}_1\cap\mathcal{G}_2$ for the rest of the proof. By Lemma~\ref{lem:SampleComplexity}-(ii), every window appears in the data set, so the sums over $k\in\cK$ in~\eqref{eq:cross_entropy_loss} and in the definition of $\cE$ are sums over $w\in\Lambda_s^M$, and
\[
    \cE(\gamma) = \sum_{w\in\Lambda_s^M}\frac{n_w}{\tilde K}D_{\rm KL}\big(y^w\parallel\sT^{\gamma}(\,\cdot\mid w)\big),\qquad\gamma\in\UU_m^{M+1}.
\]
By the triangle inequality,
\begin{equation}\label{eq:ProofSplit}
    \begin{split}
        \sum_{w\in\Lambda_s^M}p_r(w)\|\hat P_r^{M,s}(\,\cdot\mid w) - \sT^{\gamma^*}(\,\cdot\mid w)\|_\TV &\leq\sum_{w\in\Lambda_s^M}p_r(w)\|\hat P_r^{M,s}(\,\cdot\mid w) - y^w\|_\TV\\
        &\qquad + \sum_{w\in\Lambda_s^M}p_r(w)\|y^w - \sT^{\gamma^*}(\,\cdot\mid w)\|_\TV.
    \end{split}
\end{equation}
The first sum is at most $\sqrt{R(s)^MC_\delta/\tilde K}$ by Lemma~\ref{lem:SampleComplexity}-(iii). For the second sum, we use Pinsker's inequality \cite[Theorem~7.10]{polyanskiy2025information}, Jensen's inequality for the concave function $\sqrt{\cdot}$ with the weights $p_r(w)$, which sum to $1$, and the inequality $p_r(w)\leq 2n_w/\tilde K$, which holds on $\mathcal{G}_1$:
\begin{equation}\label{eq:ProofPinsker}
    \begin{split}
        \sum_{w\in\Lambda_s^M}p_r(w)\|y^w - \sT^{\gamma^*}(\,\cdot\mid w)\|_\TV &\leq\sum_{w\in\Lambda_s^M}p_r(w)\sqrt{\frac{1}{2}D_{\rm KL}\big(y^w\parallel\sT^{\gamma^*}(\,\cdot\mid w)\big)}\\
        &\leq\sqrt{\frac{1}{2}\sum_{w\in\Lambda_s^M}p_r(w)D_{\rm KL}\big(y^w\parallel\sT^{\gamma^*}(\,\cdot\mid w)\big)}\\
        &\leq\sqrt{\sum_{w\in\Lambda_s^M}\frac{n_w}{\tilde K}D_{\rm KL}\big(y^w\parallel\sT^{\gamma^*}(\,\cdot\mid w)\big)} = \sqrt{\cE(\gamma^*)}.
    \end{split}
\end{equation}
Finally, $\gamma^*$ minimizes $\cE$ over $\UU_m^{M+1}$, and the policy $\underline{\hat U}$ of Lemma~\ref{lem:TransformerPerformance} belongs to $\UU_m^{M+1}$. Therefore,
\begin{equation}\label{eq:ProofOptimality}
    \cE(\gamma^*)\leq\cE(\underline{\hat U}).
\end{equation}
 
\noindent{\it Step 5.} Fix $w\in\Lambda_s^M$, write $\eta := \eta(n,m,\beta)$ and $a_w := \|y^w - \hat P_r^{M,s}(\,\cdot\mid w)\|_\TV$. By Lemma~\ref{lem:TransformerPerformance}, $\sT^{\underline{\hat U}}(\,\cdot\mid w) = \mathrm{softmax}(\hat z_{\rm out}(w))$ with $\|\hat z_{\rm out}(w) - \bar z(w)\|_\infty\leq\eta$, and $\hat P_r^{M,s}(\,\cdot\mid w) = \mathrm{softmax}(\bar z(w))$. We apply Lemma~\ref{lem:Softmax}-(ii) with $\nu = y^w$, $z = \bar z(w)$ and $z' = \hat z_{\rm out}(w)$, and then Lemma~\ref{lem:Softmax}-(iii) with $P = y^w$ and $P' = \hat P_r^{M,s}(\,\cdot\mid w)$, whose entries are at least $q_{\rm min}\mu_{\rm min}^s$ by Lemma~\ref{lem:SampleComplexity}-(i):
\[
    D_{\rm KL}\big(y^w\parallel\sT^{\underline{\hat U}}(\,\cdot\mid w)\big)\leq D_{\rm KL}\big(y^w\parallel\hat P_r^{M,s}(\,\cdot\mid w)\big) + 2a_w\eta + \frac{\eta^2}{2}\leq\frac{4a_w^2}{q_{\rm min}\mu_{\rm min}^s} + 2a_w\eta + \frac{\eta^2}{2}.
\]
We multiply by $n_w/\tilde K$ and sum over $w$. Write $A := (\sum_w\frac{n_w}{\tilde K}a_w^2)^{1/2}$. Since $\sum_wn_w/\tilde K = 1$, the Cauchy-Bunyakovsky-Schwarz inequality gives $\sum_w\frac{n_w}{\tilde K}a_w\leq A$, and therefore
\[
    \cE(\underline{\hat U})\leq\frac{4A^2}{q_{\rm min}\mu_{\rm min}^s} + 2A\eta + \frac{\eta^2}{2}\leq\bigg(\frac{2A}{\sqrt{q_{\rm min}\mu_{\rm min}^s}} + \frac{\eta}{\sqrt{2}}\bigg)^2,
\]
where, for the last inequality, we expanded the square and used $2A\eta\leq 2\sqrt{2}A\eta/\sqrt{q_{\rm min}\mu_{\rm min}^s}$. The latter holds since $q_{\rm min}\mu_{\rm min}^s\leq Q((q^s)^{-1}(v)\mid z)\leq 1$ for all $v\in\Lambda_s$ and $z\in S$ by the proof of Lemma~\ref{lem:SampleComplexity}-(i). By Lemma~\ref{lem:SampleComplexity}-(iii), $A^2\leq R(s)^MC_\delta/(2\tilde K)$, hence
\begin{equation}\label{eq:ProofConstructed}
    \sqrt{\cE(\underline{\hat U})}\leq\sqrt{\frac{2R(s)^MC_\delta}{q_{\rm min}\mu_{\rm min}^s\tilde K}} + \frac{\eta}{\sqrt{2}}.
\end{equation}
 
\noindent{\it Step 6.} Combining~\eqref{eq:ProofSplit}--\eqref{eq:ProofConstructed} with the bound on the first sum in~\eqref{eq:ProofSplit}, and using $1\leq(q_{\rm min}\mu_{\rm min}^s)^{-1/2}$, we get
\[
    \begin{split}
        \sum_{w\in\Lambda_s^M}p_r(w)\|\hat P_r^{M,s}(\,\cdot\mid w) - \sT^{\gamma^*}(\,\cdot\mid w)\|_\TV &\leq\sqrt{\frac{R(s)^MC_\delta}{\tilde K}} + \sqrt{\frac{2R(s)^MC_\delta}{q_{\rm min}\mu_{\rm min}^s\tilde K}} + \frac{\eta}{\sqrt{2}}\\
        &\leq(1+\sqrt{2})\sqrt{\frac{R(s)^MC_\delta}{q_{\rm min}\mu_{\rm min}^s\tilde K}} + \frac{\eta}{\sqrt{2}}.
    \end{split}
\]
Plugging this into~\eqref{eq:ProofReduction} gives~\eqref{eq:UnifiedBound}. Since the realization of the data set in $\mathcal{G}_1\cap\mathcal{G}_2$ is arbitrary, the bound holds with probability at least $1-\delta$, and the proof is now concluded.\hfill$\blacksquare$

\subsection{Proof of Theorem~\ref{thm:Soft}}\label{proof:Soft}

By Lemma~\ref{lem:ErgodicityIsInherited}, $\{(Z_r,Y_r)\}$ is stationary. Hence the window law
$p(w):=\Pr\big(q^{M,s}(Y_{r-M+1:r})=w\big)$ and the objects $P^M_r$, $P^{M,s}_r$ and
$\hat P^{M,s}_r=:\hat P(\cdot\mid w)$ do not depend on $r\geq M$. Fix $M$ and $s$ with $R(s)\geq2$, and
let $\underline U$ be any open-loop policy. We proceed in four steps.

\emph{Step 1.} Replacing
the memory bound of Theorem~\ref{thm:Main} with that of Proposition~\ref{prop:Info}, we obtain
\begin{equation}\tag{$\star$}\label{eq:41-decomp}
  \begin{split}
      \sup_{r\geq M}\EE_{Y_{0:r}}\Big\|P_r-\Xi_s\big(T^{\underline U}(\cdot\mid w_r)\big)\Big\|_{\TV}
      \leq\sqrt{\Delta_M/2}&+\Big(1+\frac{2M}{q_{\min}}\Big)\frac{L_\psi}{s}\\
      &+\sum_{w\in\Lambda_s^M}p(w)\,\big\|\hat P(\cdot\mid w)-T^{\underline U}(\cdot\mid w)\big\|_{\TV},
  \end{split}
\end{equation}
where $w_r:=q^{M,s}(Y_{r-M+1:r})$.

\emph{Step 2.} Let $\xi_j:=\big(q^{M,s}(Y_{j:j+M-1}),q^s(Y_{j+M})\big)$. This is a fixed measurable function of finitely many coordinates of the shifted sequence $(Y_{j+i})_{i\geqq0}$.
Since $\{Y_r\}$ is stationary and ergodic (Lemma~\ref{lem:ErgodicityIsInherited}), so is $\{\xi_j\}_j$. By Birkhoff's ergodic
theorem, applied to the finitely many indicators $\mathbf 1\{\xi_j=(w,v)\}$, there is an event
$\Omega_{M,s}$ of probability one on which
\[
  \frac{n_w}{\tilde K}\to p(w)
  \qquad\text{and}\qquad
  \frac{1}{\tilde K}\sum_{j=1}^{\tilde K}\mathbf 1\{\xi_j=(w,v)\}\to p(w)\,\hat P(v\mid w)
\]
for all $(w,v)\in\Lambda_s^M\times\Lambda_s$. The proof of Lemma~\ref{lem:SampleComplexity}-(i) uses only the conditional
independence of the observations, so $p(w)\geq(q_{\min}\mu^s_{\min})^M>0$. Hence, on $\Omega_{M,s}$:
\begin{itemize}
  \item there exists $\tilde K_0(\omega)$ such that, for all $\tilde K\geq\tilde K_0(\omega)$, every
  window appears in the data set and $p(w)\le2n_w/\tilde K$ for every $w$;
  \item $a_w:=\|y^w-\hat P(\cdot\mid w)\|_{\TV}\to0$ for every $w$, as $\tilde K\to\infty$.
\end{itemize}

\emph{Step 3.} Fix aligning levels $(n,m,l)$ and $\beta$ with
$\eta:=\eta(n,m,\beta)\le1$, and fix $\omega\in\Omega_{M,s}$ and $\tilde K\geq\tilde K_0(\omega)$.
Steps~4 and~5 of §\ref{sub:Proof} use only the properties listed at the end of Step~2, the optimality of
$\underline U^*_{\tilde K}$, and Lemmata~\ref{lem:Stability} and~\ref{lem:Softmax}. These steps
therefore give
\[
  \sum_{w}p(w)\big\|\hat P(\cdot\mid w)-T^{\underline U^*_{\tilde K}}(\cdot\mid w)\big\|_{\TV}
  \leq\sum_w p(w)\,a_w+\frac{2A_{\tilde K}}{\sqrt{q_{\min}\mu^s_{\min}}}+\frac{\eta}{\sqrt2},
\]
with $A_{\tilde K}^2:=\sum_w\frac{n_w}{\tilde K}\,a_w^2$.

By Step~2, $a_w\to0$ for each of the finitely many $w$, and $n_w/\tilde K\le1$. Hence the first two
terms vanish as $\tilde K\to\infty$. Combining with \eqref{eq:41-decomp}, on $\Omega_{M,s}$ we obtain
\begin{equation}\tag{$\star'$}\label{eq:41-limK}
  \limsup_{\tilde K\to\infty}\ \sup_{r\geq M}\EE_{Y_{0:r}}\Big\|P_r-\Xi_s\big(T^{\underline U^*_{\tilde K}}(\cdot\mid w_r)\big)\Big\|_{\TV}
  \leq\sqrt{\Delta_M/2}+\Big(1+\frac{2M}{q_{\min}}\Big)\frac{L_\psi}{s}+\frac{\eta(n,m,\beta)}{\sqrt2}.
\end{equation}

\emph{Step 4.} The event $\Omega^*:=\bigcap_{M,s\in\mathbb N}\Omega_{M,s}$ is a
countable intersection, so it has probability one. It does not depend on $(n,m,l,\beta)$, and on it
\eqref{eq:41-limK} holds for all $M$, $s$, aligning $(n,m,l)$ and $\beta$ with $\eta\le1$. For fixed
$M$ and $s$, we have $c=M_\UU$, so each term of \eqref{eq:EtaDef} tends to $0$ as $\beta,n,m\to\infty$ along aligning
levels; in particular $\eta\le1$ eventually. Letting $(n,m,\beta)\to\infty$, then $s\to\infty$, then
$M\to\infty$, and using $\Delta_M\downarrow0$ (Proposition~\ref{prop:Info}), proves the claim on $\Omega^*$.\hfill$\blacksquare$

\bibliographystyle{abbrvnat}

\end{document}